\documentclass{article}

\PassOptionsToPackage{numbers, compress}{natbib}
 \usepackage[preprint]{neurips_2026}

\usepackage[utf8]{inputenc} 
\usepackage[T1]{fontenc}    
\usepackage{hyperref}       
\usepackage{url}            
\usepackage{booktabs}       
\usepackage{amsfonts}       
\usepackage{nicefrac}       
\usepackage{microtype}      
\usepackage{xcolor}         
\usepackage{amsmath}
\usepackage{amssymb}
\usepackage{graphicx}
\usepackage{subcaption}
\usepackage{mathtools}   
\usepackage{algorithm}
\usepackage{algpseudocode}

\newcommand{\bind}{\otimes}
\newcommand{\bundle}{\oplus}
\newcommand{\inv}{\dagger}
\newcommand{\E}{\mathbb{E}}
\newcommand{\mat}[1]{\mathbf{#1}}
\newcommand{\R}{\mathbb{R}}
\newcommand{\T}{\mathsf{T}}
\newcommand{\cov}{\operatorname{Cov}}
\DeclarePairedDelimiter{\norm}{\lVert}{\rVert}

\renewcommand{\H}{\operatorname{H}} 
\newcommand{\var}{\operatorname{Var}}
  \usepackage{amsthm}
\newtheorem{proposition}{Proposition}
\newcommand{\I}[2]{\mathbb{I}\!\left(#1;\,#2\right)}
\newtheorem{remark}{Remark}

\newcommand{\sg}[1]{\operatorname{sg}(#1)}

\title{Disentanglement with Holographic Reduced Representations}

\author{
  Jhonny J. Velasquez Olivera \\
  Virginia Tech \\
  \texttt{jhonnyv19@vt.edu} \\
  \And
  Christo K. Thomas \\
  Worcester Polytechnic Institute \\
  cthomas2@wpi.edu \\
  \And
  Walid Saad \\
  Virginia Tech \\
  \texttt{walids@vt.edu} \\
}

\begin{document}

\maketitle

\begin{abstract}
The process of disentanglement, that focuses on separating the factors of variation in a dataset using neural networks, is a long standing challenge in machine learning. Prior solutions to this problem included the design of variational autoencoders and generative adversarial networks that use concepts from variational inference and information-theoretic constraints in their loss functions, respectively. In contrast to these prior works that rely on continuous representations in their models, we propose a new design that naturally views disentangled representations as symbolic structures due to the compositional nature of the relationship between the concepts that make up a sample from a distribution. However, learning discrete symbolic structures using traditional neural network architectures while maintaining differentiability is challenging, often requiring complex architectures to accomplish. To this end, in this paper, we propose an unsupervised learning algorithm that makes use of holographic reduced representations (HRR) for disentanglement using neural networks. We find that the unbinding operation, defined over HRR vectors, yields a suitable inductive bias for the separation of factors, and yields competitive results compared to other baselines, as measured using latent traversals and disentanglement metrics. We complement these empirical findings with an information-theoretic analysis of the HRR unbinding channel. We prove that unbinding induces approximately independent symbol-value pairs (also called slots) and derive a per-slot capacity bound that quantifies how many distinct symbolic concepts the representation can reliably encode, providing a quantitative account of the inductive bias toward disentanglement. The disentangled representations produced in this process differ from other autoencoder based models in that the individual latent units are vectors themselves which are summed together, differing from the paradigm of latent units behaving as scalar dimensions of low dimensional vectors. We show that this type of HRR representation is more robust to noise than other disentangled representations and can maintain good reconstruction performance across a range of SNRs, while simultaneously gracefully degrading in the recoverable semantic content. Our results thus show that the use of vector symbolic architectures (VSA), such as HRRs, may hold a promising potential for representation learning, paving the way toward exploring new ways in which the symbolic benefits of VSAs can be used to represent data with neural networks.   
\end{abstract}

\section{Introduction}

The pursuit of learning efficient representations for data has been of interest to the deep learning field for many years, and led to the emergence of the field of representation learning~\cite{bengioRepresentationLearningReview2013}. Understanding representations of data within neural networks can shed light into how to construct better inductive biases to train them, and provide new insights into how neural networks interact with the data they use. Disentangled representation learning is a popular approach to the more general representation learning problem whose focus is on how to disentangle the generative factors that make up a dataset. Perhaps the earliest work on this was a more heavily regularized variational autoencoder (VAE)~\cite{kingmaAutoEncodingVariationalBayes2022a}, called the $\beta$-VAE~\cite{higgins2017betavae}. Following the $\beta$-VAE, a line of work emerged that attempted to further improve on its disentanglement abilities leading to works such as the total correlation (TC) VAE~\cite{chenIsolatingSourcesDisentanglement2018}, and the Factor-VAE~\cite{pmlr-v80-kim18b}, each of which build on the joint reconstruction and Kullback Leibler (KL) divergence objectives of the original VAE. Alongside these models, a parallel line of work proposed more comprehensive metrics for evaluating disentanglement~\cite{ridgewayLearningDeepDisentangled2018a, eastwood2018a}. Most prior art in disentanglement  has followed the VAE framework, with the exception of InfoGAN, an information-theoretic extension of generative adversarial networks (GANs)~\cite{goodfellowGenerativeAdversarialNets2014, chenInfoGANInterpretableRepresentation2016a}. Recently, there has also been evidence to show that disentanglement can arise using discrete representations~\cite{hsu2023disentanglement}, building on previous evidence that neural networks can learn discrete representations, in general~\cite{vandenoordNeuralDiscreteRepresentation2017}. However,
disentanglement does not arise on its own. Without inductive biases on both
the data and the model, unsupervised disentanglement is
impossible~\cite{JMLR:v21:19-976}. This points to the
inductive bias of each method as the main driver of its disentanglement
quality. Vector symbolic architectures (VSAs) offer a structurally different inductive bias, namely a \emph{compositional vector space}, but their potential for unsupervised disentanglement remains largely unexplored. A prior work explored this intersection and was able to disentangle generative factors using VSAs~\cite{korchemnyi2024symbolicdisentangledrepresentationsimages}. However, the training procedure required carefully constructed data pairs that differ in only one generative factor at a time, limiting the solution of \cite{korchemnyi2024symbolicdisentangledrepresentationsimages} to datasets that span a sufficient number of factor combinations.

A central challenge to using VSAs in representation learning is the difficulty of mapping high-dimensional data to the structured vector spaces that VSAs provide. Mapping data to vectors using VSAs is traditionally done using hand designed methods that attempt to capture the structures apparent in the raw data~\cite{kleykoSurveyHyperdimensionalComputing2022, smetsEncodingFrameworkBinarized2024, 10.1145/3195970.3196060, rennerNeuromorphicVisualScene2024a}. For example, one method of mapping images to a VSA vector involves representing the image as a sum of many vectors, each representing pixel intensity and position~\cite{rennerNeuromorphicVisualScene2024a}. While such methods are interpretable, they can only encode structure that the designer anticipates in advance, leaving little room for the model to discover new representations from data. On the other hand, neural networks have been shown to be able to learn strong representations of data that can be used for downstream tasks such as classification \cite{chenSimpleFrameworkContrastive2020a, DBLP:journals/corr/RadfordMC15, caronEmergingPropertiesSelfSupervised2021b, assranSelfSupervisedLearningImages2023a, grillBootstrapYourOwn2020a, heMaskedAutoencodersAre2022, mikolovEfficientEstimationWord2013a, radfordLearningTransferableVisual2021a}. However, using neural networks to learn VSA representations directly from data remains a relatively unexplored field. Some works have trained neural networks to map data into VSAs~\cite{herscheNeurovectorsymbolicArchitectureSolving2023a, ganesanLearningHolographicReduced2021a, alam2023towards, alamRecastingSelfAttentionHolographic2023a}, but these typically rely on labeled training data~\cite{ganesanLearningHolographicReduced2021a}, generalize poorly, and leave little room for the model to discover its own structure. Despite the rich literature on disentanglement and the long history of VSAs as a representation framework, there is no existing method that learns HRR-structured latent representations from data without supervision, and no existing theory that explains why an HRR structure should produce disentanglement. This paper closes this gap.

The main contribution of this paper is to demonstrate, both empirically and theoretically, that holographic reduced representations (HRRs) provide an effective inductive bias for unsupervised disentanglement, achieving competitive performance on standard benchmarks while admitting a quantitative information-theoretic account of why the bias works. In particular, our key contributions include:
\begin{itemize}
    
    \item We show that training a neural network with a reconstruction objective,
paired with a latent reconstruction process and an HRR structural regularizer,
\emph{produces disentangled representations of the generative factors of a
dataset} without the use of any labels, unlike prior work on learning VSA representations from
data~\cite{korchemnyi2024symbolicdisentangledrepresentationsimages}.

    \item We derive a theoretical upper bound on the mutual information between the
input and the HRR-structured latent, quantifying how many distinct symbolic
concepts the representation can reliably encode. This bound provides a
principled basis for choosing the architectural parameters HRR vector dimension, number of symbols, and the latent codebook size.
    \item Experimental results show that the resulting representations achieve the best aggregate
scores on five of six standard disentanglement metrics (InfoMEC and DCI)
across four datasets, with relative improvements of roughly $5-10\%$ over the
next-best baseline on the core disentanglement metrics, while exhibiting
noise robustness comparable to VQ-VAE and considerably better than other
disentanglement methods. HRR is the only model that achieves both superior
disentanglement and strong noise robustness on the studied benchmarks.  Qualitative analyses of latent
    component swapping further confirm that the learned representation
    aligns with individual generative factors.

\end{itemize}
\section{Preliminaries of HRR vectors}
\label{sec:hrr-prelim}

HRRs are a VSA that represents discrete entities, called \emph{symbols}, as high-dimensional random vectors, and represents structured combinations of those entities through algebraic operations on the vectors. Based on~\citep{plateHolographicReducedRepresentations1995a}, an HRR vector $\mat{x} \in \mathbb{R}^d$ is typically initialized with entries sampled independently as $x_i \sim \mathcal{N}(0,1/d)$, so that distinct randomly initialized vectors are approximately orthogonal for large $d$.

HRRs define two main operations, \emph{binding} and \emph{bundling}. Binding, denoted $\bind$, associates two vectors through circular convolution, producing a representation that is approximately dissimilar to either input. In practice, this operation can be computed efficiently as multiplication in the frequency domain. Bundling, denoted $\bundle$, superimposes vectors through addition and is commonly used to represent sets or compositions of symbols.
A useful property of HRRs is approximate unbinding. Given a bound representation $\mat{z} = \mat{x} \bind \mat{y}$, one can recover a noisy estimate of $\mat{x}$ by applying an approximate inverse of $\mat{y}: \hat{\mat{x}} = \mat{y}^{\inv} \bind \mat{z} \approx \mat{x}$. The approximation error is often called \emph{unbinding noise}. Here, ``inverse'' follows the terminology of the HRR literature rather than denoting an exact algebraic inverse; we defer further details to the appendix and refer readers to prior work on HRR inverses~\cite{ganesanLearningHolographicReduced2021a}, and the original HRR paper~\cite{plateHolographicReducedRepresentations1995a}.

\section{HRRs for disentanglement} \label{sec:hrrs_for_disentanglement}

Disentangled representations aim to capture the structure within a dataset. Although continuous representations, and more recently discrete representations~\cite{hsu2023disentanglement, 10.5555/3692070.3692839}, have been shown to separate some of this structure, we will show that symbolic representations may be better suited for this task. Symbolic structures are useful for representing data with compositional structure, where entities, attributes, and relations combine in meaningful ways. For example, consider an image of a dog in a backyard. Although the image is represented at the pixel level, we naturally interpret the scene in terms of entities, attributes, and relations. There is a dog, the dog is in a backyard, and the dog has attributes such as breed, age, color, and size. This interpretation can be organized as a simple symbolic structure, where the root node captures the main entity, \emph{dog}, and its associated attributes are represented as role-filler pairs such as \emph{breed: Chihuahua} or \emph{size: small}. In a VSA, this kind of structure can be encoded by binding each role vector to its corresponding filler vector and bundling the resulting pairs into a single distributed representation. While this example uses only one level of role-filler binding, more complex hierarchies can be expressed through recursive composition~\cite{plateHolographicReducedRepresentations1995a}. The flat latent representations used in prior continuous and discrete disentanglement work  \cite{hsu2023disentanglement,10.5555/3692070.3692839} do not naturally express this kind of compositional structure. Here, we present a proof of concept that a neural network can be trained to produce such structured vector representations directly from data.

HRRs allow us to impose such a structure on a neural network's latent space using a hetero-associative memory, where pairs of vectors are stored such that one element of the pair can be retrieved using the other as a cue. We define an HRR structured representation as a vector in the form $\mat{z} = \left(\mat{s}_1 \bind \mat{v}_1 \right) \bundle \cdots \bundle \left(\mat{s}_m \bind \mat{v}_m \right)$. Each $\mat{s}_i$ and $\mat{v}_i$ is an HRR vector, where $\bind$ and $\bundle$ are the binding and bundling operations over HRRs. We refer to the $\mat{s}_i$  as symbol vectors, which  act as the cues to the memory and to  $\mat{v}_i$ as value vectors, which  act as the items stored in the memory. We refer to the number of symbol-value pairs 
$m$ as the symbolic length of the representation. The encoder is trained to produce $\mat{z}$ from the data $\mat{x}$. Ideally, the model learns to associate each $\mat{s}_i$ with the most important generative factors of the data. For example, on images of chairs, $\mat{s}_1$ could represent   chair type (sofa, stool, recliner, etc.), $\mat{s}_2$ chair pose, and $\mat{s}_3$ chair size; the corresponding $\mat{v}_i$ would then encode the specific value of each factor (e.g., $\mat{v}_1
$ for ``sofa").  Each bound pair $\mat{s}_i \bind \mat{v}_i$ thus represents the realization
of one generative factor, and the bundled sum
$\mat{z} = \bigoplus_{i=1}^{m} \mat{s}_i \bind \mat{v}_i$ forms a
neuro-symbolic representation of $\mat{x}$. The symbolic length 
$m$ also controls how finely the model can separate generative factors. A larger 
$m$ provides more symbol-value pairs in which to place distinct factors, while a smaller 
$m$ forces the model to group factors together.

\subsection{Model architecture} \label{ssec:model_arch}

Fundamentally, our model is a CNN based autoencoder that learns to map the data, $\mat{x}$, directly to the latent representation, $\mat{z}$, with an HRR unbinding and vector quantization step. Along with the autoencoder we initialize a set of $m$ symbol vectors, $\mathcal{S} =\{\mat{s}_1, \mat{s}_2, \dots, \mat{s}_m\}$, and a codebook of $k \gg m$ value vectors, $\mathcal{C} =\{\mat{c}_1, \mat{c}_2, \dots, \mat{c}_k\}$. The vectors in $\mathcal{S}$ are randomly initialized from the HRR distribution described in Section~\ref{sec:hrr-prelim} and remain frozen throughout training. The codebook $\mathcal{C}$ is also initialized randomly but not from the HRR distribution as the codebook is optimized throughout the training process. 

The encoder first predicts the latent representation, $\mat{z}$, from the data. From the latent representation, we retrieve the vectors in the hetero-associative memory through the unbinding operation defined in Section~\ref{sec:hrr-prelim} by binding the latent once with the inverse of each of the vectors in $\mathcal{S}$. This retrieval produces the noisy value vectors $\{\tilde{\mat{v}}_1, \tilde{\mat{v}}_2, \dots, \tilde{\mat{v}}_m\}$, which each get processed by a small feedforward neural network, referred to as a denoising network, to produce $\{\bar{\mat{v}}_1, \bar{\mat{v}}_2, \dots, \bar{\mat{v}}_m\}$. From here, we do vector quantization as described in~\cite{vandenoordNeuralDiscreteRepresentation2017} to map each $\bar{\mat{v}}_i$ to its nearest Euclidean neighbor in the codebook $\mathcal{C}$, resulting in the final set of value vectors $\{\hat{\mat{v}}_1, \dots, \hat{\mat{v}}_m\}$, where each $\hat{\mat{v}}_i \in \mathcal{C}$. We will write $\bar{\mathbf{V}} \in \mathbb{R}^{m \times d}$ and
$\mathbf{C} \in \mathbb{R}^{k \times d}$ for the matrices whose rows are
the denoised retrievals and the codebook entries, respectively. We bind each $\hat{\mat{v}}_i$ with the corresponding symbol vector that it originally corresponded with to construct $\hat{\mat{z}} = \left(\mat{s}_1 \bind \hat{\mat{v}}_1 \right) \bundle \cdots \bundle \left(\mat{s}_m \bind \hat{\mat{v}}_m \right)$. The decoder then uses $\hat{\mat{z}}$ to reconstruct the original input $\mat{x}$.

\vspace{-2mm}\subsection{Regularization and optimization}

For a matrix $\mathbf{X} \in \mathbb{R}^{n \times d}$ with rows $\mathbf{x}_i$, we define the structural regularizer
\vspace{-0.2em}
\begin{equation}
    \mathcal{R}(\mathbf{X};\, \tau_n, \tau_v) =
        \left(\frac{1}{n}\sum_{i=1}^{n}\|\mathbf{x}_i\|^2 - \tau_n\right)^{\!2}
      + \left(\operatorname{mean}(\mathbf{X})\right)^2
      + \left(\operatorname{Var}(\mathbf{X}) - \tau_v\right)^{\!2},
\end{equation}
where $\operatorname{mean}$ and $\operatorname{Var}$ are computed globally over all entries.  This regularizer pushes a matrix toward zero mean, target per-row squared norm $\tau_n$, and target per-entry variance $\tau_v$, matching the statistical properties expected of HRR vectors and their bundled compositions (Appendix~\ref{app:retrieval}).  We apply $\mathcal{R}$ to three different matrices, with targets chosen so that each matches the relevant HRR distribution, as shown in Table~\ref{tab:regularizers}.

\vspace{-1.5em}
\begin{table}[h]
\caption{Specifications of the three structural regularizers.}
\vspace{0.1em}
\label{tab:regularizers}
\centering
\begin{tabular}{lccccc}
\toprule
Term & $\mathbf{X}$ & $n$ & $\tau_n$ & $\tau_v$ \\
\midrule
$\mathcal{L}_{\text{latent}}$   & $\mathbf{z}^{\top}$ & $1$ & $m$   & $m/d$ \\
$\mathcal{L}_{\text{value}}$    & $\bar{\mathbf{V}}$  & $m$ & $1$   & $1/d$ \\
$\mathcal{L}_{\text{codebook}}$ & $\mathbf{C}$        & $K$ & $1$   & $1/d$ \\
\bottomrule
\end{tabular}
\end{table}
\vspace{-0.8em}

The targets $\tau_v = 1/d$ for $\bar{\mathbf{V}}$ and $\mathbf{C}$ reproduce
the per-component variance of HRR vectors (Section~\ref{sec:hrr-prelim}); the targets
for $\mathbf{z}$ ($\tau_n = m$, $\tau_v = m/d$) match the expected squared norm
and per-component variance of a sum of $m$ bound HRR pairs (Appendix~\ref{app:retrieval}). Because the vector quantization step is non-differentiable, we use the straight-through estimator~\cite{bengio2013estimatingpropagatinggradientsstochastic} to propagate gradients through it.
    
\paragraph{Loss function}

The model is trained with an unsupervised auto-encoding objective, using the AdamW optimization algorithm~\cite{loshchilov2018decoupled}. The total loss combines a binary cross-entropy reconstruction loss
$\mathcal{L}_{\text{recon}}$, the vector quantization (VQ) and commitment
losses from the original VQ-VAE~\cite{vandenoordNeuralDiscreteRepresentation2017},
and the three structural regularizers defined above. Following~\cite{vandenoordNeuralDiscreteRepresentation2017},
we define

\vspace{-1.1em}
\[
\mathcal{L}_{\text{VQ}} = \|\text{sg}[\bar{\mat{v}}] - \mathbf{c}||^2 , \, \mathcal{L}_{\text{commit}} = \|\bar{\mat{v}} - \text{sg}[\mathbf{c}]||^2,
\]
\vspace{-1.1em}

where $\text{sg}[\cdot]$ is the stop gradient operator, $\bar{\mat{v}}$ is the denoised retrieval, and $\mathbf{c}$ is the nearest Euclidean neighbor in the codebook $\mathcal{V}$. $\mathcal{L}_{\text{VQ}}$ allows gradients to reach the codebook, while $\mathcal{L}_{\text{commit}}$ allows gradients to reach the encoder. Both terms ensure the codebook and the encoder work together to converge on mutually optimal codebook entries and latent predictions, respectively.
The total objective is
\begin{equation}
    \mathcal{L} = \underbrace{\mathcal{L}_{\text{recon}}
                + \lambda_{\text{vq}}\mathcal{L}_{\text{VQ}}
                + \lambda_{c}\mathcal{L}_{\text{commit}}}_{\text{reconstruction}}
                + \underbrace{\lambda_{\text{lat}}\mathcal{L}_{\text{latent}}
                + \lambda_{\text{val}}\mathcal{L}_{\text{value}}
                + \lambda_{\text{cb}}\mathcal{L}_{\text{codebook}}}_{\text{regularization}},
\end{equation}
where all $\lambda \geq 0$ are scalar weights. The hyperparameters used for all runs are listed in Table~\ref{tab:model_hyperparams}.

\subsection{Capacity of the HRR latent channel}
\label{sec:capacity}

Having specified the model and its training objective, we now turn to an
information-theoretic analysis of the HRR latent channel, which clarifies
how much information the symbolic latent $\hat{\mathbf{z}}$ can carry about
the input $\mathbf{x}$ and motivates the design choices made in our experiments.
In what follows, we refer to the 
$i$-th symbol-value pair $(\mathbf{s}_i, \mathbf{v}_i)$ as the 
$i$-th slot. The unbinding analysis in Appendix~\ref{app:retrieval} shows that, when retrieving the value vector
bound to slot~$i$, the result decomposes as

\vspace{-1em}
\begin{equation}
\tilde{\mat{v}}_i \;=\; (1+\xi)\,\mat{v}_i \;+\; \boldsymbol{\eta}_i \;+\!\!\sum_{k\neq i}\boldsymbol{\lambda}_i^{(k)},
\label{eq:retrieval-decomp}
\end{equation}
\vspace{-1em}

with intrinsic noise $\boldsymbol{\eta}_i$ and cross-talk $\boldsymbol{\lambda}_i^{(k)}$ from the other $m-1$ bound pairs.
Treating $\xi\to 0$ in expectation, the per-component signal power is $1/d$ and the per-component
noise power is $m/d$, giving an SNR of $1/m$ per dimension. Each retrieval thus passes through
$d$ parallel AWGN sub-channels, each with the same SNR. We use this to bound the total information
the symbolic latent $\hat{\mat{z}}$ can carry about the input $\mat{x}$.
 
\begin{proposition}
\label{prop:capacity}
Let $\hat{\mat{z}}=\sum_{i=1}^{m}\mat{s}_i\bind \hat{\mat{v}}_i$ be the symbolic latent, where each $\hat{\mat{v}}_i$ is the quantized value
selected from a codebook of size $k$. Under the unbinding noise model of
Appendix~\ref{app:retrieval}, the mutual information between input and latent satisfies

\vspace{-1em}
\begin{equation}
\I{\mat{x}}{\hat{\mat{z}}} \;\le\; \min \,\Bigl(\, m\cdot\tfrac{d}{2}\log\!\bigl(1+\tfrac{1}{m}\bigr),\; m\log k\,\Bigr).
\label{eq:capacity-bound}
\end{equation}
\end{proposition}
\vspace{-1em}

\begin{proof}
The proof proceeds in four steps. The data-processing inequality reduces
$\I{\mat{x}}{\hat{\mat{z}}}$ to a sum of per-slot mutual informations once
approximate slot independence (Proposition~\ref{prop:slot-indep}) is
established. Each per-slot term is then bounded by (a) the capacity of the
$d$-dimensional AWGN unbinding channel, $\tfrac{d}{2}\log(1+\mathrm{SNR})$
with $\mathrm{SNR}=1/m$, and (b) the codebook entropy $\log k$. Taking the
minimum of the two gives the bound. See Appendix~\ref{app:proofs} for the
full proof.
\end{proof}
\paragraph{Slot--dimension tradeoff}
The continuous part of the bound in~\eqref{eq:capacity-bound} is
\[
\frac{md}{2}\log(1+1/m),
\]
or equivalently \(\frac d2\log(1+1/m)\) per slot. Thus, increasing the number of
slots $m$ provides more slots, but also lowers
the effective per-slot SNR because each retrieved slot receives cross-talk from the other slots. For large $m$, $\log(1+1/m)\approx 1/m$, so the total continuous capacity approaches $\frac{d}{2}$. In this regime, adding more slots does not substantially increase the total amount of information that can be carried by the HRR latent, rather, it primarily divides the available capacity across more noisier retrieval channels.

This formalizes the empirical observation that unbinding performance degrades
as \(m\) increases at fixed \(d\) (Appendix~\ref{app:experimental-setup}). It also
motivates using a modest latent-to-source overcomplete ratio in our experiments, as the number of slots should be large enough to cover the generative factors, but not so large that each slot's effective retrieval channel becomes dominated by cross-talk noise. This motivates the $1.5\times$ latent-to-source ratio used in
Section~\ref{sec:disentanglement_experiments}.
 
\paragraph{Codebook sizing}

The second term, \(m\log k\), is the maximum discrete rate available after quantization. When
\[
\log k < \frac{d}{2} \log(1+1/m),
\]
the codebook resolution is the active bottleneck, and increasing \(k\) can increase the representational rate. Once
\[
\log k \geq \frac{d}{2} \log(1+1/m),
\]
the HRR retrieval channel becomes the bottleneck, so further increasing $k$ cannot increase the leading-order capacity bound. The codebook perplexity curves in Figure~\ref{fig:codebook_utilization} are consistent with this interpretation. Harder datasets such as Isaac3D and MPI3D-C drive higher codebook perplexity, suggesting that these datasets require more discrete states per slot and are therefore more likely to operate near the codebook-resolution constraint.
\subsection{Approximate slot independence}
\label{sec:slot-indep}

We write $q(\tilde{\mathbf{v}}_1,\ldots,\tilde{\mathbf{v}}_m\mid \mathbf{x})$
for the joint distribution of the unbound retrievals conditional on an input
$\mathbf{x}$, where the randomness is induced by the HRR initialization of the
symbol vectors $\mathbf{s}_i \in \mathcal{S}$. Although these symbols are fixed
after model initialization, treating them as random in the analysis allows us to
quantify the typical dependence induced by HRR cross-talk.

The proof of Proposition~\ref{prop:capacity} uses the fact that the retrieved
slots become approximately independent in high dimension. This dependence is
not exactly zero at finite $d$: the cross-talk term
$\boldsymbol{\lambda}_i^{(k)}$ in~\eqref{eq:retrieval-decomp} couples slot~$i$
to slot~$k$ through products of the form
$\mathbf{s}_i^{\dagger}\bind \mathbf{s}_k\bind \mathbf{v}_k$. However, the
resulting off-diagonal covariance blocks are small. In particular, Appendix~\ref{app:off_diagonal_blocks}
shows that for $i\neq j$,
\[
(\Sigma_{ij})_{ab}
=
\frac{1}{d}(\mathbf{v}_j\bind\mathbf{v}_i)_{a+b},
\]
so the off-diagonal covariance is generated by a circular convolution of
independent HRR vectors. Consequently, the aggregate off-diagonal covariance
vanishes with $d$.

\begin{proposition}[Approximate slot independence]
\label{prop:slot-indep}
Let $q(\tilde{\mathbf{v}}_1,\dots,\tilde{\mathbf{v}}_m \mid \mathbf{x})$ be the joint distribution of unbound
retrievals under the HRR noise model of Appendix~\ref{app:retrieval}, and let
$q_i(\tilde{\mathbf{v}}_i \mid \mathbf{x})$ be its marginals. With posterior
variance floor $t^2>0$, 
\begin{equation}
D_{\mathrm{KL}}\!\bigl(q(\tilde{\mathbf{V}}_{1:m}\mid \mathbf{x})\,\big\|\,
\textstyle\prod_{i=1}^{m}q_i(\tilde{\mathbf{v}}_i\mid \mathbf{x})\bigr)
\;=\;
\mathcal{O}\!\left(\frac{m^2}{t^4 d}\right),
\label{eq:slot-indep}
\end{equation}
which vanishes as $d\to\infty$ at fixed $m$ and fixed $t>0$.
\end{proposition}

\vspace{-2mm}
\begin{proof}
The chain rule identifies the KL divergence above with the total correlation
among the retrieved slots,
\[
D_{\mathrm{KL}}\!\bigl(q(\tilde{\mathbf{V}}_{1:m}\mid \mathbf{x})\,\big\|\,
\textstyle\prod_i q_i(\tilde{\mathbf{v}}_i\mid \mathbf{x})\bigr)
=
\sum_{i=2}^{m}
I(\tilde{\mathbf{v}}_i;\tilde{\mathbf{V}}_{<i}\mid \mathbf{x}).
\]
We bound this total correlation directly rather than bounding each conditional
mutual information separately. Under the Gaussian approximation to the joint
retrieval law, the KL between the joint distribution and the product of its
marginals has a closed-form log-determinant expression. Writing
$\bar{\Sigma}$ for the block-diagonal covariance obtained by removing the
off-diagonal covariance blocks, and $\Delta=\Sigma-\bar{\Sigma}$ for the
off-diagonal part, the trace term cancels because $\Delta$ has zero diagonal
blocks. After block-diagonal whitening, the KL reduces to a log-determinant involving
\[
\mat{M}=(\bar{\Sigma}+t^2 \mat{I})^{-1/2}\Delta(\bar{\Sigma}+t^2 \mat{I})^{-1/2}.
\]
The posterior variance floor gives $\|(\bar{\Sigma}+t^2 \mat{I})^{-1/2}\|_{\mathrm{op}}\leq t^{-1}$, while the HRR off-diagonal block estimates give $\|\Delta\|_F^2=\mathcal{O}(m^2/d)$. Therefore $\|\mat{M}\|_F^2=\mathcal{O}(m^2/(t^4d))$. A second-order expansion of $-\log\det(\mat{I}+\mat{M})$, together with the small-operator-norm condition on $\mat{M}$, then yields the stated bound. See Appendix~\ref{app:slot_independence_proof} for the full derivation.
\vspace{-2mm}
\end{proof}
Together, Propositions~\ref{prop:capacity} and~\ref{prop:slot-indep} give a partial answer to the question of \emph{why} the HRR structure favors disentangled representations. Two mechanisms act on the encoder. First, slots are statistically near-independent under unbinding (Proposition~\ref{prop:slot-indep}), so the encoder can optimize each slot largely independently of the others, rather than coping with cross-slot noise correlations. Second, attaining the capacity bound of Proposition~\ref{prop:capacity} requires non-redundant slot usage by placing distinct generative factors in distinct slots. Quantization then discretizes each slot's content and pushes the encoder toward a finite, factor-aligned set of codes.

\section{Experimental results and analysis}
\label{sec:experimental_results_analysis}

In the experiments, we will quantify the disentanglement performance of our proposed HRR model compared to previous models proposed for disentanglement, show that different slots in our representation correspond to different generative factors, and show the robustness to noise of our HRR representation. All together, these results demonstrate that HRR-inspired representations can achieve good disentanglement quality while offering a robustness to noise not seen in prior disentanglement works.

\subsection{Disentanglement performance} \label{sec:disentanglement_experiments}

We qualitatively measure our model's disentanglement performance using the InfoMEC metrics, proposed in a prior work~\cite{hsu2023disentanglement}, because it uses information-theoretic metrics, mutual information, to measure disentanglement quality as opposed to using feature importance from learned classifiers and regression models, as done in DCI, an older disentanglement metric~\cite{eastwood2018a}. However we still report the DCI scores to complement our InfoMEC scores. For consistency we use the same baselines that were proposed \cite{hsu2023disentanglement}. Specifically, our baselines include, the $\beta$-VAE \cite{higgins2017betavae}, the TC-VAE \cite{chenIsolatingSourcesDisentanglement2018}, the VQ-VAE \cite{vandenoordNeuralDiscreteRepresentation2017}, and QLAE \cite{hsu2023disentanglement}. Consistent with \cite{hsu2023disentanglement}, we use Shapes3D \cite{3dshapes18a}, MPI3D \cite{NEURIPS2019_d97d404b} (specifically the complex real world shape variant), and NVIDIA's Falcor3D and Isaac3D datasets \cite{NVlabsHighresdisentanglementdatasets2025a}. Consistent with prior works\cite{hsu2023disentanglement}, we sweep a single hyperparameter for each (model, dataset) pair and use the best performing hyperparameter over two seeds. For more details on the sweep, we refer to Table~\ref{tab:baselines_sweep_params} of Appendix~\ref{app:experimental-setup}. We note the best performing hyperparameter and re-run each model on each dataset with its best hyperparameter 5 times from which we construct a confidence interval. As is consistent with other disentanglement works, we train on the entire dataset without holding out a validation or test set~\cite{JMLR:v21:19-976}. The purpose of this is that, in disentanglement, we are mostly concerned with how the model can learn to disentangle the different latent factors, not combating overfitting. 

For each run, we sample a random subset of the dataset (without removing it from the training data) that we use to evaluate the disentanglement metrics. To ensure fairness, each model uses the same encoder and decoder architecture, and the same optimization parameters, consistent with the work of \cite{JMLR:v21:19-976}. For each dataset, each model is initialized with 1.5$\times$ latent units as there are sources. More details, such as model architecture and optimization hyperparameters can be  found in Table~\ref{tab:cnn_architecture} and Table~\ref{tab:optim_hparams} in Appendix~\ref{app:experimental-setup}. Details on the bolding criteria for the results can also be found in Appendix~\ref{app:experimental-setup}.

{
\setlength{\tabcolsep}{6pt}
\begin{table}[t]
    \caption{Disentanglement scores over $n=5$ seeds. InfoMEC cells show (InfoM InfoE InfoC); bold marks the best per column using the CI-overlap criterion. Aggregated is the mean across all datasets.}
    \label{tab:main_results}
    \centering
    \scriptsize
    \begin{tabular}{lccccc}
        \toprule
        Model & Aggregated & Shapes3D & Falcor3D & Isaac3D & MPI3D-C \\
        \cmidrule(lr){2-6}
        & \multicolumn{5}{c}{InfoMEC $:=$ (InfoM InfoE InfoC) $\uparrow$} \\
        \midrule

$\beta$-VAE & $(0.54\ 0.57\ 0.45)$ & $(0.59\ 0.81\ 0.47)$ & $(\mathbf{0.64}\ 0.68\ \mathbf{0.59})$ & $(0.54\ 0.56\ 0.46)$ & $(0.37\ 0.23\ 0.29)$ \\
$\beta$-TCVAE & $(0.56\ 0.63\ 0.51)$ & $(0.61\ 0.83\ 0.51)$ & $(\mathbf{0.59}\ 0.65\ \mathbf{0.55})$ & $(0.63\ 0.62\ \mathbf{0.57})$ & $(0.41\ 0.42\ 0.39)$ \\
VQ-VAE & $(0.57\ \mathbf{0.88}\ 0.47)$ & $(0.58\ \mathbf{1.00}\ 0.44)$ & $(\mathbf{0.50}\ \mathbf{0.86}\ 0.40)$ & $(0.70\ \mathbf{0.86}\ 0.52)$ & $(0.49\ \mathbf{0.82}\ \mathbf{0.53})$ \\
QLAE & $(0.65\ 0.84\ 0.49)$ & $(0.77\ \mathbf{1.00}\ 0.52)$ & $(\mathbf{0.61}\ 0.79\ \mathbf{0.46})$ & $(0.66\ 0.81\ 0.49)$ & $(\mathbf{0.56}\ 0.78\ 0.48)$ \\
HRR (ours) & $(\mathbf{0.68}\ 0.88\ \mathbf{0.54})$ & $(\mathbf{0.85}\ \mathbf{1.00}\ \mathbf{0.61})$ & $(\mathbf{0.63}\ 0.81\ \mathbf{0.45})$ & $(\mathbf{0.73}\ \mathbf{0.89}\ \mathbf{0.55})$ & $(\mathbf{0.51}\ \mathbf{0.83}\ \mathbf{0.53})$
 \\
        \midrule
        & \multicolumn{5}{c}{DCI $:=$ (D I C) $\uparrow$} \\
        \midrule

$\beta$-VAE & $(0.31\ 0.28\ 0.85)$ & $(0.45\ \mathbf{0.38}\ \mathbf{0.97})$ & $(0.23\ 0.22\ 0.77)$ & $(0.28\ 0.25\ 0.86)$ & $(0.29\ \mathbf{0.28}\ \mathbf{0.79})$ \\
$\beta$-TCVAE & $(0.31\ 0.29\ 0.85)$ & $(0.40\ \mathbf{0.33}\ 0.95)$ & $(0.44\ \mathbf{0.43}\ \mathbf{0.86})$ & $(0.29\ 0.26\ \mathbf{0.86})$ & $(0.12\ 0.12\ 0.73)$ \\
VQ-VAE & $(0.62\ 0.34\ 0.86)$ & $(0.71\ \mathbf{0.35}\ \mathbf{0.96})$ & $(\mathbf{0.63}\ 0.33\ 0.83)$ & $(\mathbf{0.64}\ \mathbf{0.36}\ \mathbf{0.91})$ & $(0.49\ \mathbf{0.34}\ 0.73)$ \\
QLAE & $(0.59\ 0.31\ 0.88)$ & $(0.76\ \mathbf{0.35}\ \mathbf{0.99})$ & $(0.50\ 0.28\ 0.84)$ & $(0.59\ 0.31\ \mathbf{0.92})$ & $(0.51\ \mathbf{0.29}\ \mathbf{0.78})$ \\
HRR (ours) & $(\mathbf{0.67}\ \mathbf{0.37}\ \mathbf{0.89})$ & $(\mathbf{0.88}\ \mathbf{0.41}\ \mathbf{1.00})$ & $(0.58\ \mathbf{0.37}\ \mathbf{0.85})$ & $(\mathbf{0.68}\ \mathbf{0.36}\ \mathbf{0.92})$ & $(\mathbf{0.55}\ \mathbf{0.34}\ \mathbf{0.77})$
 \\
        \bottomrule
    \end{tabular}
    \vspace{5pt}
\end{table}
}

As seen in Table~\ref{tab:main_results}, our HRR based disentanglement method outperforms most baselines across the InfoM metric, only surpassed by QLAE on the MPI3D-C dataset, with competitive performance on Falcor3D with $\beta$-VAE. Our model is also able to match or exceed baselines on InfoE and InfoC scores all datasets except on Falcor3D for the InfoE and InfoC metrics, where the VQ-VAE and $\beta$-TCVAE reach the highest scores, respectively. However, in aggregate, our method outperforms, or meets, other baselines across all metrics, providing a 4.6\% improvement in InfoM, and 5.9\% improvement in InfoC, over the second highest scores in the respective categories. The pattern across metrics is consistent with the analysis of
Section~\ref{sec:capacity}. HRR's strongest gains are on InfoM (modularity),
where the per-slot independence induced by unbinding
(Proposition~\ref{prop:slot-indep}) directly biases the encoder toward placing
distinct factors in distinct slots. The smaller gains on InfoE and InfoC,
particularly on Falcor3D where our proposed model is outperformed by VQ-VAE and $\beta$-TCVAE, indicate that
explicitness and compactness are governed less by the slot structure and
more by the choice of decoder and codebook size.

\subsection{Latent robustness to noise when decoding} \label{ssec:robustness_experiments}

Another quantity that we measure is the robustness to noise of the different latent representations. This provides insight into how much a latent representation could be corrupted by external noise while still being able to be accurately decoded. We posit that robust representations will play an important role in the upcoming era of physical AI, as we deploy machine learning algorithms into the real world. Research on representations has largely taken place in controlled settings, and has ignored the challenges associated with real world use. Noisy representations will be an important challenge to overcome, and should be learned in such a way that they can tolerate corruption. For instance, two embodied agents wishing to collaborate should be able to understand each other despite corruption in a message, especially in critical scenarios. Robust representations are also of interest in latent-based world models, since preserving the semantic content of internal representations is crucial for their performance. A more robust representation could allow for more error toleration when making predictions in latent space since the result may still be able to be decoded accurately. To this end, we measure the reconstruction peak-signal-to-noise-ratio (PSNR) of representations under zero mean Gaussian noise, at varying signal-to-noise (SNR) ratios. We apply noise to the latent representation from +20 dB SNR down to -20 dB SNR and measure the PSNR between the resulting reconstruction and the input. As seen in Fig.~\ref{fig:gaussian_noise_plot}, the proposed HRR model and VQ-VAE models both outperform the other baselines, with the former performing the best. This is because VQ-VAE's optimization allows for codebook entries to spread apart in
Euclidean space. A noise perturbation must therefore be large enough
to push a quantized vector closer to a different codebook entry before
the discrete code changes, leaving the recovered code unchanged for
moderate noise levels. A visualization of the reconstruction performance can be seen in Fig.~\ref{fig:gaussian_noise_grid_3dshapes}, in which we see that the HRR and VQ-VAE both are able to maintain their semantic content under most noise intensities. We include more results for other datasets in Appendix~\ref{app:additional_noise_results}.

\begin{figure}[t]
    \centering
    \includegraphics[width=0.9\linewidth]{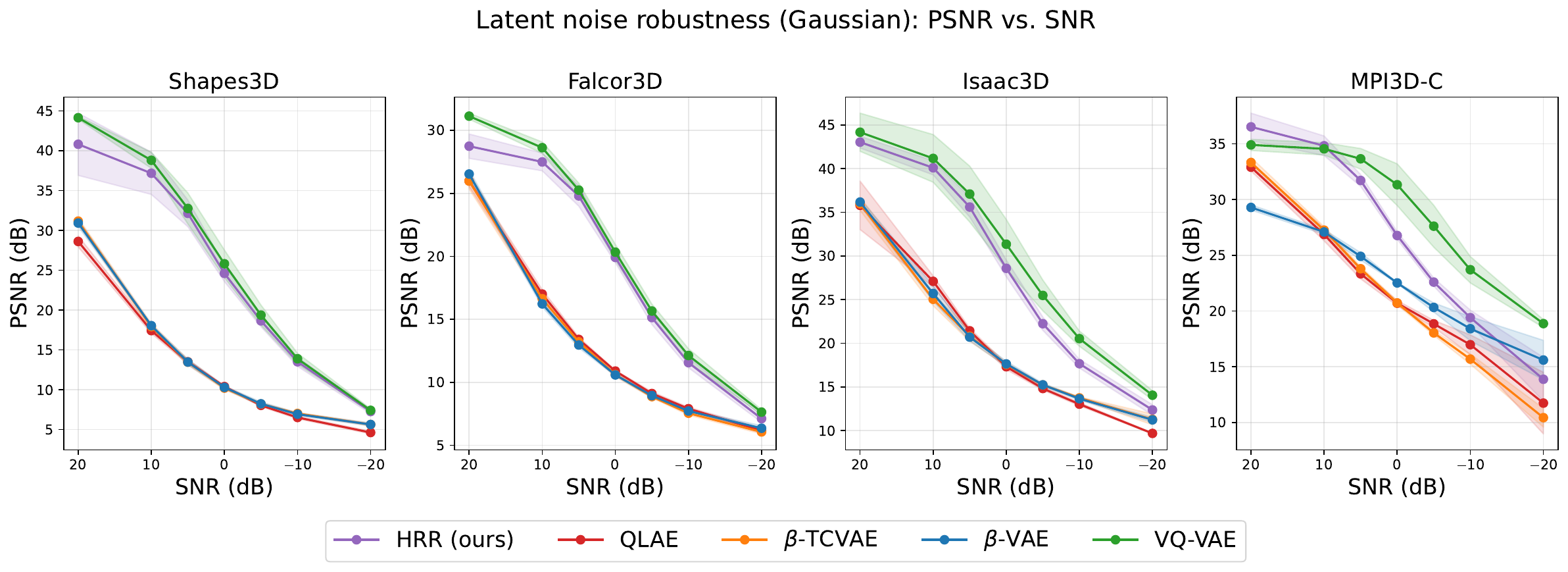}
    \caption{Average performance of each model, as measured by PSNR, under varying levels of Gaussian noise intensity. The gap between HRR and VQ-VAE is small across most datasets, and both models perform considerably better than the other baselines. Shaded regions represent a 95\% confidence interval.}
    \label{fig:gaussian_noise_plot}
\vspace{-3mm}\end{figure}

\newpage

\begin{figure}[t]
    \centering
    \includegraphics[width=0.8\linewidth]{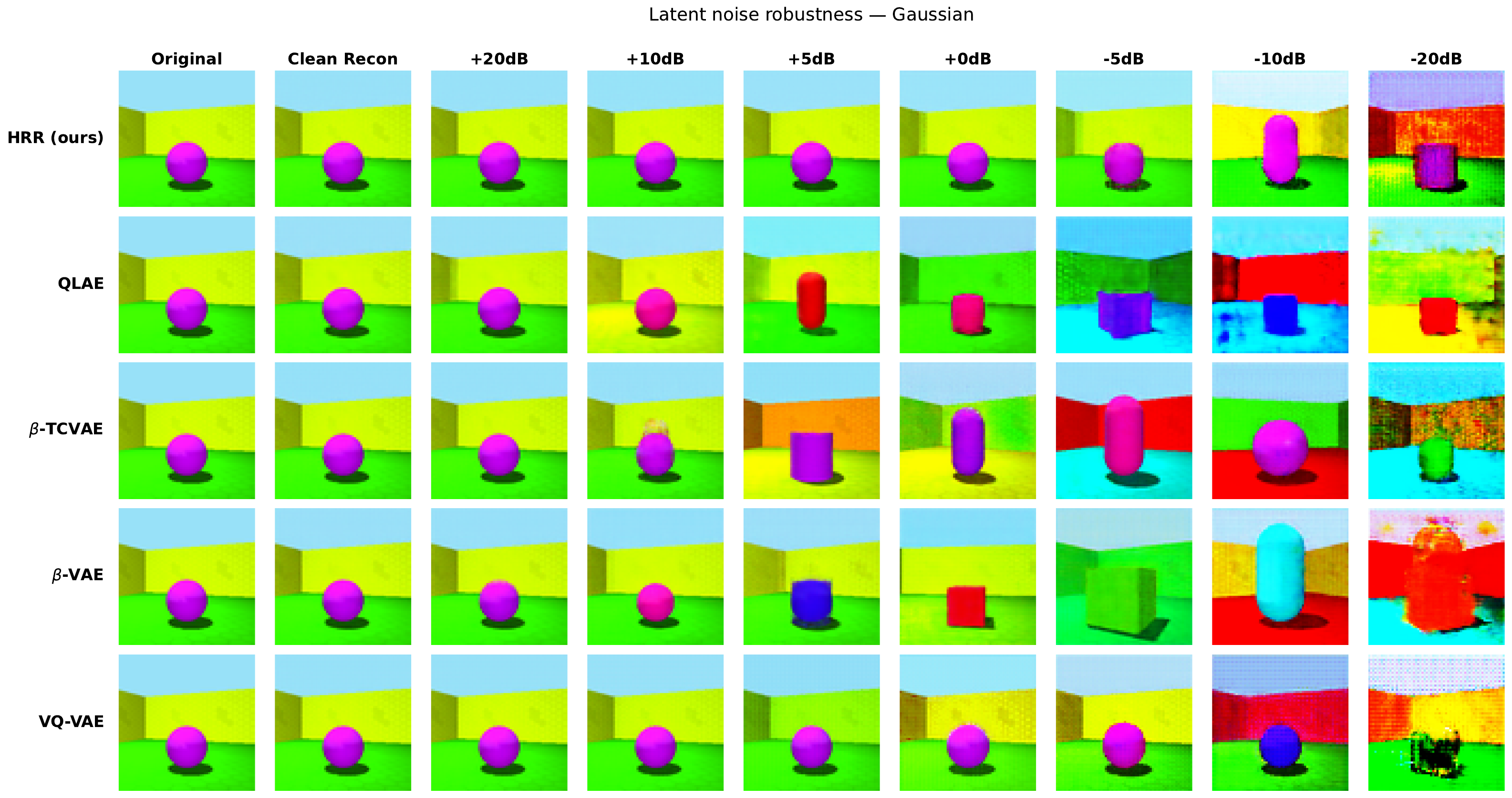}
    \caption{Each row is a visualization of how reconstruction quality degrades as noise intensity increases for each model. For each model, the seed with the highest InfoM score was used for evaluation. The HRR model and VQ-VAE maintain their semantic content the best. Signal power is estimated over the entire validation set.}
    \label{fig:gaussian_noise_grid_3dshapes}
\end{figure}
\vspace*{-4em}
\subsection{Latent component swapping}\label{sec:swaps}
Because of the vector quantization process that assigns value vectors to each slot in our latent variables, traditional latent traversals (where scalar dimensions are varied through an interval) cannot be used as a fair qualitative method of assessment for our method. Instead, we opt to show the results of what happens when you swap latent units between two representations from the dataset, i.e, inject a latent factor from one representation into the other. For truly disentangled representations, it would be reasonable to expect a degree of modularity (in the qualitative sense, not the quantitative one used as a metric) between the representations, such that the latent units are compositional. 

Fig.~\ref{fig:component_swap_3dshapes} compares latent component swaps
for HRR and other baselines. We make the following observations here. First, HRR's swaps consistently modify a single
generative factor at a time while leaving others intact. Where QLAE's
swaps occasionally produce visible artifacts in unrelated factors, HRR
maintains the unswapped attributes more cleanly. Second, HRR preserves
global scene semantics across nearly all swap positions, while QLAE
occasionally alters the orientation of objects during a swap. The concentration of changes within a single generative factor follows from the approximate slot independence established in Proposition~\ref{prop:slot-indep}. If the per-slot
retrievals were tightly coupled, swapping one slot would propagate
changes across others, but HRR's slot-by-slot swapping remain contained. From Fig.~\ref{fig:component_swap_3dshapes}, we can make two key observations.  First, these results provide qualitative support for Proposition~\ref{prop:slot-indep}. If the per-slot retrievals were tightly coupled, swapping one component would propagate changes across the others, but HRR's swaps largely preserve unrelated factors. Second, the qualitative gap with QLAE is informative because both methods quantize their latents but only HRR imposes a symbolic binding structure, suggesting that the binding-and-bundling architecture may be important to enable single-factor edits.

\begin{figure}[t]
    \centering
    \includegraphics[width=0.8\linewidth]{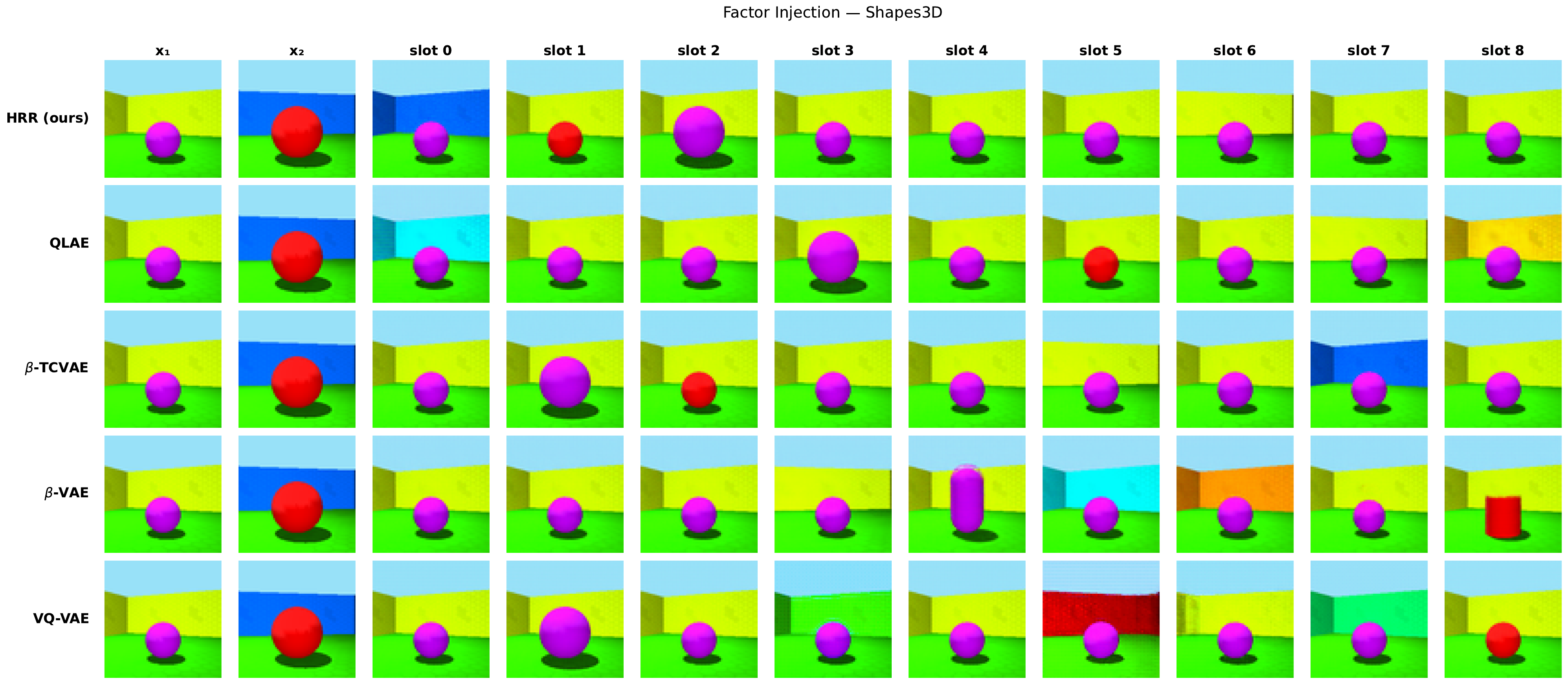}
    \caption{Latent component swaps performance for highest performing model relative to InfoM score. The slot 0 column represents taking the first latent unit from $x_2$, using it for the first latent unit for $x_1$, and decoding the result. Slot 1 does the same, but for the second latent unit, and so on. Each slot is replaced with its original value after visualization, so we expect each decoded result to manipulate either one, or no, aspects of the image.}
    \label{fig:component_swap_3dshapes}
\vspace{-2mm}
\end{figure}

\vspace{-2mm}
\subsection{Latent interpolation}

A natural extension of component swapping is to swap slots cumulatively rather
than individually. Starting from a source image's latent $\mathbf{z}_1$, we
progressively replace slots with the corresponding slots from a target image's
latent $\mathbf{z}_2$. At step $i$, slots $1$ through $i$ have been swapped from
$\mathbf{z}_2$. For a disentangled representation, the decoded sequence should
walk from the source image to the target one factor at a time, with each step
changing only the most recently swapped factor.
Figure~\ref{fig:interpolation_3dshapes} visualizes this for the HRR-based model and all the baselines, evaluated on the same image. Cumulative swapping is a more stringent test of slot independence than the single-component swap experiment of Section~\ref{sec:swaps}.  In this scenario, errors from earlier swaps could distort later ones if slots were entangled, but each step in the HRR-based model's sequence changes one factor cleanly, providing qualitative evidence consistent with Proposition~\ref{prop:slot-indep}. We note that this is the natural notion of interpolation for a quantized symbolic latent, since smooth blending of codebook entries is not meaningful and a discrete walk through factor configurations is what compositional structure should support.

\begin{figure}[t]
    \centering
    \includegraphics[width=0.9\linewidth]{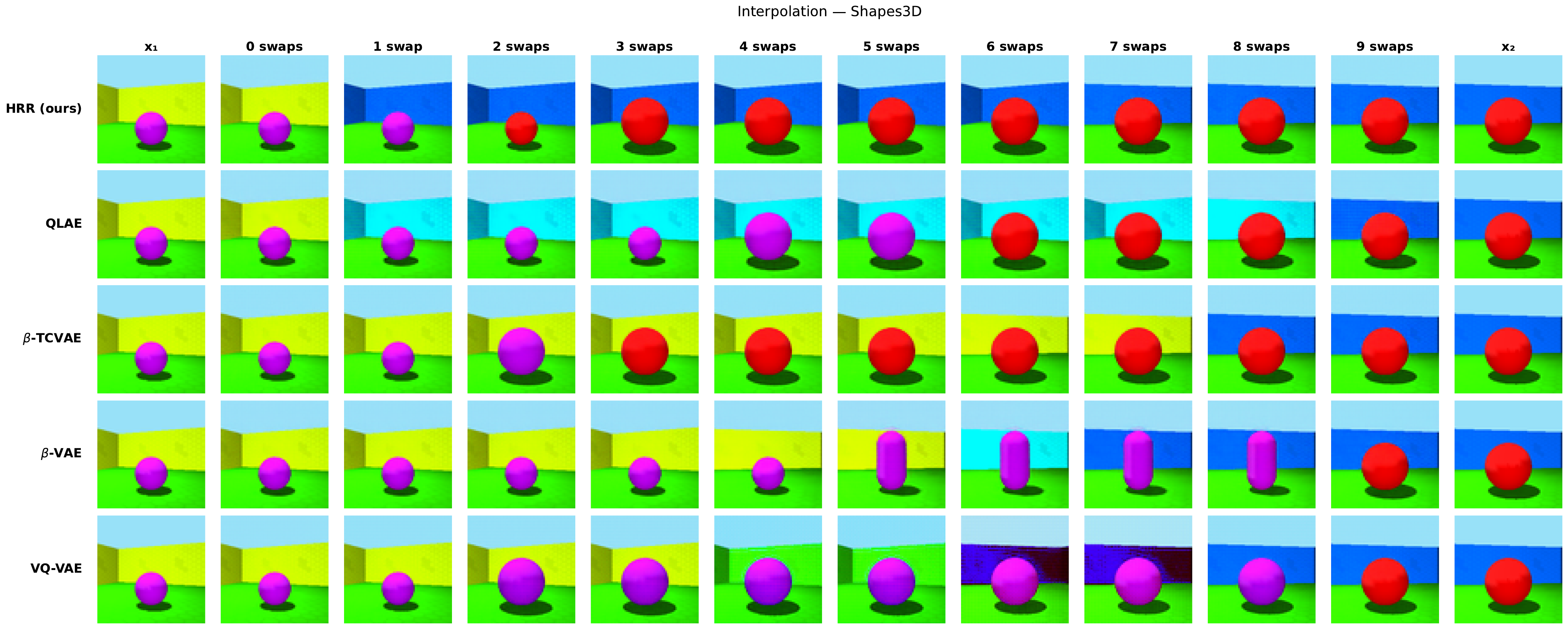}
    \caption{Latent interpolation swaps components progressively until the representation is fully transformed into that of the second image. We see that in the interpolation, the HRR-based model does not introduce any extraneous generative factors that were not present in the original two images. For instance, no decoded results for the HRR introduce a light blue wall, as the QLAE and $\beta$-VAE do.}
    \label{fig:interpolation_3dshapes}
\vspace{-1.0em}\end{figure}

\vspace{-4mm}\section{Discussion}
\label{sec:discussion}\vspace{-2mm}

In this paper, we have proposed using HRRs for disentanglement and shown that the resulting representations match or exceed strong baselines on the InfoMEC and DCI metrics across four standard datasets while being more robust to additive latent noise than every disentanglement-focused baseline. To the best of our knowledge, this is the first work to use HRRs for disentanglement without supervision, and the first to train a neural network to produce HRR-structured latents directly from data.
Sections~\ref{sec:capacity} and~\ref{sec:slot-indep} provide a partial mechanism for why the HRR structure favors disentanglement.  Approximate slot independence under the unbinding channel (Proposition~\ref{prop:slot-indep}) lets the encoder optimize each slot near-independently, while the per-slot capacity bound (Proposition~\ref{prop:capacity}) penalizes redundant slot usage. Together with the quantization inductive bias identified in~\cite{hsu2023disentanglement}, these two pressures favor placing distinct generative factors in distinct slots. These results are a good foundation for explaining the disentanglement that emerges from the proposed HRR-based model. Tighter constants, a formal connection to identifiability, and an empirical study of the $d/m$ ratio are natural next steps that would further validate these results. Beyond disentanglement, our results indicate that VSA-inspired representations are worth investigating more broadly in the deep learning field. The robustness to latent noise observed here may be particularly relevant for emerging use cases, such as physical-AI. In physical settings, representations must tolerate the noise that is fundamental in sensors and wireless channels. Robustness will be an important factor in ensuring representations remain useful despite corruption. Future work could explore whether the symbolic structure of HRRs translates into compositional generalization on factor combinations unseen during training, which prior work has noted as a weakness of standard disentanglement methods~\cite{montero2021the}.

\newpage

\bibliography{references}

@inproceedings{10.5555/3692070.3692839,
  title = {Tripod: Three Complementary Inductive Biases for Disentangled Representation Learning},
  booktitle = {Proceedings of the 41st International Conference on Machine Learning},
  author = {Hsu, Kyle and Hamid, Jubayer Ibn and Burns, Kaylee and Finn, Chelsea and Wu, Jiajun},
  year = 2024,
  series = {{{ICML}}'24},
  publisher = {JMLR.org},
  address = {Vienna, Austria},
  articleno = {769}
}

@misc{3dshapes18a,
  title = {{{3D}} Shapes Dataset},
  author = {Burgess, Chris and Kim, Hyunjik},
  year = 2018
}

@inproceedings{alamRecastingSelfAttentionHolographic2023a,
  title = {Recasting {{Self-Attention}} with {{Holographic Reduced Representations}}},
  booktitle = {Proceedings of the 40th {{International Conference}} on {{Machine Learning}}},
  author = {Alam, Mohammad Mahmudul and Raff, Edward and Biderman, Stella and Oates, Tim and Holt, James},
  year = 2023,
  month = jul,
  pages = {490--507},
  publisher = {PMLR},
  issn = {2640-3498},
  urldate = {2026-05-07}
}

@inproceedings{assranSelfSupervisedLearningImages2023a,
  title = {Self-{{Supervised Learning}} from {{Images}} with a {{Joint-Embedding Predictive Architecture}}},
  booktitle = {2023 {{IEEE}}/{{CVF Conference}} on {{Computer Vision}} and {{Pattern Recognition}} ({{CVPR}})},
  author = {Assran, Mahmoud and Duval, Quentin and Misra, Ishan and Bojanowski, Piotr and Vincent, Pascal and Rabbat, Michael and LeCun, Yann and Ballas, Nicolas},
  year = 2023,
  month = jun,
  pages = {15619--15629},
  issn = {2575-7075},
  doi = {10.1109/CVPR52729.2023.01499},
  urldate = {2026-05-07}
}

@misc{bengio2013estimatingpropagatinggradientsstochastic,
  title = {Estimating or Propagating Gradients through Stochastic Neurons for Conditional Computation},
  author = {Bengio, Yoshua and L{\'e}onard, Nicholas and Courville, Aaron},
  year = 2013,
  eprint = {1308.3432},
  primaryclass = {cs.LG},
  archiveprefix = {arXiv}
}

@article{bengioRepresentationLearningReview2013,
  title = {Representation {{Learning}}: {{A Review}} and {{New Perspectives}}},
  shorttitle = {Representation {{Learning}}},
  author = {Bengio, Y. and Courville, A. and Vincent, P.},
  year = 2013,
  month = aug,
  journal = {IEEE Transactions on Pattern Analysis and Machine Intelligence},
  volume = {35},
  number = {8},
  pages = {1798--1828},
  issn = {0162-8828, 2160-9292},
  doi = {10.1109/TPAMI.2013.50},
  urldate = {2026-05-07},
  copyright = {https://ieeexplore.ieee.org/Xplorehelp/downloads/license-information/IEEE.html}
}

@inproceedings{caronEmergingPropertiesSelfSupervised2021b,
  title = {Emerging {{Properties}} in {{Self-Supervised Vision Transformers}}},
  booktitle = {2021 {{IEEE}}/{{CVF International Conference}} on {{Computer Vision}} ({{ICCV}})},
  author = {Caron, Mathilde and Touvron, Hugo and Misra, Ishan and Jegou, Herv{\'e} and Mairal, Julien and Bojanowski, Piotr and Joulin, Armand},
  year = 2021,
  month = oct,
  pages = {9630--9640},
  issn = {2380-7504},
  doi = {10.1109/ICCV48922.2021.00951},
  urldate = {2026-05-07}
}

@inproceedings{chenIsolatingSourcesDisentanglement2018,
  title = {Isolating {{Sources}} of {{Disentanglement}} in {{Variational Autoencoders}}},
  booktitle = {Advances in {{Neural Information Processing Systems}}},
  author = {Chen, Ricky T. Q. and Li, Xuechen and Grosse, Roger B and Duvenaud, David K},
  year = 2018,
  volume = {31},
  publisher = {Curran Associates, Inc.},
  urldate = {2026-05-07}
}

@inproceedings{chenSimpleFrameworkContrastive2020a,
  title = {A {{Simple Framework}} for {{Contrastive Learning}} of {{Visual Representations}}},
  booktitle = {Proceedings of the 37th {{International Conference}} on {{Machine Learning}}},
  author = {Chen, Ting and Kornblith, Simon and Norouzi, Mohammad and Hinton, Geoffrey},
  year = 2020,
  month = nov,
  pages = {1597--1607},
  publisher = {PMLR},
  issn = {2640-3498},
  urldate = {2026-05-07}
}

@inproceedings{DBLP:journals/corr/RadfordMC15,
  title = {Unsupervised Representation Learning with Deep Convolutional Generative Adversarial Networks},
  booktitle = {4th International Conference on Learning Representations, {{ICLR}} 2016, San Juan, Puerto Rico, May 2-4, 2016, Conference Track Proceedings},
  author = {Radford, Alec and Metz, Luke and Chintala, Soumith},
  editor = {Bengio, Yoshua and LeCun, Yann},
  year = 2016,
  bibsource = {dblp computer science bibliography, https://dblp.org}
}

@inproceedings{eastwood2018a,
  title = {A Framework for the Quantitative Evaluation of Disentangled Representations},
  booktitle = {International Conference on Learning Representations},
  author = {Eastwood, Cian and Williams, Christopher K. I.},
  year = 2018
}

@inproceedings{ganesanLearningHolographicReduced2021a,
  title = {Learning with {{Holographic Reduced Representations}}},
  booktitle = {Advances in {{Neural Information Processing Systems}}},
  author = {Ganesan, Ashwinkumar and Gao, Hang and Gandhi, Sunil and Raff, Edward and Oates, Tim and Holt, James and McLean, Mark},
  year = 2021,
  volume = {34},
  pages = {25606--25620},
  publisher = {Curran Associates, Inc.},
  urldate = {2026-05-07}
}

@inproceedings{goodfellowGenerativeAdversarialNets2014,
  title = {Generative {{Adversarial Nets}}},
  booktitle = {Advances in {{Neural Information Processing Systems}}},
  author = {Goodfellow, Ian J. and {Pouget-Abadie}, Jean and Mirza, Mehdi and Xu, Bing and {Warde-Farley}, David and Ozair, Sherjil and Courville, Aaron and Bengio, Yoshua},
  year = 2014,
  volume = {27},
  publisher = {Curran Associates, Inc.},
  urldate = {2026-05-07}
}

@inproceedings{grillBootstrapYourOwn2020a,
  title = {Bootstrap {{Your Own Latent}} - {{A New Approach}} to {{Self-Supervised Learning}}},
  booktitle = {Advances in {{Neural Information Processing Systems}}},
  author = {Grill, Jean-Bastien and Strub, Florian and Altch{\'e}, Florent and Tallec, Corentin and Richemond, Pierre and Buchatskaya, Elena and Doersch, Carl and Avila Pires, Bernardo and Guo, Zhaohan and Gheshlaghi Azar, Mohammad and Piot, Bilal and {kavukcuoglu}, koray and Munos, Remi and Valko, Michal},
  year = 2020,
  volume = {33},
  pages = {21271--21284},
  publisher = {Curran Associates, Inc.},
  urldate = {2026-05-07}
}

@inproceedings{heMaskedAutoencodersAre2022,
  title = {Masked {{Autoencoders Are Scalable Vision Learners}}},
  booktitle = {2022 {{IEEE}}/{{CVF Conference}} on {{Computer Vision}} and {{Pattern Recognition}} ({{CVPR}})},
  author = {He, Kaiming and Chen, Xinlei and Xie, Saining and Li, Yanghao and Doll{\'a}r, Piotr and Girshick, Ross},
  year = 2022,
  month = jun,
  pages = {15979--15988},
  issn = {2575-7075},
  doi = {10.1109/CVPR52688.2022.01553},
  urldate = {2026-05-07}
}

@article{herscheNeurovectorsymbolicArchitectureSolving2023a,
  title = {A Neuro-Vector-Symbolic Architecture for Solving {{Raven}}'s Progressive Matrices},
  author = {Hersche, Michael and Zeqiri, Mustafa and Benini, Luca and Sebastian, Abu and Rahimi, Abbas},
  year = 2023,
  month = mar,
  journal = {Nature Machine Intelligence},
  volume = {5},
  number = {4},
  pages = {363--375},
  issn = {2522-5839},
  doi = {10.1038/s42256-023-00630-8},
  urldate = {2026-05-07}
}

@inproceedings{higgins2017betavae,
  title = {Beta-{{VAE}}: {{Learning}} Basic Visual Concepts with a Constrained Variational Framework},
  booktitle = {International Conference on Learning Representations},
  author = {Higgins, Irina and Matthey, Loic and Pal, Arka and Burgess, Christopher and Glorot, Xavier and Botvinick, Matthew and Mohamed, Shakir and Lerchner, Alexander},
  year = 2017
}

@inproceedings{hsu2023disentanglement,
  title = {Disentanglement via Latent Quantization},
  booktitle = {Thirty-Seventh Conference on Neural Information Processing Systems},
  author = {Hsu, Kyle and Dorrell, Will and Whittington, James C. R. and Wu, Jiajun and Finn, Chelsea},
  year = 2023
}

@article{JMLR:v21:19-976,
  title = {A Sober Look at the Unsupervised Learning of Disentangled Representations and Their Evaluation},
  author = {Locatello, Francesco and Bauer, Stefan and Lucic, Mario and Raetsch, Gunnar and Gelly, Sylvain and Sch{\"o}lkopf, Bernhard and Bachem, Olivier},
  year = 2020,
  journal = {Journal of Machine Learning Research},
  volume = {21},
  number = {209},
  pages = {1--62}
}

@misc{kingmaAutoEncodingVariationalBayes2022a,
  title = {Auto-{{Encoding Variational Bayes}}},
  author = {Kingma, Diederik P. and Welling, Max},
  year = 2022,
  month = dec,
  number = {arXiv:1312.6114},
  eprint = {1312.6114},
  primaryclass = {stat},
  publisher = {arXiv},
  doi = {10.48550/arXiv.1312.6114},
  urldate = {2026-05-07},
  archiveprefix = {arXiv}
}

@inproceedings{chenInfoGANInterpretableRepresentation2016a,
  title = {{{InfoGAN}}: {{Interpretable Representation Learning}} by {{Information Maximizing Generative Adversarial Nets}}},
  shorttitle = {{{InfoGAN}}},
  booktitle = {Advances in {{Neural Information Processing Systems}}},
  author = {Chen, Xi and Duan, Yan and Houthooft, Rein and Schulman, John and Sutskever, Ilya and Abbeel, Pieter},
  year = 2016,
  volume = {29},
  publisher = {Curran Associates, Inc.},
  urldate = {2026-05-07}
}

@inproceedings{10.1145/3195970.3196060,
  title = {Hierarchical Hyperdimensional Computing for Energy Efficient Classification},
  booktitle = {Proceedings of the 55th Annual Design Automation Conference},
  author = {Imani, Mohsen and Huang, Chenyu and Kong, Deqian and Rosing, Tajana},
  year = 2018,
  series = {Dac '18},
  publisher = {Association for Computing Machinery},
  address = {New York, NY, USA},
  doi = {10.1145/3195970.3196060},
  articleno = {108},
  isbn = {978-1-4503-5700-5}
}

@article{kleykoSurveyHyperdimensionalComputing2022,
  title = {A {{Survey}} on {{Hyperdimensional Computing}} Aka {{Vector Symbolic Architectures}}, {{Part I}}: {{Models}} and {{Data Transformations}}},
  shorttitle = {A {{Survey}} on {{Hyperdimensional Computing}} Aka {{Vector Symbolic Architectures}}, {{Part I}}},
  author = {Kleyko, Denis and Rachkovskij, Dmitri A. and Osipov, Evgeny and Rahimi, Abbas},
  year = 2022,
  month = dec,
  journal = {ACM Comput. Surv.},
  volume = {55},
  number = {6},
  pages = {130:1--130:40},
  issn = {0360-0300},
  doi = {10.1145/3538531},
  urldate = {2026-05-07}
}

@misc{korchemnyi2024symbolicdisentangledrepresentationsimages,
  title = {Symbolic Disentangled Representations for Images},
  author = {Korchemnyi, Alexandr and Kovalev, Alexey K. and Panov, Aleksandr I.},
  year = 2024,
  eprint = {2412.19847},
  primaryclass = {cs.CV},
  archiveprefix = {arXiv}
}

@inproceedings{loshchilov2018decoupled,
  title = {Decoupled Weight Decay Regularization},
  booktitle = {International Conference on Learning Representations},
  author = {Loshchilov, Ilya and Hutter, Frank},
  year = 2019
}

@inproceedings{montero2021the,
  title = {The Role of Disentanglement in Generalisation},
  booktitle = {International Conference on Learning Representations},
  author = {Montero, Milton Llera and Ludwig, Casimir JH and Costa, Rui Ponte and Malhotra, Gaurav and Bowers, Jeffrey},
  year = 2021
}

@misc{NVlabsHighresdisentanglementdatasets2025a,
  title = {{{NVlabs}}/{{High-res-disentanglement-datasets}}},
  year = 2025,
  month = nov,
  urldate = {2026-05-07},
  copyright = {CC-BY-4.0},
  howpublished = {NVIDIA Research Projects}
}

@inproceedings{paszkePyTorchImperativeStyle2019a,
  title = {{{PyTorch}}: {{An Imperative Style}}, {{High-Performance Deep Learning Library}}},
  shorttitle = {{{PyTorch}}},
  booktitle = {Advances in {{Neural Information Processing Systems}}},
  author = {Paszke, Adam and Gross, Sam and Massa, Francisco and Lerer, Adam and Bradbury, James and Chanan, Gregory and Killeen, Trevor and Lin, Zeming and Gimelshein, Natalia and Antiga, Luca and Desmaison, Alban and Kopf, Andreas and Yang, Edward and DeVito, Zachary and Raison, Martin and Tejani, Alykhan and Chilamkurthy, Sasank and Steiner, Benoit and Fang, Lu and Bai, Junjie and Chintala, Soumith},
  year = 2019,
  volume = {32},
  publisher = {Curran Associates, Inc.},
  urldate = {2026-05-07}
}

@article{plateHolographicReducedRepresentations1995a,
  title = {Holographic Reduced Representations},
  author = {Plate, T.A.},
  year = 1995,
  month = may,
  journal = {IEEE Transactions on Neural Networks},
  volume = {6},
  number = {3},
  pages = {623--641},
  issn = {1941-0093},
  doi = {10.1109/72.377968},
  urldate = {2026-05-07}
}

@inproceedings{pmlr-v80-kim18b,
  title = {Disentangling by Factorising},
  booktitle = {Proceedings of the 35th International Conference on Machine Learning},
  author = {Kim, Hyunjik and Mnih, Andriy},
  editor = {Dy, Jennifer and Krause, Andreas},
  year = 2018,
  month = jul,
  series = {Proceedings of Machine Learning Research},
  volume = {80},
  pages = {2649--2658},
  publisher = {PMLR}
}

@inproceedings{radfordLearningTransferableVisual2021a,
  title = {Learning {{Transferable Visual Models From Natural Language Supervision}}},
  booktitle = {Proceedings of the 38th {{International Conference}} on {{Machine Learning}}},
  author = {Radford, Alec and Kim, Jong Wook and Hallacy, Chris and Ramesh, Aditya and Goh, Gabriel and Agarwal, Sandhini and Sastry, Girish and Askell, Amanda and Mishkin, Pamela and Clark, Jack and Krueger, Gretchen and Sutskever, Ilya},
  year = 2021,
  month = jul,
  pages = {8748--8763},
  publisher = {PMLR},
  issn = {2640-3498},
  urldate = {2026-05-07}
}

@inproceedings{ridgewayLearningDeepDisentangled2018a,
  title = {Learning {{Deep Disentangled Embeddings With}} the {{F-Statistic Loss}}},
  booktitle = {Advances in {{Neural Information Processing Systems}}},
  author = {Ridgeway, Karl and Mozer, Michael C},
  year = 2018,
  volume = {31},
  publisher = {Curran Associates, Inc.},
  urldate = {2026-05-07}
}

@inproceedings{singh2023neural,
  title = {Neural Systematic Binder},
  booktitle = {The Eleventh International Conference on Learning Representations},
  author = {Singh, Gautam and Kim, Yeongbin and Ahn, Sungjin},
  year = 2023
}

@article{smetsEncodingFrameworkBinarized2024,
  title = {An Encoding Framework for Binarized Images Using Hyperdimensional Computing},
  author = {Smets, Laura and Van Leekwijck, Werner and Tsang, Ing Jyh and Latr{\'e}, Steven},
  year = 2024,
  month = jun,
  journal = {Frontiers in Big Data},
  volume = {7},
  publisher = {Frontiers},
  issn = {2624-909X},
  doi = {10.3389/fdata.2024.1371518},
  urldate = {2026-05-07}
}

@inproceedings{vandenoordNeuralDiscreteRepresentation2017,
  title = {Neural {{Discrete Representation Learning}}},
  booktitle = {Advances in {{Neural Information Processing Systems}}},
  author = {{van den Oord}, Aaron and Vinyals, Oriol and {kavukcuoglu}, koray},
  year = 2017,
  volume = {30},
  publisher = {Curran Associates, Inc.},
  urldate = {2026-05-07}
}

@article{rennerNeuromorphicVisualScene2024a,
    title = {Neuromorphic visual scene understanding with resonator networks},
    volume = {6},
    copyright = {2024 The Author(s), under exclusive licence to Springer Nature Limited},
    issn = {2522-5839},
    url = {https://www.nature.com/articles/s42256-024-00848-0},
    doi = {10.1038/s42256-024-00848-0},
    language = {en},
    number = {6},
    urldate = {2026-04-29},
    journal = {Nature Machine Intelligence},
    publisher = {Nature Publishing Group},
    author = {Renner, Alpha and Supic, Lazar and Danielescu, Andreea and Indiveri, Giacomo and Olshausen, Bruno A. and Sandamirskaya, Yulia and Sommer, Friedrich T. and Frady, E. Paxon},
    month = jun,
    year = {2024},
    pages = {641--652},
}

@misc{mikolovEfficientEstimationWord2013a,
    title = {Efficient {Estimation} of {Word} {Representations} in {Vector} {Space}},
    url = {http://arxiv.org/abs/1301.3781},
    doi = {10.48550/arXiv.1301.3781},
    urldate = {2026-04-13},
    publisher = {arXiv},
    author = {Mikolov, Tomas and Chen, Kai and Corrado, Greg and Dean, Jeffrey},
    month = sep,
    year = {2013},
    note = {arXiv:1301.3781 [cs]},
}

@inproceedings{alam2023towards,
    title = {Towards generalization in subitizing with neuro-symbolic loss using holographic reduced representations},
    url = {https://openreview.net/forum?id=AOAP8sLYdt},
    booktitle = {Neuro-symbolic learning and reasoning in the era of large language models},
    author = {Alam, Mohammad Mahmudul and Raff, Edward and Oates, Tim},
    year = {2023},
}

@inproceedings{NEURIPS2019_d97d404b,
    title = {On the transfer of inductive bias from simulation to the real world: a new disentanglement dataset},
    volume = {32},
    url = {https://proceedings.neurips.cc/paper/2019/file/d97d404b6119214e4a7018391195240a-Paper.pdf},
    booktitle = {Advances in neural information processing systems},
    publisher = {Curran Associates, Inc.},
    author = {Gondal, Muhammad Waleed and Wuthrich, Manuel and Miladinovic, Djordje and Locatello, Francesco and Breidt, Martin and Volchkov, Valentin and Akpo, Joel and Bachem, Olivier and Schölkopf, Bernhard and Bauer, Stefan},
    editor = {Wallach, H. and Larochelle, H. and Beygelzimer, A. and dAlché-Buc, F. and Fox, E. and Garnett, R.},
    year = {2019},
}
\bibliographystyle{unsrtnat}

\newpage

\appendix

\section{Additional details on model design} 

The denoising network described in Section~\ref{ssec:model_arch} is implemented as a single hidden layer feed forward neural network, with a hidden layer size that is $2\times$ the size of the input and that uses leaky relu activation. The codebook design was inspired from a prior work in slot attention~\cite{singh2023neural}. Specifically, we maintain $k$ "seed" vectors of dimension $d$, same as $\mat{z}$, and generate the codebook by processing each seed vector through a single hidden layer feed forward neural network with a hidden layer size that is the same size as the input and also uses leaky relu activation. We take the output of this neural network to be the approximate-HRR codebook vectors which we regularize.

A step that we took to improve training stability was the inclusion of multiplication constant at the output of the encoder and input to the decoder. Given the structural inductive biases imposed on the encoder, directly producing a zero-mean latent with $m/d$ per-component variance would push the model towards having small weights, potentially leading to training instability and vanishing gradients. Thus, we multiply the output of the encoder by $\sqrt{m/d}$ to ensure a higher variance output by the encoder. Similarly, since the reconstructed latent will also have approximately $m/d$ variance, we multiply the input to the decoder by $\sqrt{d/m}$ to transform it to have approximately unit variance for the decoder's input. We treat the output of the encoder the same way as the output of the encoder, and apply a $\sqrt{1/d}$ scaling to the produced codebook vectors. Although we observe in Figure~\ref{fig:latent_codebook_stats} of Appendix~\ref{app:hrr_initialization_importance} that the empirical variance during training is not exactly as assumed, we found that including the normalization constant improved training stability, but other choices may also have the same effect.

\section{Algorithm details}

This section covers more details on our training algorithm that we presented in Section~\ref{sec:hrrs_for_disentanglement}. Algorithm~\ref{alg:hrrvq} describes the latent reconstruction process that occurs at the bottleneck of the autoencoder. Algorithm~\ref{alg:hrr-train} references Algorithm~\ref{alg:hrrvq} in the broader training loop. 

\begin{algorithm}
\caption{HRR-VQ: Slot Retrieval, Denoising, and Quantization}
\label{alg:hrrvq}
\begin{algorithmic}[1]
\Require Normalized latent $\mat{z} \in \mathbb{R}^d$,
         frozen symbol vectors $\mathcal{S} = \{\mat{s}_1, \ldots, \mat{s}_m\}$,
         codebook $\mat{C} \in \mathbb{R}^{k \times d}$ with rows $\mat{c}_1, \ldots, \mat{c}_k$,
         denoising network $h : \mathbb{R}^d \to \mathbb{R}^d$
\Ensure Quantized latent $\hat{\mat{z}}$,
        quantized value vectors $\{\hat{\mat{v}}_i\}$,
        denoised retrievals $\{\bar{\mat{v}}_i\}$,
        assignment indices $\{j_i\}$
\For{$i = 1, \ldots, m$}
    \State $\tilde{\mat{v}}_i \gets \mat{s}_i^{-1} \bind \mat{z}$
    \Comment{HRR unbind: retrieve noisy value vector for slot $i$}
\EndFor
\For{$i = 1, \ldots, m$}
    \State $\bar{\mat{v}}_i \gets h(\tilde{\mat{v}}_i)$
    \Comment{Denoising network}
    \State $j_i \gets \arg\min_{j} \|\bar{\mat{v}}_i - \mat{c}_j\|^2$
    \Comment{Nearest-codebook-vector lookup}
    \State $\hat{\mat{v}}_i \gets \mat{c}_{j_i}$
    \Comment{Quantized value vector}
    \State $\mat{v}^{\mathrm{st}}_i \gets \bar{\mat{v}}_i + \sg{\hat{\mat{v}}_i - \bar{\mat{v}}_i}$
    \Comment{Straight-through estimator}
\EndFor
\State $\hat{\mat{z}} \gets \sum_{i=1}^{m} \mat{s}_i \bind \mat{v}^{\mathrm{st}}_i$
\Comment{Rebind and bundle all slots}
\State \Return $\hat{\mat{z}},\;\{\hat{\mat{v}}_i\},\;\{\bar{\mat{v}}_i\},\;\{j_i\}$
\end{algorithmic}
\end{algorithm}

\begin{algorithm}
\caption{Pseudocode for optimizing the HRR Autoencoder}
\label{alg:hrr-train}
\begin{algorithmic}[1]
\Require Dataset $\mathcal{D}$, batch size $b$,
         AdamW hyperparameters $(\alpha, \beta_1, \beta_2, \mathrm{wd})$,
         loss weights $\boldsymbol{\lambda} \coloneqq
           (\lambda_r,\,\lambda_{\mathrm{vq}},\,\beta,\,\lambda_z,\,\lambda_c)$,
         number of symbols $m$, codebook size $k$, latent dimension $d$
\State \textbf{initialize} CNN encoder $f:\mathcal{X} \to \mathbb{R}^d$,
       CNN decoder $g:\mathbb{R}^d \to \mathcal{X}$,
       denoising network $h:\mathbb{R}^d \to \mathbb{R}^d$
\State \textbf{initialize} seed matrix $\mat{S} \in \mathbb{R}^{k \times d}$,
       codebook generator $\varphi:\mathbb{R}^d \to \mathbb{R}^d$
\State \textbf{sample and freeze}
       $\mathcal{S} = \{\mat{s}_i\}_{i=1}^{m},\quad \mat{s}_i \overset{\mathrm{i.i.d.}}{\sim} \mathcal{N}(\mat{0},\,       
  \tfrac{1}{d}\mat{I}_d)$
\Comment{Random HRR symbol vectors; never updated}
\While{$(f, g, h, \mat{S}, \varphi)$ has not converged}
    \For{$n = 1, \ldots, b$}
        \State $\mat{x} \sim \mathcal{D}$
        \State $\mat{z} \gets f(\mat{x})\cdot\sqrt{m/d}$
        \Comment{Encode and normalize latent variance}
        \State $\mat{C} \gets \varphi(\mat{S})\cdot\sqrt{1/d}$
        \Comment{Generate codebook from seed matrix}
        \State $\hat{\mat{z}},\;\{\hat{\mat{v}}_i\},\;\{\bar{\mat{v}}_i\},\;\{j_i\}
               \gets \textsc{HRR-VQ}(\mat{z},\,\mathcal{S},\,\mat{C},\,h)$
        \Comment{Algorithm~\ref{alg:hrrvq}}
        \State $\hat{\mat{x}} \gets g(\hat{\mat{z}}\cdot\sqrt{d/m})$
        \Comment{Denormalize and decode}
        \State $\mathcal{L}_{\mathrm{recon}} \gets \lambda_r\cdot\mathrm{BCE}(\hat{\mat{x}},\,\mat{x})$
        \State $\mathcal{L}_{\mathrm{vq}} \gets \lambda_{\mathrm{vq}}\cdot
               \tfrac{1}{m}\sum_{i}\bigl\|\hat{\mat{v}}_i - \sg{\bar{\mat{v}}_i}\bigr\|^2$
        \Comment{Codebook toward denoised retrievals}
        \State $\mathcal{L}_{\mathrm{commit}} \gets \beta\cdot
               \tfrac{1}{m}\sum_{i}\bigl\|\sg{\hat{\mat{v}}_i} - \bar{\mat{v}}_i\bigr\|^2$
        \Comment{Encoder commits to codebook}
        \State $R_z \gets \mathbb{E}[\mat{z}]^2
               + \bigl(\mathrm{Var}(\mat{z}) - \tfrac{m}{d}\bigr)^{\!2}
               + \bigl(\|\mat{z}\|^2 - m\bigr)^{\!2}$
        \State $R_v \gets \mathbb{E}[\bar{\mat{v}}]^2
               + \bigl(\mathrm{Var}(\bar{\mat{v}}) - \tfrac{1}{d}\bigr)^{\!2}
               + \Bigl(\tfrac{1}{m}\textstyle\sum_i\|\bar{\mat{v}}_i\|^2 - 1\Bigr)^{\!2}$
        \State $R_c \gets \mathbb{E}[\mat{C}]^2
               + \bigl(\mathrm{Var}(\mat{C}) - \tfrac{1}{d}\bigr)^{\!2}
               + \Bigl(\tfrac{1}{k}\textstyle\sum_j\|\mat{c}_j\|^2 - 1\Bigr)^{\!2}$
        \State $\mathcal{L}_{\mathrm{reg}} \gets \lambda_z(R_z + R_v) + \lambda_c\cdot R_c$
        \State $\mathcal{L}^{(n)} \gets \mathcal{L}_{\mathrm{recon}}
               + \mathcal{L}_{\mathrm{vq}}
               + \mathcal{L}_{\mathrm{commit}}
               + \mathcal{L}_{\mathrm{reg}}$
    \EndFor
    \State $(f,g,h) \gets \mathrm{AdamW}\!\left(
               \nabla_{\!f,g,h}\,\tfrac{1}{b}\textstyle\sum_n\mathcal{L}^{(n)},\;
               (f,g,h),\;\alpha,\;\beta_1,\;\beta_2,\;\mathrm{wd}\right)$
    \State $(\mat{S},\varphi) \gets \mathrm{AdamW}\!\left(
               \nabla_{\!\mat{S},\varphi}\,\tfrac{1}{b}\textstyle\sum_n\mathcal{L}^{(n)},\;
               (\mat{S},\varphi),\;2\alpha,\;\beta_1,\;\beta_2,\;\mathrm{wd}\right)$
    \Comment{Codebook LR $= 2\alpha$}
\EndWhile
\end{algorithmic}
\end{algorithm}

\newpage

\section{Properties and structure of the HRR}
\label{app:retrieval}

In the following sections, we show that retrieval through the unbinding operation over HRRs introduces two terms that describe noise from different sources within the representation, and show that the particular variance used to initialize the HRR is not an essential component to preserving per-component SNR.

\subsection{Distribution of a retrieval} \label{ssec:retrieval_distribution}

It is possible to derive the distribution of a retrieval from the unbinding operation. This can be done through the circular correlation of unbinding. A version of this example was originally shown in ~\cite{plateHolographicReducedRepresentations1995a} and also revisited in ~\cite{ganesanLearningHolographicReduced2021a}. Under the circular correlation definition of binding (circular convolution with the inverse would be equivalent),

\[
\mathbf{c}^\dagger \otimes (\mathbf{c} \otimes \mathbf{x}) =
\begin{bmatrix}
x_0(c_0^2 + c_1^2 + c_2^2)
+ x_1 c_0 c_2 + x_2 c_0 c_1 + x_1 c_0 c_1 \\
\quad + x_2 c_1 c_2 + x_1 c_1 c_2 + x_2 c_0 c_2 \\[6pt]

x_1(c_0^2 + c_1^2 + c_2^2)
+ x_0 c_0 c_1 + x_2 c_0 c_2 + x_0 c_0 c_2 \\
\quad + x_2 c_1 c_2 + x_0 c_1 c_2 + x_2 c_0 c_1 \\[6pt]

x_2(c_0^2 + c_1^2 + c_2^2)
+ x_0 c_0 c_1 + x_1 c_1 c_2 + x_0 c_1 c_2 \\
\quad + x_1 c_0 c_2 + x_0 c_0 c_2 + x_1 c_0 c_1
\end{bmatrix}
\]

\[
=
\begin{bmatrix}
x_0(1 + \xi) + \eta_0 \\
x_1(1 + \xi) + \eta_1 \\
x_2(1 + \xi) + \eta_2
\end{bmatrix}
= (1 + \xi)\,\mathbf{x} + \boldsymbol{\eta}.
\]

Here, $\xi = (c_0^2 + \cdots + c_d^2) - 1$. Through the central limit theorem and by assuming independence between $c_i$ and $x_i$, we can show that $\xi \sim \mathcal{N}(0, 2/d)$, and $\eta_i = \mathcal{N}(0, (d-1)/d^2)$. For the average case, the $1 + \xi$ term can be taken to 1, since $\E\left[ \|\mat{c}\|^2 \right]=1$, which leaves the final result to be our intended retrieval target plus a noise term.   

These results can be extended to our case for retrieval from a sum of bound vectors. We add another pair to the quantity from the previous example, $\mathbf{c}^\dagger \otimes \left[(\mathbf{c} \otimes \mathbf{x}) + (\mathbf{d} \otimes \mathbf{y})\right]$. For compactness, denote $w_i = (\mathbf{c} \otimes \mathbf{x})_i$, and $v_i = (\mathbf{d} \otimes \mathbf{y})_i$. The retrieval can be expressed as below.

\[
\begin{aligned}
\mathbf{c}^\dagger \otimes \left[(\mathbf{c} \otimes \mathbf{x}) + (\mathbf{d} \otimes \mathbf{y})\right] 
&=
\begin{bmatrix}
c_0 (w_0 + v_0) + c_1 (w_1 + v_1) + c_2 (w_2 + v_2) \\
\cdots \\
\cdots
\end{bmatrix} \\
&=
\begin{bmatrix}
c_0 (c_0 x_0 + c_1 x_2 + c_2 x_1 + d_0 y_0 + d_1 y_2 + d_2 y_1) \\
\quad + c_1 (c_0 x_1 + c_1 x_0 + c_2 x_2 + d_0 y_1 + d_1 y_0 + d_2 y_2 \\
\quad + c_2 (c_0 x_2 + c_1 x_1 + c_2 x_0 + d_0 y_2 + d_1 y_1 + d_2 y_0 \\
\cdots \\
\cdots
\end{bmatrix} \\
&=
\begin{bmatrix}
(c_0^2 + c_1^2 + c_2^2) \ x_0 \\
\quad + c_0 c_1 x_2 + c_0 c_2 x_1 + c_0 d_0 y_0 + c_0 d_1 y_2 + c_0 d_2 y_1 \\
\quad + c_0 c_1 x_1 + c_1 c_2 x_2 + c_1 d_0 y_1 + c_1 d_1 y_0 + c_1 d_2 y_2 \\
\quad + c_0 c_2 x_2 + c_1 c_2 x_1 + c_2 d_0 y_2 + c_2 d_1 y_1 + c_2 d_2 y_0 \\
\cdots \\
\cdots
\end{bmatrix} \\
&=
\begin{bmatrix}
x_0 (1 + \xi) + \eta_0 + \lambda_0 \\
x_1 (1 + \xi) + \eta_1 + \lambda_1 \\
x_2 (1 + \xi) + \eta_2 + \lambda_2
\end{bmatrix} \\
&= (1 + \xi) \mathbf{x} + \boldsymbol{\eta} + \boldsymbol{\lambda}.
\end{aligned}
\]

The final expression, for the $m=2$ case, results in the same expression for the $m=1$ case, with an extra term for the cross talk noise due to unbinding. Each $\lambda_i$ has $d^2$ terms, each like $c_j d_k y_l$. Assuming independence between terms, $c_j d_k y_l \sim \mathcal{N}(0, 1/d^3)$. Since there are $d^2$ terms, $\lambda_i = \mathcal{N}(0, 1/d)$. In general, for the retrieval of a vector from a latent $\mathbf{z}$ with $m$ bound terms, 

\[
\hat{\mathbf{x}} = (\mathbf{c}^{(i)})^\dagger \otimes \mathbf{z} = \underbrace{(1+\xi)\mathbf{x}}_{\text{retrieval}} + \underbrace{\boldsymbol{\eta}}_{\text{intrinsic noise}} + \underbrace{\sum_{k\neq i}^m \boldsymbol{\lambda}^{(k)}}_{\text{cross-talk noise}},
\]

where $\mathbf{c}^{(i)}$ is the cue associated with a target we are trying to retrieve, and $\boldsymbol{\lambda}^{(k)}$ is the cross talk term due to the $k$'th bound pair in the representation.

Thus, we can interpret the retrieval as being composed of 3 main parts, the intended retrieval target, the intrinsic noise, and the cross talk noise. Since our signal and noise terms are both zero mean, we can quantify the per component signal power as $S=1/d$, and the per component noise power as $N = \frac{d-1}{d^2} + \frac{m-1}{d}=\frac{m}{d}-\frac{1}{d^2}$. For large $d$, we have $N \approx \frac{m}{d}$. Then the per component SNR for a retrieval is

\[
\text{SNR} = \frac{1/d}{m/d} = \frac{1}{m}.
\]

\subsection{Importance of the HRR initialization} \label{app:hrr_initialization_importance}

The analysis in Appendix~\ref{ssec:retrieval_distribution} shows that the zero-mean requirement for the HRR vectors ensures that all the noise terms are also zero-mean. The per component variance of $1/d$ ensures that the expected squared norm of HRR vectors is one. This allows us to ensure that the norm of a representation does not grow unbounded as it is bound with more vectors. This was a desirable property for the original intended purpose of the HRR, which was to be able to create representations of arbitrary depth~\cite{plateHolographicReducedRepresentations1995a}. For our case, this requirement is not strictly necessary since we only use HRR vectors to create bundles of $m$ bound pairs, not for binding arbitrary amounts. If we were to instead consider the case in which HRR vector components were instead sampled from $\mathcal{N}(0, c)$, where $c > 0$. Then we can find that the binding variance becomes 

\[
\var((a \otimes b)_i) = d \cdot c^2.
\]

For retrieval, the per component cross talk variances becomes

\[
\var(\lambda_i) = m d^2 c^3,
\]

and the intrinsic noise becomes

\[
\var(\eta_i) = d(d-1)c^3 \approx d^2\cdot c^3.
\]

The per component signal power becomes $S = \var((c_0^2 + \cdots + c_d^2) \cdot x_i) = (2dc^2+c^2 d^2) \cdot c \approx c^2 d^2 \cdot c$, for $d \gg c$, and thus the SNR is

\[
\text{SNR} = \frac{d^2c^3}{(m-1) d^2 c^3 + d^2c^3} = \frac{d^2c^3}{md^2 c^3} = \frac{1}{m},
\]

thus showing that the per component SNR actually remains unchanged with respect to the variance of the atomic HRR vectors. This shows that the $1/d$ variance requirement from the original HRR implementation may not be necessary for our purposes, and thus any reasonable choice, $c$ for the per component variance may also work. However, the original choice of variance is retained for consistency with the original HRR implementation and other works using the HRR. The results in Fig.~\ref{fig:latent_codebook_stats} support this finding. The model learns to keep the zero mean requirement as it is essential for the unbinding process to work. The variances drift from the $1/d$ constant associated with the HRR, but this finding supports the analysis above, as the $1/d$ requirement is not strict. This leaves room for the model to converge on a separate variance that is more optimal for that particular dataset, which we believe is what the figure reflects.

\newpage 

\begin{figure}[t]
    \centering
    \includegraphics[width=0.9\linewidth]{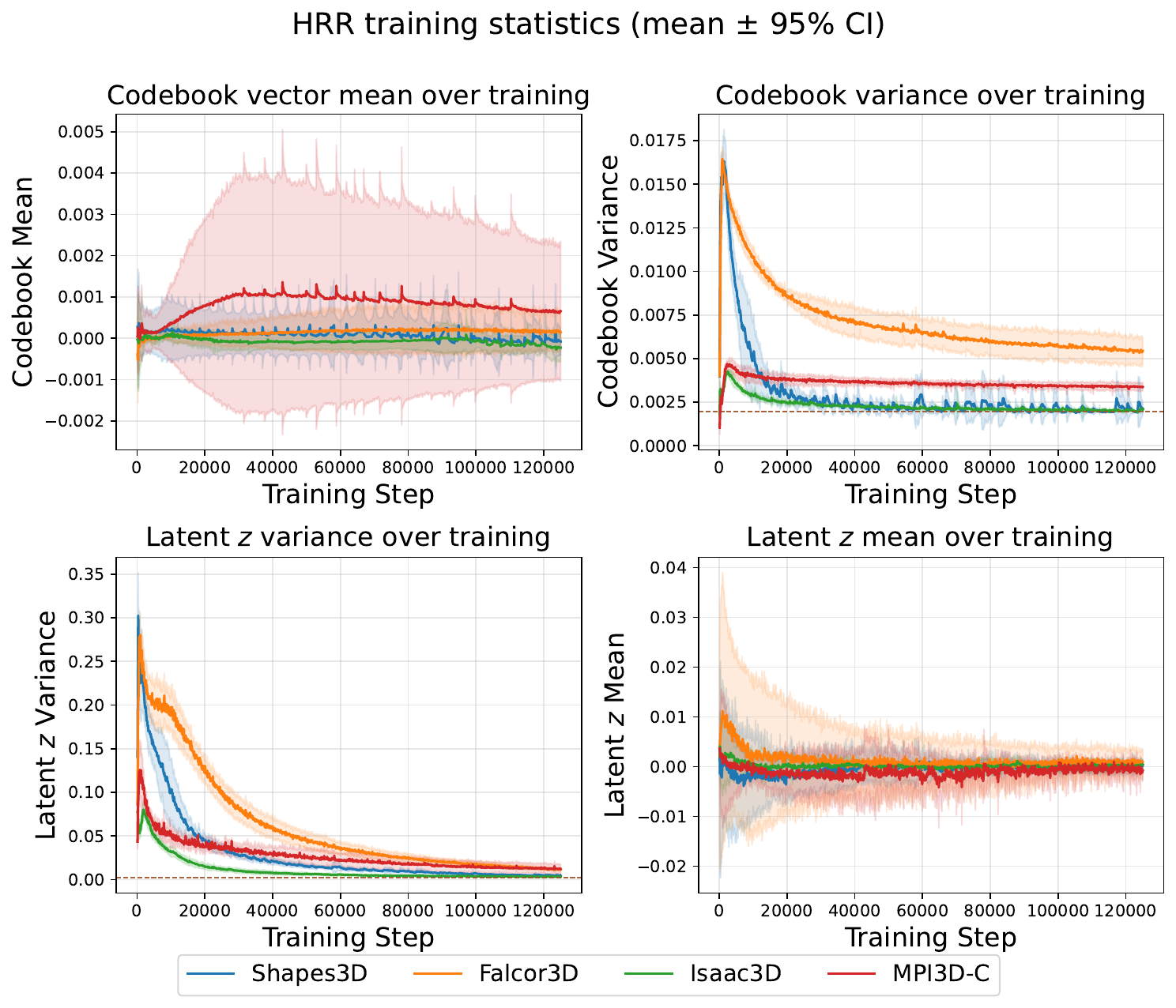}
    \caption{Mean and variance over training iterations, for both the codebook vectors and the latent representations. Most statistics converge on stable values as training progresses.}
    \label{fig:latent_codebook_stats}
\end{figure}

\newpage

\begin{figure}[t]
    \centering
    \includegraphics[width=0.8\linewidth]{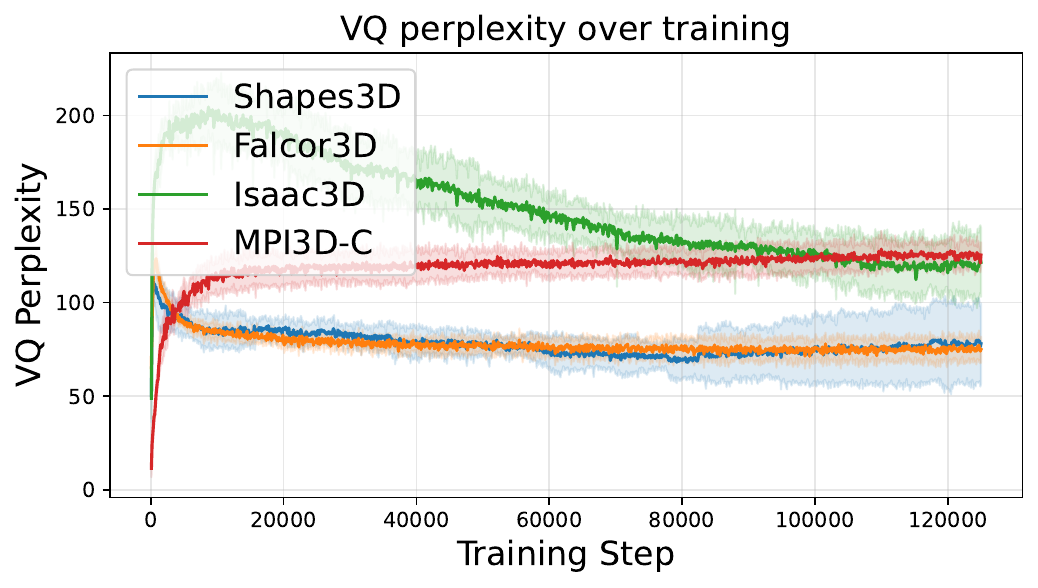}
    \caption{Codebook usage over training iterations, measured by perplexity. We can see that the model learns to utilize the codebook vectors, with the harder datasets having higher codebook usage.}
    \label{fig:codebook_utilization}
\end{figure}

\clearpage

\section{Proofs of theoretical results}
\label{app:proofs}
 
Next, we prove Propositions~\ref{prop:capacity} and~\ref{prop:slot-indep} from
Sections~\ref{sec:capacity} and~\ref{sec:slot-indep}. In this section, expectations are with respect to
the HRR initialization distribution (each component i.i.d.\ $\mathcal{N}(0,1/d)$) and the data
distribution, as appropriate. We write $\tilde{\mat{V}}_{1:m}$ for the tuple
$(\tilde{\mat{v}}_1,\dots,\tilde{\mat{v}}_m)$ and $\tilde{\mat{V}}_{<i}$ for
$(\tilde{\mat{v}}_1,\dots,\tilde{\mat{v}}_{i-1})$.
 
\subsection{Proof of proposition~\ref{prop:slot-indep} (approximate slot independence)}\label{app:slot_independence_proof}
 
By the chain rule of probability,
\begin{equation}
\log\frac{q(\tilde{\mat{V}}_{1:m}\mid \mat{x})}{\prod_{i}q_i(\tilde{\mat{v}}_i\mid \mat{x})}
\;=\; \sum_{i=2}^{m} \log\frac{q_i(\tilde{\mat{v}}_i\mid \tilde{\mat{V}}_{<i},\mat{x})}{q_i(\tilde{\mat{v}}_i\mid \mat{x})}.
\label{eq:chain-rule}
\end{equation}
The term where $i=1$ is identically zero since the conditioning set is empty. Taking expectations under
$q(\tilde{\mat{V}}_{1:m}\mid \mat{x})$,
\begin{equation}
D_{\mathrm{KL}}\!\bigl(q\,\|\,\textstyle\prod_i q_i\bigr)
\;=\; \sum_{i=2}^{m} \mathbb{E}_q\!\left[\log\frac{q_i(\tilde{\mat{v}}_i\mid \tilde{\mat{V}}_{<i},\mat{x})}{q_i(\tilde{\mat{v}}_i\mid \mat{x})}\right]
\;=\; \sum_{i=2}^{m} \I{\tilde{\mat{v}}_i}{\tilde{\mat{V}}_{<i}\mid \mat{x}}.
\label{eq:sum-of-mi}
\end{equation}

We bound the right-hand side of~\eqref{eq:sum-of-mi} by treating the total correlation
as a single aggregate quantity, rather than bounding each term
$\I{\tilde{\mat{v}}_i}{\tilde{\mat{V}}_{<i}\mid \mat{x}}$ individually. As we discuss in Section~\ref{sec:reduction}, bounding each term in isolation leads to expressions
involving the smallest eigenvalue of a conditioning block, which is not accessible from
the operator norms we can establish. Instead, we work with the closed-form KL
representation of the full sum and bound it through a single log-determinant expression.
This approach requires an understanding of the block covariance structure of the joint
retrieval $\tilde{\mat{V}}_{1:m}$, which we develop next.

\subsection{Covariance Block Structure}

The mean of the joint retrieval distribution $q(\tilde{\mat{V}}_{1:m}\mid \mat{x})$ is
$(\mat{v}_1(\mat{x}),\dots,\mat{v}_m(\mat{x}))$. The block covariance matrix
$\Sigma \in \R^{md \times md}$, with block entries

\[
\Sigma=
\begin{bmatrix}
    \Sigma_{11} & \Sigma_{12} & \cdots & \Sigma_{1m} \\
    \Sigma_{21} & \Sigma_{22} & \cdots & \Sigma_{2m} \\
    \vdots      & \vdots      & \ddots & \vdots      \\
    \Sigma_{m1} & \cdots      & \cdots & \Sigma_{mm} \\
\end{bmatrix},
\]

where any entry $\Sigma_{ij} = \E \left[ \tilde{\mat{v}}_i \tilde{\mat{v}}_j^\T \right] \in \R^{d \times d}$. Now, recall that the structure of a retrieval can be expressed as
\[
\tilde{\mat{v}}_i = (1+\xi_i) \mat{v}_i + \mat{\eta}_i + \sum_{k\neq i} \mat{\lambda}^{(ik)},
\]
where $\mat{\lambda}^{(ik)}$ represents the cross-talk introduced from $\mat{s}_i^\inv \bind (\mat{s}_k \bind \mat{v}_k)$. From the definition of the covariance, we have
\begin{align*}
    \cov \left(\tilde{\mat{v}}_i, \tilde{\mat{v}}_j \mid \tilde{\mat{V}}_{1:m}, \mat{x} \right)
    &= \E \left[(\tilde{\mat{v}}_i - \mat{v}_i) ((\tilde{\mat{v}}_j - \mat{v}_j)^\T \mid \tilde{\mat{V}}_{1:m}, \mat{x} \right] \\
    &= \E \left[\epsilon_i \epsilon_j^\T \middle\vert \tilde{\mat{V}}_{1:m}, \mat{x} \right] \\
    &= \E \left[\left(\xi_i \mat{v}_i + \mat{\eta}_i + \sum_{\ell \neq i} \mat{\lambda}^{(ik)}\right) \left(\xi_j \mat{v}_j + \mat{\eta}_j + \sum_{r \neq j} \mat{\lambda}^{(jr)}\right)^\T 
    \middle\vert \tilde{\mat{V}}_{1:m}, \mat{x} \right] \\ 
\end{align*}

\subsubsection{Diagonal Blocks}

We begin by analyzing the diagonal blocks of the block covariance matrix, $\Sigma$. Specifically, we'd like to find the structure of an entry in $\Sigma_{ii}$, defined as
\[
(\Sigma_{ii})_{ab} = \E \left[\left(\xi_i \mat{v}_i + \mat{\eta}_i + \sum_{\ell \neq i} \mat{\lambda}^{(i \ell)}\right)_a \cdot \left(\xi_i \mat{v}_i + \mat{\eta}_i + \sum_{r \neq i} \mat{\lambda}^{(ir)}\right)_b 
    \, \middle\vert \tilde{\mat{V}}_{1:m}, \mat{x} \right].
\]
We begin by computing the expected values of the terms in the resulting product.

\begin{enumerate}
    \item $\E \left[ (\xi_i \mat{v}_i)_a  \cdot (\xi_i \mat{v}_i)_b\right] = (\mat{v}_i)_a (\mat{v}_i)_b \, \E \left[ \xi_i^2 \right] =(\mat{v}_i)_a (\mat{v}_i)_b \cdot \frac{2}{d} $

    \item $\E \left[ (\mat{\eta}_i)_a (\mat{\eta}_i)_b \right] \neq 0$

    \item $\E \left[ (\xi_i \mat{v}_i)_a (\mat{\eta}_i)_b \right] = (\mat{v}_i)_a\E \left[ \xi_i (\mat{\eta}_i)_b \right] = 0$
    
    \item $\E \left[ (\mat{\eta}_i)_a (\mat{\lambda}^{(i \ell)})_b \right] =\E \left[ (\mat{\eta}_i)_a (\mat{s}_i^\inv \bind \mat{s}_{\ell} \bind \mat{v}_{\ell})_b \right] $

    Observe that $(\mat{\eta}_i)_a$ will contain traces of elements of $\mat{s}_i$, so the unbinding operation will introduce terms with squared elements of $\mat{s}_i$. However, these terms will still have elements of $\mat{s}_{\ell}$ to the first power, which in expectation, from the HRR distribution, will result in 
    \[
    \E \left[ (\mat{\eta}_i)_a (\mat{\lambda}^{(i \ell)})_b \right] = 0.
    \]
    
    \item $\E \left[ (\xi_i \mat{v}_i)_a (\mat{\lambda}^{(i \ell)})_b \right]$

    Applying the same reasoning as in 4), entries from $\mat{s}_{\ell}$ will remain in the product to first power, resulting in
    \[
    \E \left[ (\xi_i \mat{v}_i)_a (\mat{\lambda}^{(i \ell)})_b \right] = 0.
    \]
    
    \item $\E \left[ (\mat{\lambda}^{(i \ell)})_a (\mat{\lambda}^{(ir)})_b \right] \neq 0$

    For $\ell \neq r$, the expected value will evaluate to zero due to the elements of $\mat{s}_{\ell}$ and $\mat{s}_r$ appearing to first degree, however, the case where $\ell = r$ will not be zero since squared elements of the symbols will appear.
\end{enumerate}

From the analysis above, we see that the covariance, $\Sigma_{ii}$, can be expressed as
\[
\Sigma_{ii} = \E \left[ (\xi_i \mat{v}_i + \mat{\eta}_i) (\xi_i \mat{v}_i + \mat{\eta}_i)^\T  \right] + \sum_{\ell \neq i} \E \left[ \mat{\lambda}^{(i \ell)} (\mat{\lambda}^{(i \ell)})^\T  \right],
\]
as the other cross-terms cancel. Then, applying the definitions of circular convolution and circular correlation, we can express an entry from a single cross-talk contribution. Letting $\mat{s}_k^{(i)}$ denote the $k$'th entry of $\mat{s}_i$, we have
\begin{align*}
    \E \left[ (\mat{\lambda}^{(i \ell)})_a (\mat{\lambda}^{(i \ell)})_b  \right]
    &= \E \left[ \left( \sum_{k=0}^{d-1} \sum_{m=0}^{d-1} \mat{s}^{(i)}_k \mat{s}^{(\ell)}_m (\mat{v}_{\ell})_{k+a-m} \right) \cdot \left( \sum_{k'=0}^{d-1} \sum_{m'=0}^{d-1} \mat{s}^{(i)}_{k'} \mat{s}^{(\ell)}_{m'} (\mat{v}_{\ell})_{k' +b-m'} \right) \right] \\
    &= \E \left[ \sum_{k=0}^{d-1} \sum_{m=0}^{d-1} \sum_{k'=0}^{d-1} \sum_{m'=0}^{d-1} \mat{s}^{(i)}_k \mat{s}^{(\ell)}_m \mat{s}^{(i)}_{k'} \mat{s}^{(\ell)}_{m'} (\mat{v}_{\ell})_{k+a-m} (\mat{v}_{\ell})_{k' +b-m'} \right] \\
    &= \sum_{k=0}^{d-1} \sum_{m=0}^{d-1} \sum_{k'=0}^{d-1} \sum_{m'=0}^{d-1} \E \left[ \mat{s}^{(i)}_k \mat{s}^{(\ell)}_m \mat{s}^{(i)}_{k'} \mat{s}^{(\ell)}_{m'} \right] (\mat{v}_{\ell})_{k+a-m} (\mat{v}_{\ell})_{k' +b-m'} \\
    &= \sum_{k=0}^{d-1} \sum_{m=0}^{d-1} \E \left[ \mat{s}^{(i)}_k \mat{s}^{(\ell)}_m \mat{s}^{(i)}_{k} \mat{s}^{(\ell)}_{m} \right] (\mat{v}_{\ell})_{k+a-m} (\mat{v}_{\ell})_{k +b-m} \\
    &= \frac{1}{d^2} \sum_{k=0}^{d-1} \sum_{m=0}^{d-1} (\mat{v}_{\ell})_{k+a-m} (\mat{v}_{\ell})_{k +b-m} \\
    &= \frac{1}{d^2} \sum_{k=0}^{d-1} \sum_{n=0}^{d-1} (\mat{v}_{\ell})_{a+n} (\mat{v}_{\ell})_{b+n} \\
    &= \frac{1}{d} \sum_{n=0}^{d-1} (\mat{v}_{\ell})_{a+n} (\mat{v}_{\ell})_{b+n}, \\
\end{align*}
where $n=k-m \mod d$. Thus, the contribution of the cross-talk term can be written as
\begin{align*}
\sum_{\ell \neq i} \E \left[ (\mat{\lambda}^{(i \ell)})_a (\mat{\lambda}^{(i \ell)})_b \right]
&= \sum_{\ell \neq i} \frac{1}{d} \sum_{n=0}^{d-1} (\mat{v}_{\ell})_{a+n} (\mat{v}_{\ell})_{b+n}. \\
\end{align*}

Evaluating the non-crosstalk contributions, we have that

\begin{align*}
    \E \left[ (\xi_i \mat{v}_i + \mat{\eta}_i)_a (\xi_i \mat{v}_i + \mat{\eta}_i)_b  \right]
    &= \E \left[ (\xi_i \mat{v}_i)_a (\xi_i \mat{v}_i)_b + (\mat{\eta}_i)_a (\mat{\eta}_i)_b \right] \\
    &= (\mat{v}_i)_a (\mat{v}_i)_b \, \E\left[ \xi_i^2 \right] + \E \left[ (\mat{\eta}_i)_a (\mat{\eta}_i)_b \right] \\
    &= (\mat{v}_i)_a (\mat{v}_i)_b \cdot \frac{2}{d} + \E \left[ (\mat{\eta}_i)_a (\mat{\eta}_i)_b \right].
\end{align*}

We can derive an expression for $(\mat{\eta}_i)_a$ directly from the unbinding definition,
\begin{align*}
    (\tilde{\mat{v}}_i)_a 
    &= \sum_{k=0}^{d-1} \sum_{m=0}^{d-1} \mat{s}^{(i)}_k \mat{s}^{(i)}_m (\mat{v}_i)_{k+a-m} \\
    &= \sum_{k=0}^{d-1} \left(\mat{s}^{(i)}_k \mat{s}^{(i)}_k (\mat{v}_i)_{k+a-k} + \sum_{m \neq k} \mat{s}^{(i)}_k \mat{s}^{(i)}_m (\mat{v}_i)_{k+a-m} \right) \\
    &= \sum_{k=0}^{d-1} \left( \left(\mat{s}^{(i)}_k\right)^2 (\mat{v}_i)_{a} + \sum_{m \neq k} \mat{s}^{(i)}_k \mat{s}^{(i)}_m (\mat{v}_i)_{k+a-m} \right) \\
    &= \left[ \left(\mat{s}^{(i)}_1\right)^2 + \left(\mat{s}^{(i)}_2\right)^2 + \cdots + \left(\mat{s}^{(i)}_d\right)^2 \right] (\mat{v}_i)_{a} + \sum_{k=0}^{d-1} \sum_{m \neq k}^{d-1} \mat{s}^{(i)}_k \mat{s}^{(i)}_m (\mat{v}_i)_{k+a-m}  \\
    &= \| \mat{s}_i \|^2 (\mat{v}_i)_{a} + \sum_{k=0}^{d-1} \sum_{m \neq k}^{d-1} \mat{s}^{(i)}_k \mat{s}^{(i)}_m (\mat{v}_i)_{k+a-m}.  \\
\end{align*}
So we have that
\begin{align*}
    \E \left[ (\mat{\eta}_i)_a (\mat{\eta}_i)_b \right] 
    &= 
    \E \left[ \left( \sum_{k=0}^{d-1} \sum_{m \neq k}^{d-1} \mat{s}^{(i)}_k \mat{s}^{(i)}_m (\mat{v}_i)_{k+a-m} \right) \cdot 
    \left( \sum_{k'=0}^{d-1} \sum_{m' \neq k'}^{d-1} \mat{s}^{(i)}_{k'} \mat{s}^{(i)}_{m'} (\mat{v}_i)_{k'+b-m'} \right) \right] \\
    &= 
    \E \left[ 
    \sum_{k=0}^{d-1} \sum_{m \neq k}^{d-1} \sum_{k'=0}^{d-1} \sum_{m' \neq k'}^{d-1} \mat{s}^{(i)}_k \mat{s}^{(i)}_m \mat{s}^{(i)}_{k'} \mat{s}^{(i)}_{m'} 
    (\mat{v}_i)_{k+a-m} (\mat{v}_i)_{k'+b-m'} \right] \\
    &=  
    \sum_{k=0}^{d-1} \sum_{m \neq k}^{d-1} \sum_{k'=0}^{d-1} \sum_{m' \neq k'}^{d-1} 
    \E \left[ \mat{s}^{(i)}_k \mat{s}^{(i)}_m \mat{s}^{(i)}_{k'} \mat{s}^{(i)}_{m'} \right]
    (\mat{v}_i)_{k+a-m} (\mat{v}_i)_{k'+b-m'} \\
\end{align*}
Observe that the expectation is non-zero only when $k=k',m=m'$ and when $k=m', m = k'$, thus, 
\begin{align*} 
    \E \left[ (\mat{\eta}_i)_a (\mat{\eta}_i)_b \right] 
    &= \sum_{k=0}^{d-1} \sum_{m \neq k}^{d-1} \E \left[ \mat{s}^{(i)}_k \mat{s}^{(i)}_m \mat{s}^{(i)}_{k} \mat{s}^{(i)}_{m} \right] (\mat{v}_i)_{k+a-m} (\mat{v}_i)_{k+b-m} \\
    & \qquad + \sum_{k=0}^{d-1} \sum_{m \neq k}^{d-1} \E \left[ \mat{s}^{(i)}_k \mat{s}^{(i)}_m \mat{s}^{(i)}_{m} \mat{s}^{(i)}_{k} \right] (\mat{v}_i)_{k+a-m} (\mat{v}_i)_{m+b-k} \\
    &= \frac{1}{d^2} \sum_{k=0}^{d-1} \sum_{m \neq k}^{d-1} (\mat{v}_i)_{k+a-m} (\mat{v}_i)_{k+b-m} + \frac{1}{d^2} \sum_{k=0}^{d-1} \sum_{m \neq k}^{d-1} (\mat{v}_i)_{k+a-m} (\mat{v}_i)_{m+b-k} \\
    &= \frac{1}{d^2} \left[ \sum_{k=0}^{d-1} \sum_{m \neq k}^{d-1} (\mat{v}_i)_{k+a-m} (\mat{v}_i)_{k+b-m} + \sum_{k=0}^{d-1} \sum_{m \neq k}^{d-1} (\mat{v}_i)_{k+a-m} (\mat{v}_i)_{m+b-k} \right] \\
    &= \frac{1}{d} \left[ \sum_{n \neq 0} (\mat{v}_i)_{a+n} (\mat{v}_i)_{b+n} + \sum_{n \neq 0} (\mat{v}_i)_{a+n} (\mat{v}_i)_{b-n} \right]. \\
\end{align*}

Finally, we can state
\begin{align*}
    (\Sigma_{ii})_{ab} = \underbrace{(\mat{v}_i)_a (\mat{v}_i)_b \cdot \frac{2}{d}}_{\E \left[ (\xi_i \mat{v}_i)_a (\xi_i \mat{v}_i)_b \right]} 
    &+ \underbrace{\sum_{\ell \neq i} \frac{1}{d} \sum_{n=0}^{d-1} (\mat{v}_{\ell})_{a+n} (\mat{v}_{\ell})_{b+n}}_{\sum_{\ell \neq i} \E \left[ (\mat{\lambda}^{(i \ell)})_a (\mat{\lambda}^{(i \ell)})_b \right]} \\
    & \qquad+ \underbrace{\frac{1}{d} \left[ \sum_{n \neq 0} (\mat{v}_i)_{a+n} (\mat{v}_i)_{b+n} + \sum_{n \neq 0} (\mat{v}_i)_{a+n} (\mat{v}_i)_{b-n} \right]}_{\E \left[ (\mat{\eta}_i)_a (\mat{\eta}_i)_b \right]}.
\end{align*}

Observe that when $a=b$, and when $d \gg 0$, we'll have, for the average case
\begin{align*}
    (\Sigma_{ii})_{aa} 
    &= (\mat{v}_i)_a^2 \cdot \frac{2}{d} + \sum_{\ell \neq i} \frac{1}{d} \sum_{n=0}^{d-1} (\mat{v}_{\ell})_{a+n}^2 + \frac{1}{d} \left[ \sum_{n \neq 0} (\mat{v}_i)_{a+n}^2 + \sum_{n \neq 0} (\mat{v}_i)_{a+n} (\mat{v}_i)_{b-n} \right] \\
    &= \underbrace{(\mat{v}_i)_a^2}_{1/d} \cdot \frac{2}{d} + \frac{1}{d} \sum_{\ell \neq i} \underbrace{\|\mat{v}_{\ell}\|^2}_{\approx 1} + \frac{1}{d} \left[ \underbrace{\left(\|\mat{v}_i\|^2 - (\mat{v}_i)_a^2 \right)}_{\approx 1 - 1/d} + \underbrace{\sum_{n \neq 0} (\mat{v}_i)_{a+n} (\mat{v}_i)_{b-n}}_{\approx 0} \right] \\
    &\approx \frac{2}{d^2} + \frac{m-1}{d} + \frac{1}{d} - \frac{1}{d^2} \\
    &= \frac{1}{d^2} + \frac{m}{d} \\
    &\approx \frac{m}{d}.
\end{align*}
The approximations defined under the braces come from the fact that elements of $\mat{v}_i$ are sampled from the HRR distribution, $\mathcal{N}(0, 1/d)$.

In the alternative case, for $a \neq b$, we have $(\Sigma_{ii})_{ab} \approx 0$ due to the products of terms that have zero mean. Thus, we have the approximation
\[
\Sigma_{ii} \approx \frac{m}{d} \mat{I}.
\]
\subsubsection{Off-Diagonal Blocks} \label{app:off_diagonal_blocks}

For the off-diagonal blocks, we are interested in 

\[
(\Sigma_{ij})_{ab} = \E \left[\left(\xi_i \mat{v}_i + \mat{\eta}_i + \sum_{\ell \neq i} \mat{\lambda}^{(i \ell)}\right)_a \cdot \left(\xi_j \mat{v}_j + \mat{\eta}_j + \sum_{r \neq j} \mat{\lambda}^{(jr)}\right)_b 
    \, \middle\vert \tilde{\mat{V}}_{1:m}, \mat{x} \right].
\]

By going through cases again, we make the following observations.

\begin{enumerate}
    \item $\E \left[ (\xi_i \mat{v}_i)_a  \cdot (\xi_j \mat{v}_j)_b\right] = (\mat{v}_i)_a (\mat{v}_j)_b \, \E \left[ \xi_i \xi_j \right] = 0$

    \item $\E \left[ (\mat{\eta}_i)_a (\mat{\eta}_j)_b \right] = 0$

    \item $\E \left[ (\xi_i \mat{v}_i)_a (\mat{\eta}_j)_b \right] = (\mat{v}_i)_a\E \left[ \xi_i (\mat{\eta}_j)_b \right] = 0$
    
    \item $\E \left[ (\mat{\eta}_i)_a (\mat{\lambda}^{(jr)})_b \right] = 0 $

    We have already seen that for the case when $i = j$, the expectation is zero. Additionally, when $i \neq j$, the expectation is still zero, as entries of $\mat{s}_i$, $\mat{s}_j$, and $\mat{s}_r$, all appear at first power.
    
    \item $\E \left[ (\xi_i \mat{v}_i)_a (\mat{\lambda}^{(jr)})_b \right] = 0$

    Applying the same logic as 4), we see that this expectation is also zero.

    \item $\E \left[ (\mat{\lambda}^{(i \ell)})_a (\mat{\lambda}^{(jr)})_b \right]$

    For this case, the expectation will be non-zero only for the case when $\ell = j$ and $r=i$.
\end{enumerate}

Examining case six closer, similarly as before, we can write
\begin{align*}
    (\Sigma_{ij})_{ab} 
    &= \sum_{k=0}^{d-1} \sum_{m=0}^{d-1} \sum_{k'=0}^{d-1} \sum_{m'=0}^{d-1} 
    \E \left[ \mat{s}_k^{(i)} \mat{s}_m^{(j)} \mat{s}_{k'}^{(j)} \mat{s}_{m'}^{(i)}  \right] (\mat{v}_j)_{k+a-m} (\mat{v}_i)_{k'+b-m'} \\
    &= \sum_{k=0}^{d-1} \sum_{m=0}^{d-1} 
    \E \left[ \mat{s}_k^{(i)} \mat{s}_m^{(j)} \mat{s}_{m}^{(j)} \mat{s}_{k}^{(i)}  \right] (\mat{v}_j)_{k+a-m} (\mat{v}_i)_{m+b-k} \\
    &= \frac{1}{d^2}\sum_{k=0}^{d-1} \sum_{m=0}^{d-1} 
    (\mat{v}_j)_{k+a-m} (\mat{v}_i)_{m+b-k} \\
    &= \frac{1}{d} \sum_{n=0}^{d-1}
    (\mat{v}_j)_{a+n} (\mat{v}_i)_{b-n}. \\
\end{align*}
This establishes that entries of $\Sigma_{ij}$ are at least $\mathcal{O}(\frac{1}{d})$. Then, observe that from the definition of circular convolution, we can rewrite the final equality as
\begin{align*}
    (\Sigma_{ij})_{ab}
    &= \frac{1}{d} \sum_{n=0}^{d-1} (\mat{v}_j)_{a+n} (\mat{v}_i)_{b-n} \\
    &= \frac{1}{d} (\mat{v}_j \bind \mat{v}_i)_{a+b}.
\end{align*}
We can then express the squared Frobenius norm of $\Sigma_{ij}$ as
\begin{align*}
    \| \Sigma_{ij} \|^2_F
    &= \sum_{a=1}^d \sum_{b=1}^d (\Sigma_{ij})_{ab}^2 \\
    &= \frac{1}{d^2} \sum_{a=0}^{d-1} \sum_{b=0}^{d-1} |(\mat{v}_j \bind \mat{v}_i)_{a+b}|^2 \\
    &= \frac{1}{d^2} \cdot d  \sum_{c=0}^{d-1} |(\mat{v}_j \bind \mat{v}_i)_{c}|^2 \\
    &= \frac{1}{d} \sum_{c=0}^{d-1} |(\mat{v}_j \bind \mat{v}_i)_{c}|^2 \\
    &= \frac{1}{d} \|(\mat{v}_j \bind \mat{v}_i)\|^2. \\
\end{align*}
Since each entry of $(\mat{v}_j \bind \mat{v}_i)$ has variance $\frac{1}{d}$, we can say that $(\mat{v}_j \bind \mat{v}_i)_{c} = \mathcal{O}(\frac{1}{\sqrt{d}})$. From this, we can establish that
\begin{align*}
    \|(\mat{v}_j \bind \mat{v}_i)\|^2 
    &= \sum_{c=0}^{d-1} |(\mat{v}_j \bind \mat{v}_i)_{c}|^2 \\
    &= \sum_{c=0}^{d-1} \mathcal{O}\left(\frac{1}{d} \right) \\
    &= \mathcal{O} \left( 1 \right). \\
\end{align*}
Thus, we have that
\[
    \| \Sigma_{ij} \|^2_F
    = \frac{1}{d} \underbrace{\|(\mat{v}_j \bind \mat{v}_i)\|^2}_{\mathcal{O} \left( 1 \right)},
\]
which implies that $\| \Sigma_{ij} \|^2_F$ is $\mathcal{O} \left( \frac{1}{d} \right)$.

\subsubsection{Sharpened Off-Diagonal Operator Norm}

In general, for a matrix $\mat{A}$, we define the operator norm as $\norm{\mat{A}}_{\mathrm{op}} = \max_{\|\mat{x}\| \neq 0} \frac{\|\mat{A}\mat{x}\|}{\|\mat{x}\|}$. We also establish a sharper bound on $\norm{\Sigma_{ij}}_{\mathrm{op}}$. While the Frobenius bound gives $\norm{\Sigma_{ij}}_{\mathrm{op}} \leq \norm{\Sigma_{ij}}_F = \mathcal{O}(1/\sqrt{d})$, the circulant structure of $\Sigma_{ij}$ yields a tighter estimate. Recall $(\Sigma_{ij})_{ab} = \tfrac{1}{d}(\mat{v}_j \bind \mat{v}_i)_{a+b}$, i.e.\ $\Sigma_{ij} = \tfrac{1}{d}\,\mathrm{H}(\mat{v}_j \bind \mat{v}_i)$ where $\mathrm{H}(\mat{g})_{ab} = g_{(a+b)\bmod d}$. The matrix $\mathrm{H}(\mat{g})$ is a circulant composed with the index-reversal permutation, $\mathrm{H}(\mat{g}) = P\,\mathrm{Circ}(\mat{g})$ with $P$ orthogonal, so its singular values coincide with those of $\mathrm{Circ}(\mat{g})$, namely the moduli $\{|\hat{g}_k|\}$ of the DFT of $\mat{g}$. Let 
\[
\mat{g} = \mat{v}_j \bind \mat{v}_i,
\qquad 
g_\ell = (\mat{v}_j \bind \mat{v}_i)_\ell,
\]
with all subscripts interpreted modulo \(d\). Then
\[
(\Sigma_{ij})_{ab}
=
\frac{1}{d} g_{a+b},
\]
so
\[
\Sigma_{ij}
=
\frac{1}{d}
\begin{bmatrix}
g_0      & g_1      & g_2      & \cdots & g_{d-1} \\
g_1      & g_2      & g_3      & \cdots & g_0     \\
g_2      & g_3      & g_4      & \cdots & g_1     \\
\vdots   & \vdots   & \vdots   & \ddots & \vdots  \\
g_{d-1}  & g_0      & g_1      & \cdots & g_{d-2}
\end{bmatrix}.
\]
Equivalently,
\[
\Sigma_{ij}
=
\frac{1}{d}\,\mathrm{H}(\mat{g}),
\qquad
\mathrm{H}(\mat{g})_{ab}
=
g_{(a+b)\bmod d}.
\]
Using the convention
\[
\mathrm{Circ}(\mat{g})_{ab}
=
g_{(b-a)\bmod d},
\]
we have
\[
\mathrm{Circ}(\mat{g})
=
\begin{bmatrix}
g_0      & g_1      & g_2      & \cdots & g_{d-1} \\
g_{d-1}  & g_0      & g_1      & \cdots & g_{d-2} \\
g_{d-2}  & g_{d-1}  & g_0      & \cdots & g_{d-3} \\
\vdots   & \vdots   & \vdots   & \ddots & \vdots  \\
g_1      & g_2      & g_3      & \cdots & g_0
\end{bmatrix}.
\]
Now define the index-reversal permutation matrix \(P\) by
\[
P_{ab}
=
\mathbf{1}\{b \equiv -a \pmod d\}.
\]
Explicitly,
\[
P
=
\begin{bmatrix}
1      & 0      & 0      & \cdots & 0      \\
0      & 0      & 0      & \cdots & 1      \\
0      & 0      & \dots  & 1      & 0      \\
\vdots & \vdots & \vdots & \vdots & \vdots \\
0      & 1      & 0      & \cdots & 0
\end{bmatrix}.
\]
Therefore,
\[
\mathrm{H}(\mat{g})
=
P\,\mathrm{Circ}(\mat{g}),
\]
and hence
\[
\Sigma_{ij}
=
\frac{1}{d}P\,\mathrm{Circ}(\mat{v}_j \bind \mat{v}_i).
\]
Using that binding multiplies DFT coefficients pointwise,
\[
\norm{\Sigma_{ij}}_{\mathrm{op}}
= \tfrac{1}{d}\max_{k} \bigl|\widehat{(\mat{v}_j \bind \mat{v}_i)}_k\bigr|
= \tfrac{1}{d}\max_{k} |\hat{\mat{v}}_{j,k}|\,|\hat{\mat{v}}_{i,k}|.
\]
For $\mat{v} \sim$ the HRR distribution, each transform coefficient is complex Gaussian with
$\E|\hat{\mat{v}}_k|^2 = 1$, so with high probability over the codebook
$\max_k |\hat{\mat{v}}_k|^2 = \mathcal{O}(\log d)$, giving
\[
\norm{\Sigma_{ij}}_{\mathrm{op}} = \mathcal{O}\!\left(\tfrac{\log d}{d}\right).
\]

\subsection{A Posterior Precision Floor}

The log-determinant reduction developed below requires the joint covariance $\Sigma$ to be
nonsingular. We first show that, for the idealized noiseless retrieval model, $\Sigma$ is in
fact singular. The key point is that the retrieved slots follow an exact deterministic linear constraint. Recall that the latent structure that we use takes the form
\[
\mat{M}
=
\sum_{k=1}^{m} \mat{s}_k \bind \mat{v}_k,
\]
and the retrieval of slot \(i\) is
\[
\tilde{\mat{v}}_i
=
\mat{s}_i^{\inv} \bind \mat{M}.
\]
We analyze this relation in the Fourier domain, where circular convolution and
correlation diagonalize. For a real vector \(\mat{u}\in\R^d\), write its discrete
Fourier transform as
\[
\hat{u}_\omega
=
\sum_{n=0}^{d-1} u_n e^{-2\pi i \omega n/d},
\qquad
\omega = 0,\dots,d-1.
\]
The circular convolution theorem gives
\[
\widehat{(\mat{a}\bind \mat{b})}_\omega
=
\hat{a}_\omega \hat{b}_\omega.
\]
Similarly, for the unbinding operation
\[
(\mat{c}^{\inv}\bind \mat{y})_j
=
\sum_{k=0}^{d-1} c_k y_{k+j},
\]
a direct computation gives
\[
\widehat{(\mat{c}^{\inv}\bind \mat{y})}_\omega
=
\overline{\hat{c}_\omega}\,\hat{y}_\omega.
\]
Therefore,
\[
\widehat{\mat{M}}_\omega
=
\sum_{k=1}^{m} \hat{s}_{k,\omega}\hat{v}_{k,\omega}.
\]
Define the shared Fourier coefficient
\[
Z_\omega
:=
\widehat{\mat{M}}_\omega
=
\sum_{k=1}^{m} \hat{s}_{k,\omega}\hat{v}_{k,\omega}.
\]
Then the retrieved slot satisfies
\[
\widehat{\tilde{\mat{v}}}_{i,\omega}
=
\overline{\hat{s}_{i,\omega}} Z_\omega.
\]
Thus every slot's retrieval at frequency \(\omega\) is obtained from the same scalar
\(Z_\omega\), modulated by that slot's own symbol coefficient.

Now weighing each retrieved coefficient by the conjugate of the corresponding target coefficient and summing over slots, we get
\[
\sum_{i=1}^{m}
\overline{\hat{v}_{i,\omega}}\,
\widehat{\tilde{\mat{v}}}_{i,\omega}
=
\sum_{i=1}^{m}
\overline{\hat{v}_{i,\omega}}\,
\overline{\hat{s}_{i,\omega}}\,
Z_\omega.
\]
Factoring out \(Z_\omega\), we obtain
\[
\sum_{i=1}^{m}
\overline{\hat{v}_{i,\omega}}\,
\widehat{\tilde{\mat{v}}}_{i,\omega}
=
Z_\omega
\sum_{i=1}^{m}
\overline{\hat{s}_{i,\omega}\hat{v}_{i,\omega}}
=
Z_\omega
\overline{
\sum_{i=1}^{m}
\hat{s}_{i,\omega}\hat{v}_{i,\omega}
}.
\]
By the definition of \(Z_\omega\), this becomes
\[
\sum_{i=1}^{m}
\overline{\hat{v}_{i,\omega}}\,
\widehat{\tilde{\mat{v}}}_{i,\omega}
=
Z_\omega \overline{Z_\omega}
=
|Z_\omega|^2.
\]
The right-hand side is real-valued for every realization of the symbols. Hence
\[
\operatorname{Im}
\left(
\sum_{i=1}^{m}
\overline{\hat{v}_{i,\omega}}\,
\widehat{\tilde{\mat{v}}}_{i,\omega}
\right)
=
0
\]
identically.

Conditioned on the codebook \(\{\mat{v}_i\}_{i=1}^m\), the coefficients
\(\overline{\hat{v}_{i,\omega}}\) are fixed. Also, the Fourier transform is a fixed
invertible linear map on the real coordinates of each retrieved vector. Therefore
\[
L_\omega(\tilde{\mat{V}}_{1:m})
:=
\operatorname{Im}
\left(
\sum_{i=1}^{m}
\overline{\hat{v}_{i,\omega}}\,
\widehat{\tilde{\mat{v}}}_{i,\omega}
\right)
\]
is a fixed real-linear functional of the joint retrieval vector
\[
\tilde{\mat{V}}_{1:m}
=
(\tilde{\mat{v}}_1,\dots,\tilde{\mat{v}}_m)
\in \R^{md}.
\]
The identity above says that
\[
L_\omega(\tilde{\mat{V}}_{1:m}) = 0
\qquad
\text{almost surely}.
\]
Equivalently, there exists a real vector \(a_\omega\in\R^{md}\), depending only on
the fixed codebook and on the Fourier frequency, such that
\[
a_\omega^\top \tilde{\mat{V}}_{1:m} = 0
\qquad
\text{almost surely}.
\]
Thus
\[
\operatorname{Var}
\left(
a_\omega^\top \tilde{\mat{V}}_{1:m}
\right)
=
a_\omega^\top \Sigma a_\omega
=
0.
\]
Because \(\Sigma\) is a covariance matrix, it is positive semidefinite. Hence
\(a_\omega^\top \Sigma a_\omega=0\) implies
\[
\Sigma a_\omega = 0.
\]
Therefore \(a_\omega\) is a nontrivial null direction of \(\Sigma\), so \(\Sigma\) is
not full rank and
\[
\det \Sigma = 0.
\]

Under the joint Gaussian model, this makes
\(\log\det(\bar{\Sigma}^{\inv}\Sigma)\) diverge to \(-\infty\), so the log-determinant
expression for the total correlation becomes infinite. This singularity is therefore a property
of the idealized noiseless retrieval model, not of the log-determinant reduction itself.

The resolution is that the variational model does not produce deterministic slots.
Each posterior factor \(q_i(\tilde{\mat{v}}_i\mid \mat{x})\) is a non-degenerate
density, since the encoder assigns the slot a finite posterior precision. We model
this posterior uncertainty as independent additive Gaussian noise across slots and
coordinates, with per-coordinate variance \(\tau^2>0\). Under this regularization,
the joint covariance becomes
\[
\Sigma_\tau
=
\Sigma + \tau^2 \mat{I}_{md},
\]
and the block-diagonal marginal covariance becomes
\[
\bar{\Sigma}_\tau
=
\bar{\Sigma} + \tau^2 \mat{I}_{md}.
\]
The off-diagonal blocks \(\Sigma_{ij}\), \(i\neq j\), are unchanged, while every
degenerate direction receives variance \(\tau^2\). Since \(\Sigma\succeq 0\), we have
\[
\Sigma_\tau
=
\Sigma + \tau^2 \mat{I}_{md}
\succ 0.
\]
Thus the regularized covariance is nonsingular, and the log-determinant expression is
well-defined.

Henceforth, in the total-correlation bound, we replace \(\bar{\Sigma}\) by
\(\bar{\Sigma}_\tau = \bar{\Sigma}+\tau^2\mat{I}_{md}\). A useful consequence is that,
because \(\bar{\Sigma}\succeq 0\),
\[
\lambda_{\min}(\bar{\Sigma}+\tau^2\mat{I}_{md})
\geq
\tau^2.
\]
Therefore
\begin{equation}
\norm{
(\bar{\Sigma}+\tau^2\mat{I}_{md})^{-1/2}
}_{\mathrm{op}}
=
\frac{1}{
\sqrt{
\lambda_{\min}(\bar{\Sigma}+\tau^2\mat{I}_{md})
}
}
\leq
\tau^{-1}.
\label{eq:white-bound}
\end{equation}
This gives a posterior precision floor for the whitening operation. Its operator norm is controlled by \(\tau\) alone, and the conditioning of the noiseless diagonal blocks
\(\Sigma_{ii}\) never enters.

\subsection{Reduction to the Log-Determinant} \label{sec:reduction}

\begin{remark}
A natural approach to bounding~\eqref{eq:sum-of-mi} would be to work with each
$\I{\tilde{\mat{v}}_i}{\tilde{\mat{V}}_{<i}\mid \mat{x}}$ individually. For jointly
Gaussian variables, each such term equals $-\tfrac{1}{2}\log\det(\mat{I} - R_i)$, where
$R_i = \Sigma_{ii}^{-1/2}\Sigma_{i,<i}\Sigma_{<i,<i}^{-1}\Sigma_{<i,i}\Sigma_{ii}^{-1/2}$
is the canonical correlation matrix. Bounding $R_i$ via the submultiplicative chain
\[
\norm{R_i}_{\mathrm{op}} \leq \norm{\Sigma_{ii}^{-1/2}}_{\mathrm{op}}^2 \,
\norm{\Sigma_{i,<i}}_{\mathrm{op}}^2 \, \norm{\Sigma_{<i,<i}^{\inv}}_{\mathrm{op}}
\]
would require $\norm{\Sigma_{<i,<i}^{\inv}}_{\mathrm{op}} = 1/\lambda_{\min}(\Sigma_{<i,<i})$,
which is the \emph{smallest} eigenvalue of $\Sigma_{<i,<i}$ and is not controlled by the
operator norms we can establish. Moreover, even granting such a bound, the
submultiplicative product discards the alignment between the ill-conditioned directions
of $\Sigma_{<i,<i}^{\inv}$ and the directions in which $\Sigma_{i,<i}$ carries mass, so
the resulting estimate would not be tight. We therefore do not bound the $R_i$
individually.
\end{remark}

Since we are ultimately interested in the aggregate quantity
established in~\eqref{eq:sum-of-mi},
\[
\sum_{i=2}^{m} \I{\tilde{\mat{v}}_i}{\tilde{\mat{V}}_{<i}\mid \mat{x}}
\;=\; D_{\mathrm{KL}}\!\bigl(q \,\|\, \textstyle\prod_i q_i\bigr),
\]
we bound this total correlation directly. Under the joint Gaussian model with mean
$(\mat{v}_1(\mat{x}),\dots,\mat{v}_m(\mat{x}))$ and covariance $\Sigma$, this KL admits the
closed form
\[
D_{\mathrm{KL}}\!\bigl(q \,\|\, \textstyle\prod_i q_i\bigr)
\;=\; \tfrac{1}{2}\!\left[\operatorname{tr}\!\bigl(\bar{\Sigma}^{\inv}\Sigma\bigr) - md
- \log\det\!\bigl(\bar{\Sigma}^{\inv}\Sigma\bigr)\right],
\]
where $\bar{\Sigma} = \operatorname{blkdiag}(\Sigma_{11},\dots,\Sigma_{mm})$ is the
covariance of the product of marginals $\prod_i q_i$.

Write $\Delta = \Sigma - \bar{\Sigma}$ for the part of $\Sigma$ consisting of its
off-diagonal blocks $\Sigma_{ij}$ ($i \neq j$); by construction $\Delta$ has vanishing
diagonal blocks. Then $\bar{\Sigma}^{\inv}\Sigma = \mat{I} + \bar{\Sigma}^{\inv}\Delta$, and
because $\bar{\Sigma}^{\inv}$ is block-diagonal while $\Delta$ has no diagonal blocks,
\[
\operatorname{tr}\!\bigl(\bar{\Sigma}^{\inv}\Sigma\bigr)
= md + \operatorname{tr}\!\bigl(\bar{\Sigma}^{\inv}\Delta\bigr)
= md + \sum_{i} \operatorname{tr}\!\bigl(\Sigma_{ii}^{\inv}\,\Delta_{ii}\bigr)
= md,
\]
so the trace term cancels exactly. Introducing the symmetric matrix
\[
M \;=\; \bar{\Sigma}^{-1/2}\,\Delta\,\bar{\Sigma}^{-1/2},
\]
which satisfies $\det(\mat{I}+\bar{\Sigma}^{\inv}\Delta) = \det(\mat{I}+M)$ by similarity and $\operatorname{tr}(M) = \operatorname{tr}(\bar{\Sigma}^{\inv}\Delta) = 0$, we are left with
\begin{equation}
\sum_{i=2}^{m} \I{\tilde{\mat{v}}_i}{\tilde{\mat{V}}_{<i}\mid \mat{x}}
\;=\; -\tfrac{1}{2}\log\det(\mat{I}+M)
\;=\; -\tfrac{1}{2}\sum_{k} \log(1+\mu_k),
\label{eq:tc-logdet}
\end{equation}
where $\{\mu_k\}$ are the eigenvalues of $M$. Since $\mat{I}+M = \bar{\Sigma}^{-1/2}\Sigma\, \bar{\Sigma}^{-1/2} \succeq 0$, each eigenvalue satisfies $\mu_k \geq -1$. Note that this reduction requires only the marginal whitening $\bar{\Sigma}^{-1/2}$, and does not invert the conditioning block $\Sigma_{<i,<i}$.

\subsection{Bounding the Total Correlation}

We now bound~\eqref{eq:tc-logdet} with $\bar{\Sigma}$ regularized as above. The eigenvalues of
$M$ obey $\sum_k \mu_k = \operatorname{tr}(M) = 0$ and $\mu_k > -1$. For each $k$, the
second-order Taylor expansion of $\log(1+\cdot)$ with Lagrange remainder gives
\[
\log(1+\mu_k) = \mu_k - \frac{\mu_k^2}{2(1+\zeta_k)^2},
\qquad \zeta_k \ \text{between}\ 0 \ \text{and}\ \mu_k.
\]
In either sign of $\mu_k$ one has $1 + \zeta_k > 1 + \mu_{\min} \geq 1 - \norm{M}_{\mathrm{op}}$,
so provided $\norm{M}_{\mathrm{op}} < 1$,
\[
\log(1+\mu_k) > \mu_k - \frac{\mu_k^2}{2\,(1-\norm{M}_{\mathrm{op}})^2}.
\]
Summing over $k$ and using $\sum_k \mu_k = 0$ and $\sum_k \mu_k^2 = \norm{M}_F^2$,
\begin{equation}
\sum_{i=2}^{m} \I{\tilde{\mat{v}}_i}{\tilde{\mat{V}}_{<i}\mid \mat{x}}
= -\tfrac{1}{2}\sum_k \log(1+\mu_k)
\;\leq\; \frac{\norm{M}_F^2}{4\,(1-\norm{M}_{\mathrm{op}})^2}.
\label{eq:tc-bound}
\end{equation}

We now need to bound the two quantities on the right, $\norm{M}_F^2$, and $\norm{M}_{\mathrm{op}}$. Writing
$W = (\bar{\Sigma} + \tau^2\mat{I})^{-1/2}$ so that $M = W\Delta W$, the
submultiplicativity of the Frobenius norm under operator-norm multiplication together
with~\eqref{eq:white-bound} gives
\[
\norm{M}_F = \norm{W\Delta W}_F \leq \norm{W}_{\mathrm{op}}^2 \,\norm{\Delta}_F
\leq \tau^{-2}\,\norm{\Delta}_F.
\]
Since $\Delta$ collects the off-diagonal blocks, $\norm{\Delta}_F^2 = \sum_{i \neq j}
\norm{\Sigma_{ij}}_F^2$, and substituting the established
$\norm{\Sigma_{ij}}_F^2 = \tfrac{1}{d}\norm{\mat{v}_j \bind \mat{v}_i}^2$ with
$\norm{\mat{v}_j \bind \mat{v}_i}^2 = \mathcal{O}(1)$,
\begin{equation}
\norm{M}_F^2 \;\leq\; \frac{1}{\tau^4}\sum_{i \neq j} \norm{\Sigma_{ij}}_F^2
\;=\; \frac{1}{\tau^4 d}\sum_{i \neq j} \norm{\mat{v}_j \bind \mat{v}_i}^2
\;=\; \mathcal{O}\!\left(\frac{m^2}{\tau^4 d}\right).
\label{eq:Mfro}
\end{equation}
For the operator norm, the same reasoning gives $\norm{M}_{\mathrm{op}} \leq \tau^{-2}
\norm{\Delta}_{\mathrm{op}}$. Applying the triangle inequality and Cauchy-Schwarz to
$\Delta$, which collects all off-diagonal blocks $\Sigma_{ij}$, $i \neq j$,
\[
\norm{\Delta}_{\mathrm{op}} \leq \sqrt{\sum_{i \neq j} \norm{\Sigma_{ij}}_{\mathrm{op}}^2}
= \sqrt{m(m-1)\,\mathcal{O}\!\left(\tfrac{\log^2 d}{d^2}\right)}
= \mathcal{O}\!\left(\tfrac{m \log d}{d}\right),
\]
where we use the sharpened estimate $\norm{\Sigma_{ij}}_{\mathrm{op}} = \mathcal{O}(\log d / d)$
from the preceding section. Therefore $\norm{M}_{\mathrm{op}} = \mathcal{O}\!\bigl(m \log d / (\tau^2 d)\bigr)$. The regularity condition $\norm{M}_{\mathrm{op}} < 1$ of~\eqref{eq:tc-bound} holds
once $d \gtrsim m \log d / \tau^2$, in which regime $(1 - \norm{M}_{\mathrm{op}})^{-2} = 1 +
o(1)$.

Combining~\eqref{eq:tc-bound} and~\eqref{eq:Mfro}, we obtain the final estimate. Conditioned on the codebook, with per-slot posterior precision $\tau^2$ and $d \gtrsim m\log d/\tau^2$,
\[
\boxed{\;\sum_{i=2}^{m} \I{\tilde{\mat{v}}_i}{\tilde{\mat{V}}_{<i}\mid \mat{x}}
\;=\; \mathcal{O}\!\left(\frac{m^2}{\tau^4\, d}\right).\;}
\]
The total correlation among the retrieved slots thus vanishes as $d^{-1}$, grows quadratically in the number of slots, and vanishes with the posterior precision through $\tau^{-4}$.
 
\subsection{Proof of proposition~\ref{prop:capacity} (HRR latent capacity)}
 
We first bound $\I{\mat{x}}{\hat{\mat{z}}}$ by a sum over slots, then bound each term by both an additive white gaussian noise
(AWGN) capacity and a codebook entropy, and take the minimum.
 
\paragraph{Step 1: reduction to a sum}
The quantized symbolic latent $\hat{\mat{z}} = \sum_i \mat{s}_i\bind \hat{\mat{v}}_i$ is a deterministic
function of $(\hat{\mat{v}}_1,\dots,\hat{\mat{v}}_m)$ given the (frozen) symbol vectors. Hence
$\I{\mat{x}}{\hat{\mat{z}}} \le \I{\mat{x}}{\hat{\mat{V}}_{1:m}}$ by the data-processing inequality.
By the chain rule and the identity
$\I{\mat{x}}{\hat{\mat{v}}_i\mid \hat{\mat{V}}_{<i}} - \I{\mat{x}}{\hat{\mat{v}}_i} =
\I{\hat{\mat{v}}_i}{\hat{\mat{V}}_{<i}\mid \mat{x}} - \I{\hat{\mat{v}}_i}{\hat{\mat{V}}_{<i}}$,
\[
\I{\mat{x}}{\hat{\mat{V}}_{1:m}} \;\le\; \sum_{i=1}^{m} \I{\mat{x}}{\hat{\mat{v}}_i}
\;+\; \sum_{i=2}^{m} \I{\hat{\mat{v}}_i}{\hat{\mat{V}}_{<i}\mid \mat{x}}
=
\sum_{i=1}^{m} \I{\mat{x}}{\hat{\mat{v}}_i}
\;+\; \mathcal{O}\left( \frac{m^2}{t^4 d} \right),
\]
where we dropped the term $\I{\hat{\mat{v}}_i}{\hat{\mat{V}}_{<i}}$. Since quantization is applied slotwise, $(\hat{\mathbf{v}}_1,\dots,\hat{\mathbf{v}}_m)$ is a post-processing of $(\tilde{\mathbf{v}}_1,\dots,\tilde{\mathbf{v}}_m)$. Therefore the conditional total correlation among the quantized slots is no larger than that of the continuous retrievals, so Proposition~\ref{prop:slot-indep} applies to the second sum.
 
\paragraph{Step 2: per-slot AWGN bound}
Conditional on the encoder output, slot~$i$ is observed through the unbinding channel $\tilde{\mat{v}}_i = \mat{v}_i(\mat{x}) + \boldsymbol{\epsilon}_i$ where $\boldsymbol{\epsilon}_i$ has covariance $(m/d)\mat{I}_d$ (Appendix~\ref{app:retrieval}). The signal $\mat{v}_i(\mat{x})$ has expected squared norm at most~$1$ (by the HRR distribution and the value regularizer $\mathcal{L}_{\text{value}}$ which targets $\tau_n=1$). The capacity of a $d$-dimensional AWGN channel with average input power $P$ and per-dimension noise variance $\sigma^2$ is $\tfrac{d}{2}\log(1+P/(d\sigma^2))$. Substituting $P=1$ and $\sigma^2 = m/d$ gives per-slot capacity $\tfrac{d}{2}\log(1+\frac{1}{m})$. Quantization can only reduce mutual information (data-processing), so $\I{\mat{x}}{\hat{\mat{v}}_i} \le \I{\mat{x}}{\tilde{\mat{v}}_i} \le \tfrac{d}{2}\log(1+\frac{1}{m})$.
 
\paragraph{Step 3: per-slot codebook bound}
Since $\hat{\mat{v}}_i$ takes values in a finite codebook of size $k$,
$\I{\mat{x}}{\hat{\mat{v}}_i} \le \H(\hat{\mat{v}}_i) \le \log k$.
 
\paragraph{Step 4: combining}
Each term satisfies
$\I{\mat{x}}{\hat{\mat{v}}_i} \le \min\,\bigl(\tfrac{d}{2}\log(1+1/m),\,\log k\bigr)$, so
\[
\I{\mat{x}}{\hat{\mat{z}}} \;\le\; m\cdot \min\,\bigl(\tfrac{d}{2}\log(1+1/m),\,\log k\bigr)
\;=\; \min\,\bigl(m\cdot\tfrac{d}{2}\log(1+1/m),\, m\log k\bigr),
\]
which is the claim. \hfill$\square$
 
\paragraph{Remark on tightness}
The AWGN bound is tight when the encoder produces values $\mat{v}_i(\mat{x})$ that are
Gaussian-distributed with covariance proportional to identity (the capacity-achieving input). In
practice, the combination of the variance regularizer $\mathcal{L}_{\text{value}}$ (which targets
variance $1/d$ per entry, matching the HRR distribution) and quantization toward a learned codebook
drives the encoder toward this regime. The codebook bound $m\log k$ is tight when codebook entries
are used uniformly.
 
\paragraph{Design corollary}
For a target representation rate \(R\) in nats, Proposition~\ref{prop:capacity} requires
\[
m\cdot \frac{d}{2}\log(1+1/m) \ge R
\qquad \text{and} \qquad
m\log k \ge R.
\]
Equivalently,
\[
d \ge \frac{2R}{m\log(1+1/m)}
\qquad \text{and} \qquad
k \ge e^{R/m}.
\]
The first condition captures the continuous HRR retrieval bottleneck, while the second captures
the discrete codebook bottleneck. For large \(m\), since
\(\log(1+1/m)\approx 1/m\), the dimension condition becomes approximately
\[
d \gtrsim 2R.
\]
Thus, increasing $m$ lowers the effective per-slot SNR, but also distributes the target rate across more slots. These effects approximately balance in the large-$m$ regime, causing the continuous HRR capacity to saturate near $d/2$. In contrast, the codebook requirement $k \ge e^{R/m}$ becomes easier to satisfy as $m$ increases, since each slot needs to represent a smaller fraction of the total rate. Together, these conditions suggest to choose $m$ to be large enough to cover the expected number of generative factors, $d$ to be large enough to support the desired total continuous rate, and $k$ large enough to avoid a discrete quantization bottleneck.       

\section{Experimental setup}
\label{app:experimental-setup}

\paragraph{Choice of baselines}

Our baselines are chosen to position HRR against two reference points: established VAE-based disentanglement methods~\cite{higgins2017betavae, chenIsolatingSourcesDisentanglement2018} and methods that share our use of vector quantization~\cite{vandenoordNeuralDiscreteRepresentation2017, hsu2023disentanglement}. 

\paragraph{Hyperparameters for models} As described in Section~\ref{sec:disentanglement_experiments}, to ensure fair comparisons we ran a sweep for each model and dataset over a single hyperparameter, over two seeds. We then computed the average InfoM score across both seeds for each hyperparameter, and took the hyperparameter that had the highest average InfoM scores. The values that were swept over can be found in Table~\ref{tab:baselines_sweep_params}, and are consistent with prior works~\cite{hsu2023disentanglement}. The results of the sweep, the respective hyperparameters used for the main results in Table~\ref{tab:main_results}, can be found in Table~\ref{tab:selected_hyperparameters}.

Our HRR model uses a single set of hyperparameters across all datasets as we found that tuning hyperparameters did not provide much of an improvement in disentanglement performance. The exact values can be found in Table~\ref{tab:model_hyperparams}. An interesting observation early in our investigation is that performance was better when the weight of the regularization on the latent produced by the encoder, $\lambda_{\text{lat}}$, was low. We found similar performance keeping in a range between 0 and $1e-5$. Although we keep it at zero, and thus has no influence on the optimization of our model, we still include $\lambda_{\text{lat}}$ as a design parameter for future works. Although higher weight on $\lambda_{\text{lat}}$ produced latent representations closer to a true expected HRR vector, it tended to trade off disentanglement performance. We hypothesize that giving the encoder more freedom in producing its representations outweighs the benefit of the structure obtained by imposing it. However, even without regularizing the latent, Fig.~\ref{fig:latent_codebook_stats} shows us that the encoder still learns to produce latent representations that are zero mean and that have relatively constant per-dimension variance.

\paragraph{Computational details for experiments}

All models were trained using PyTorch~\cite{paszkePyTorchImperativeStyle2019a} on an internal SLURM cluster. Each training job used a single GPU. Due to limited availability on the shared cluster, runs were scheduled on whichever compatible GPU resources were available, primarily NVIDIA L40S, A30, and V100 GPUs. GPU type affected wall-clock runtime, but all models were trained using the same codebase, datasets, training protocol, evaluation pipeline, and random-seed structure.

On an NVIDIA L40S, a typical HRR training run, as well as the baseline training runs under the same experimental protocol, required slightly over one hour, although runs could take considerably longer on A30 or V100 GPUs. The main results reported in Table~\ref{tab:main_results} required approximately 7 aggregate GPU-days. The codebook-size and latent-dimension comparisons in Appendix~\ref{app:latent_dim_codebook_size_comparisons} required approximately another 7 aggregate GPU-days. The hyperparameter sweeps used to select baseline configurations required approximately 4 aggregate GPU-days. In total, the experiments reported in this paper required approximately 18 aggregate GPU-days. The full research project required additional compute due to preliminary experiments, failed runs, debugging, and exploratory model variants that are not included in the reported results.

After training, we used an NVIDIA GeForce RTX 3090 for inference on trained models and for constructing final figures and visualizations.

\paragraph{Computation of disentanglement metrics with HRRs} Since our final representation, $\hat{\mat{z}}$, is composed of discrete components from the codebook, we use the codebook indices produced by the vector quantization process to compute the disentanglement metrics, as was done for VQ-VAE and QLAE paper~\cite{hsu2023disentanglement}. 

\paragraph{Common experiment variables} All models share the same CNN architecture, inspired from the model used in the QLAE paper~\cite{hsu2023disentanglement}. A detailed description of the autoencoder architecture is in Table~\ref{tab:cnn_architecture}. Additionally, all models share the same optimization configuration, detailed in Table~\ref{tab:optim_hparams}. The HRR model uses one set of hyperparameters for all datasets, which are listed in Table~\ref{tab:model_hyperparams}.

\paragraph{Confidence intervals and bold font criteria}
The results in Table~\ref{tab:main_results} summarize performance over five independent seeds, specifically the sample mean of the scores for each respective metric. For each model, dataset, and metric, we report the sample mean together with a 95\% confidence interval for the mean, in Table~\ref{tab:full_infomec_results}. The interval is centered at the sample mean and uses a Student-$t$ critical value with $n-1$ degrees of freedom:
\[
\bar{x} \pm t_{0.975,n-1}\frac{s}{\sqrt{n}},
\]
where $\bar{x}$ is the sample mean, $s$ is the sample standard deviation computed with Bessel's correction, and $n=5$ is the number of seeds. For bolding, we use a CI-overlap criterion. In each column, we first identify the entry with the highest upper confidence bound. We then use that entry's lower confidence bound as a threshold: any entry whose lower confidence bound meets or exceeds this threshold is also bolded.

\begin{table}[t]
    \centering
    \small
    \begin{tabular}{llccccl}
    \toprule
    \textbf{Stage} & \textbf{Layer} & \textbf{In} & \textbf{Out} & $k$ & $s$ & \textbf{Activation} \\
\midrule
\multicolumn{7}{l}{\textit{Encoder}} \\
\midrule
Conv 1  & Conv2d$^\dagger$       & $C$       & 64        & 4 & 2 & LReLU(0.2) \\
Conv 2  & Conv2d$^\dagger$       & 64        & 128       & 4 & 2 & LReLU(0.2) \\
Conv 3  & Conv2d$^\dagger$       & 128       & 256       & 4 & 2 & LReLU(0.2) \\
Conv 4  & Conv2d$^\dagger$       & 256       & 512       & 4 & 2 & LReLU(0.2) \\
\cmidrule(lr){1-7}
Flatten & ---                    & ---       & 8{,}192   & --- & --- & --- \\
FC 1    & Linear                 & 8{,}192   & 512       & --- & --- & ReLU \\
FC 2    & Linear                 & 512       & $d_z$     & --- & --- & --- \\
\midrule
\multicolumn{7}{l}{\textit{Decoder}} \\
\midrule
Canvas  & Learnable param.\ $\mathbf{h}_0$ & --- & $512\!\times\!4\!\times\!4$ & --- & --- & --- \\
Style   & Linear                 & $d_z$     & 512       & --- & --- & ReLU \\
\cmidrule(lr){1-7}
Up 1    & StyleConvT$^\ddagger$  & 512       & 512       & 4 & 2 & LReLU(0.2) \\
Up 2    & StyleConvT$^\ddagger$  & 512       & 256       & 4 & 2 & LReLU(0.2) \\
Up 3    & StyleConvT$^\ddagger$  & 256       & 128       & 4 & 2 & LReLU(0.2) \\
Up 4    & StyleConvT$^\ddagger$  & 128       & 64        & 4 & 2 & LReLU(0.2) \\
\cmidrule(lr){1-7}
Out     & Conv2d                 & 64        & $C$       & 3 & 1 & --- \\
\bottomrule

    \end{tabular}
    \caption{%
      Architecture and shared training hyperparameters for the CNN autoencoder.
      $d_z$ denotes the latent dimension and $C$ the number of image channels.
      All convolutional layers omit bias terms.
    }
    \label{tab:cnn_architecture}

    \vspace{0.1em}

    \begin{minipage}{0.95\linewidth}
    \footnotesize
    $^\dagger$ Followed by InstanceNorm2d with learnable affine parameters
    ($\gamma$, $\beta$); padding $p=1$.\\
    $^\ddagger$ StyleConvT: ConvTranspose2d ($p=1$) $\to$ LReLU(0.2) $\to$
    InstanceNorm2d (no affine) $\to$ AdaIN, where the per-channel scale and bias
    are predicted from the style vector $\mathbf{w}$ via a linear layer.
    \end{minipage}

    \vspace{5em}

    \begin{tabular}{lc}
    \toprule
    
    \textbf{Hyperparameter} & \textbf{Value} \\
\midrule
Optimizer         & AdamW \\
Max. gradient steps   & 125,000 \\
Learning rate $\eta$     & $3 \times 10^{-4}$ \\
Batch size        & 128 \\
Gradient clip     & 0.5 \\
$\beta_1$         & $0.9$ \\
$\beta_2$         & $0.999$ \\
$\varepsilon$     & $1 \times 10^{-8}$ \\
\bottomrule

    \end{tabular}
    \caption{Optimizer hyperparameters used for all runs}
    \label{tab:optim_hparams}
\end{table}

\begin{table}[t]
\centering
\begin{tabular}{lcc}
  \toprule
  \textbf{Hyperparameter} & \textbf{Symbol} & \textbf{Value} \\
  \midrule
  Latent dimension        & $d_z$               & $512$ \\
  Number of symbols       & $m$                 & $9$ \\
  Codebook size           & $k$                 & $512$ \\
  \midrule
  Commitment cost         & $\beta$             & $0.6$ \\
  Latent reg.\ weight     & $\lambda_{\text{lat}}$  & $0$ \\
  Codebook reg.\ weight   & $\lambda_{\text{cb}}$   & $1 \times 10^{-3}$ \\
  VQ loss weight          & $\lambda_{\text{vq}}$   & $0.25$ \\
  Reconstruction weight   & $\lambda_{\text{rec}}$  & $1.0$ \\
  \midrule
  Learning rate           & $\eta$              & $3 \times  10^{-4}$ \\
  Codebook Learning Rate & $\eta_{\text{cb}}$   & 2$\eta$ \\
  Weight decay            & ---                 & $0$ \\
  \bottomrule
\end{tabular}
\vspace{0.8em}
\caption{Model hyperparameters for the proposed HRR autoencoder. Used for all experiments}
\label{tab:model_hyperparams}
\end{table}

\begin{table}[h]
    \centering
    \begin{tabular}{lll}
        \toprule
        \textbf{method} & \textbf{hyperparameter} & \textbf{values} \\
        \midrule
        $\beta$-VAE~\cite{higgins2017betavae}
            & $\beta = \lambda_{\mathrm{KL}}$
            & $[0.1,\, 0.3,\, 1,\, 3,\, 10]$ \\
        $\beta$-TCVAE~\cite{chenIsolatingSourcesDisentanglement2018}
            & $\beta = \lambda_{\mathrm{total\ correlation}}$
            & $[0.1,\, 0.3,\, 1,\, 3,\, 10]$ \\
        VQ-VAE~\cite{vandenoordNeuralDiscreteRepresentation2017}
            & weight decay
            & $[0.001,\, 0.01,\, 0.1,\, 1]$ \\
        QLAE~\cite{hsu2023disentanglement}
            & weight decay
            & $[0.001,\, 0.01,\, 0.1,\, 1]$ \\
        \bottomrule
    \end{tabular}
    \vspace{0.8em}
    \caption{Regularization hyperparameter tuning for each baseline}
    \label{tab:baselines_sweep_params}
\end{table}

\begin{table}[h]
    \centering
    \small
    \setlength{\tabcolsep}{6pt}
    \renewcommand{\arraystretch}{1.1}
    \begin{tabular}{l c c c c}
        \toprule
        model & Shapes3D & Falcor3D & Isaac3D & MPI3D-C \\
        \midrule
        $\beta$-VAE
    & $\beta = 10$ 
    & $\beta = 10$ 
    & $\beta = 3$ 
    & $\beta = 10$ \\
$\beta$-TCVAE
    & $\beta = 3$ 
    & $\beta = 3$ 
    & $\beta = 3$ 
    & $\beta = 3$ \\
VQ-VAE
    & $\lambda_{\mathrm{wd}} = .001$ 
    & $\lambda_{\mathrm{wd}} = .01$ 
    & $\lambda_{\mathrm{wd}} = .001$ 
    & $\lambda_{\mathrm{wd}} = 1$ \\
QLAE
    & $\lambda_{\mathrm{wd}} = 1$ 
    & $\lambda_{\mathrm{wd}} = 1$ 
    & $\lambda_{\mathrm{wd}} = 1$ 
    & $\lambda_{\mathrm{wd}} = 1$
 \\
        \bottomrule
    \end{tabular}
    \vspace{0.8em}
    \caption{Selected regularization hyperparameters for each model and dataset. Values were chosen from the corresponding hyperparameter sweeps.}
    \label{tab:selected_hyperparameters}
\end{table}

\clearpage

\section{Limitations} \label{app:limitations}

\paragraph{Empirical comparisons} Our empirical comparison is limited by the set of baselines that could be evaluated under a shared experimental protocol. In particular, we considered including Tripod~\cite{10.5555/3692070.3692839}, a recent follow-up to QLAE~\cite{hsu2023disentanglement} that combines multiple inductive biases for disentanglement and reports strong InfoMEC and DCI performance. However, Tripod differs from the other baselines not only in its objective, but also in its architectural assumptions and computational requirements. One component of its objective involves estimating Hessian-based regularization terms during training, which substantially increased training time in our setting. In addition, the architecture proposed for Tripod uses residual connections and differed from the CNN backbone used for the other methods in our comparison.

We attempted to evaluate Tripod in two ways: using the architecture proposed by the original authors, and adapting its Hessian-based regularization to the CNN backbone used in our experiments. The former was prohibitively expensive under our available compute budget, while the latter did not yield stable results consistent with those reported in the original paper. Across five seeds, we observed large variation in InfoM, with some runs reaching high scores and others performing substantially worse. Because we could not determine whether this instability reflected an implementation issue, sensitivity to architectural changes, or hyperparameter dependence, we chose not to include Tripod as a baseline. We believe this is the more conservative choice, since reporting an unreliable reproduction could unfairly understate the performance of Tripod.

As a result, our claims should be interpreted as comparisons against the baselines evaluated under our shared-backbone protocol, rather than as a definitive comparison against all recent disentanglement methods. Our main contribution is the investigation of VSA-inspired representations as an inductive bias for unsupervised disentanglement, rather than establishing absolute state-of-the-art disentanglement performance across all possible architectures and objectives.

\paragraph{Latent-to-source ratio} The choice of using $1.5\times$ latents to sources ratio differs from the $2\times$ ratio that was used in~\cite{hsu2023disentanglement}. The reason for this is that unbinding performance degrades with $m$ in our model since the unbinding noise becomes worse the more symbol-value pairs are bound together. Thus, we use a $1.5\times$ ratio to maintain the difficulty of the problem without penalizing performance of the HRR-based model. 

\paragraph{Increased autoencoder parameter count with HRRs} We point out that due to the larger size of the latent vector produced by our model, the autoencoder inevitably ends up having more parameters due to the larger projection matrices at the bottleneck. In our experiments, we found that this caused the autoencoder for our HRR model to have around 15.4M parameters, while the baselines had around 14.9M parameters. Although this causes the model architectures to not be completely identical, we argue that this is a feature of our model design, as the representations of prior disentanglement works are limited to keeping vector dimensionality close to the number of ground truth factors. Thus they are not able to take advantage of increasing model capacity without also hurting disentanglement performance by creating more latent units than necessary.

\paragraph{Broader Impacts} \label{app:impact_statement}

This work is primarily methodological and studies VSA-inspired inductive biases for unsupervised disentanglement on benchmark datasets. A potential positive impact of this work is that more structured latent representations may improve the interpretability, modularity, and controllability of learned representations. Such properties could be useful in scientific modeling, representation analysis, controllable generation, and downstream systems where understanding the factors encoded by a model is important.

At the same time, improvements in disentangled representation learning may also have negative downstream uses. More controllable latent representations can make it easier to manipulate generated or reconstructed content along semantically meaningful axes. In generative modeling contexts, this could contribute to misuse such as deceptive media editing, impersonation, or other forms of content manipulation. More interpretable representations could also be misused in surveillance or profiling settings if applied to sensitive human data.

\section{Effect of HRR latent dimension and codebook size} \label{app:latent_dim_codebook_size_comparisons}

Although we chose a single configuration for our HRR model for results, we run sweeps over each dataset for 9 combinations of different latent dimension and codebook sizes. The results for each disentanglement metric, modularity, explicitness, and compactness, are included in Figs.~\ref{fig:hrr_infom_sweep}~\ref{fig:hrr_infoe_sweep}~\ref{fig:hrr_infoc_sweep}, respectively.

\begin{figure}[h]
    \centering
    \includegraphics[width=1.0\linewidth]{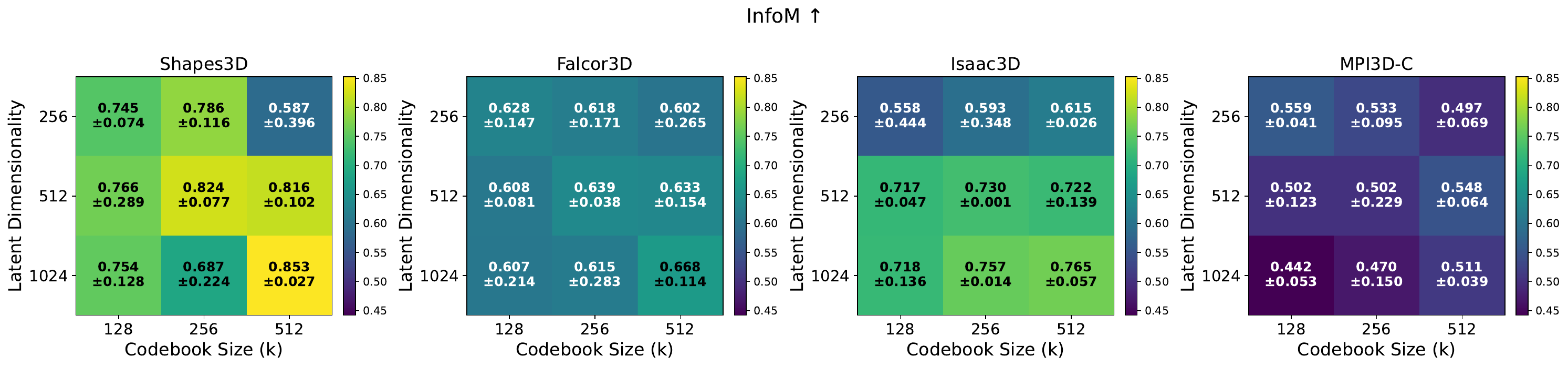}
    \caption{InfoM performance over different combinations of codebook size and latent dimension}
    \label{fig:hrr_infom_sweep}
\end{figure}

\begin{figure}[h]
    \centering
    \includegraphics[width=1.0\linewidth]{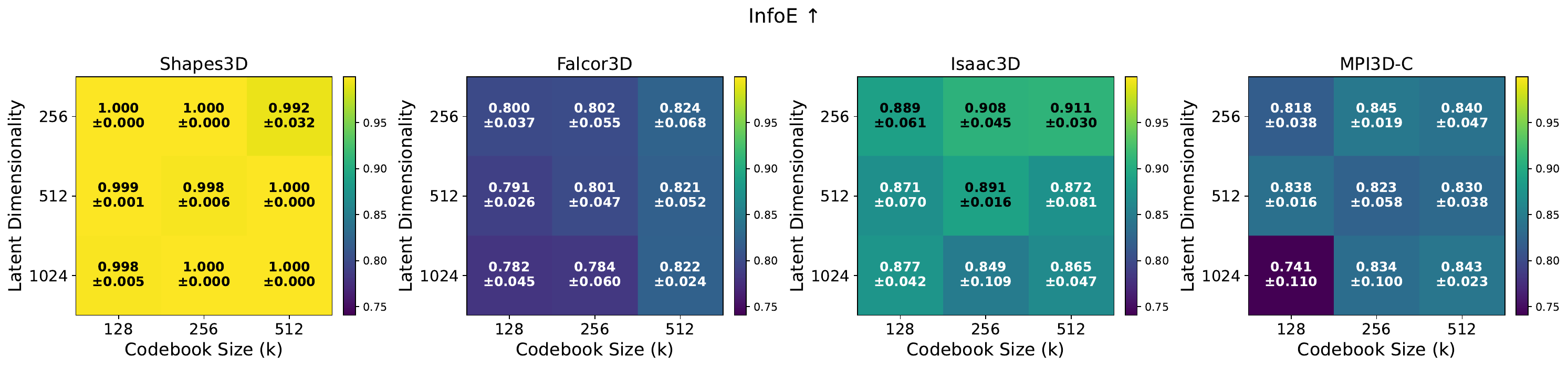}
    \caption{InfoE performance over different combinations of codebook size and latent dimension}
    \label{fig:hrr_infoe_sweep}
\end{figure}

\begin{figure}[t]
    \centering
    \includegraphics[width=1.0\linewidth]{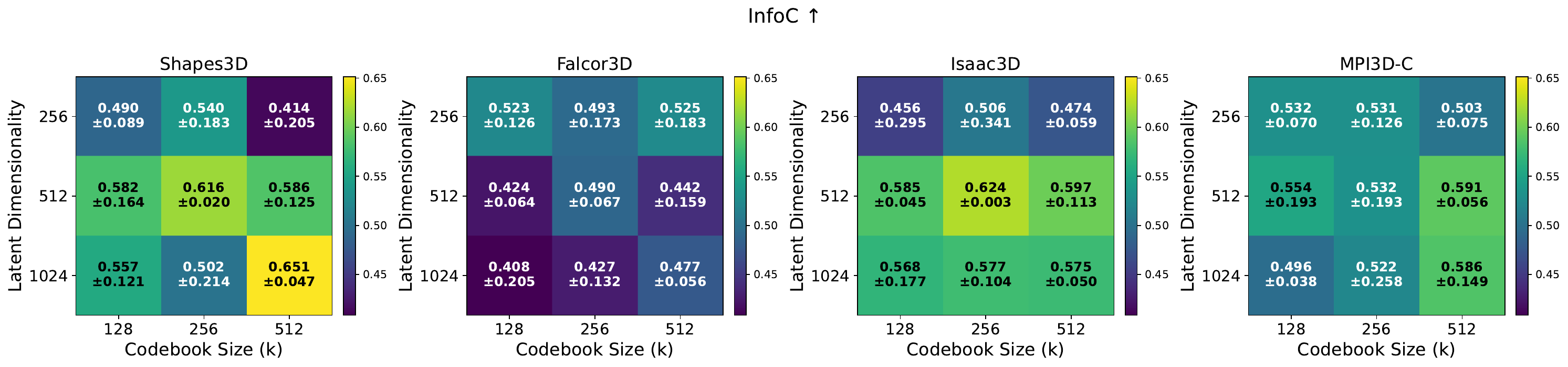}
    \caption{InfoC performance over different combinations of codebook size and latent dimension}
    \label{fig:hrr_infoc_sweep}
\end{figure}

\clearpage

\subsection{Negative results for the HRR model}

In the design of our model we ablated a few design choices but found that they either had no benefit, or ended up hurting the model's performance. We list the most notable of those below.

\begin{enumerate}
    \item Initializing the symbol vectors, $\mathcal{S}$, with an improved scheme proposed in a prior work~\cite{ganesanLearningHolographicReduced2021a}.
    \item Setting the codebook, $\mathcal{C}$, to be static.
    \item Setting the symbol vectors, $\mathcal{S}$ to be learnable.
    \item Removing the pre-processing network described in Section~\ref{ssec:model_arch}.
    \item Directly keeping approximate-HRR vectors in the codebook, as opposed to the parameterized approach described in the previous paragraph.
\end{enumerate}

\section{Complete disentanglement results} \label{app:complete_disentanglement_results}

The full disentanglement results for each dataset, with the mean and 95\% confidence intervals described in Appendix~\ref{app:experimental-setup}, can be found in Table~\ref{tab:full_infomec_results}.

\newpage
\begin{table}[t]
    \centering
    \tiny
    \setlength{\tabcolsep}{5pt}
    
    \begin{tabular}{l c c c c c c c}
        \toprule
        model & InfoM $\uparrow$ & InfoE $\uparrow$ & InfoC $\uparrow$ & D $\uparrow$ & I $\uparrow$ & C $\uparrow$ & PSNR (dB) $\uparrow$ \\
        \midrule
$\beta$-VAE & $0.59 \pm 0.04$ & $0.8 \pm 0.1$ & $0.47 \pm 0.04$ & $0.64 \pm 0.09$ & $\mathbf{0.997 \pm 0.005}$ & $\mathbf{0.55 \pm 0.06}$ & $\mathbf{38.9 \pm 0.3}$ \\
$\beta$-TCVAE & $0.6 \pm 0.1$ & $0.8 \pm 0.1$ & $0.51 \pm 0.04$ & $0.6 \pm 0.2$ & $\mathbf{0.98 \pm 0.03}$ & $\mathbf{0.5 \pm 0.2}$ & $\mathbf{40.8 \pm 0.8}$ \\
VQ-VAE & $0.6 \pm 0.1$ & $\mathbf{1.000 \pm 0.005}$ & $0.44 \pm 0.06$ & $0.74 \pm 0.07$ & $\mathbf{0.98 \pm 0.01}$ & $0.35 \pm 0.02$ & $\mathbf{44.8 \pm 0.3}$ \\
QLAE & $0.8 \pm 0.1$ & $\mathbf{0.998 \pm 0.005}$ & $0.52 \pm 0.07$ & $0.76 \pm 0.08$ & $\mathbf{0.993 \pm 0.006}$ & $0.35 \pm 0.04$ & $36 \pm 1$ \\
HRR (ours) & $\mathbf{0.85 \pm 0.05}$ & $\mathbf{0.999 \pm 0.005}$ & $\mathbf{0.61 \pm 0.04}$ & $\mathbf{0.88 \pm 0.02}$ & $\mathbf{0.995 \pm 0.007}$ & $\mathbf{0.41 \pm 0.02}$ & $\mathbf{41 \pm 4}$ \\
        \bottomrule
    \end{tabular}
    \vspace{0.8em}
    \caption{Full results for Shapes3D}
    \label{tab:full_results_3dshapes}
    \vspace{1.0em}

    \begin{tabular}{l c c c c c c c}
        \toprule
        model & InfoM $\uparrow$ & InfoE $\uparrow$ & InfoC $\uparrow$ & D $\uparrow$ & I $\uparrow$ & C $\uparrow$ & PSNR (dB) $\uparrow$ \\
        \midrule
        $\beta$-VAE & $\mathbf{0.64 \pm 0.03}$ & $0.68 \pm 0.02$ & $\mathbf{0.59 \pm 0.07}$ & $0.40 \pm 0.04$ & $0.83 \pm 0.01$ & $\mathbf{0.39 \pm 0.04}$ & $29.8 \pm 0.2$ \\
$\beta$-TCVAE & $\mathbf{0.6 \pm 0.2}$ & $0.65 \pm 0.02$ & $\mathbf{0.6 \pm 0.2}$ & $0.3 \pm 0.2$ & $0.79 \pm 0.06$ & $\mathbf{0.3 \pm 0.2}$ & $31.0 \pm 0.1$ \\
VQ-VAE & $\mathbf{0.50 \pm 0.07}$ & $\mathbf{0.86 \pm 0.03}$ & $0.40 \pm 0.03$ & $\mathbf{0.64 \pm 0.04}$ & $0.81 \pm 0.02$ & $\mathbf{0.340 \pm 0.007}$ & $\mathbf{31.4 \pm 0.2}$ \\
QLAE & $\mathbf{0.61 \pm 0.07}$ & $0.79 \pm 0.01$ & $\mathbf{0.46 \pm 0.07}$ & $0.50 \pm 0.03$ & $0.84 \pm 0.02$ & $\mathbf{0.28 \pm 0.01}$ & $28.3 \pm 0.4$ \\
HRR (ours) & $\mathbf{0.63 \pm 0.05}$ & $0.81 \pm 0.01$ & $\mathbf{0.45 \pm 0.07}$ & $0.58 \pm 0.05$ & $\mathbf{0.85 \pm 0.02}$ & $\mathbf{0.37 \pm 0.01}$ & $29.0 \pm 0.6$ \\
        \bottomrule
    \end{tabular}
    \vspace{0.8em}
    \caption{Full results for Falcor3D}
    \label{tab:full_results_falcor3d}
    \vspace{1.0em}

    \begin{tabular}{l c c c c c c c}
        \toprule
        model & InfoM $\uparrow$ & InfoE $\uparrow$ & InfoC $\uparrow$ & D $\uparrow$ & I $\uparrow$ & C $\uparrow$ & PSNR (dB) $\uparrow$ \\
        \midrule
        $\beta$-VAE & $0.54 \pm 0.06$ & $0.56 \pm 0.03$ & $0.46 \pm 0.06$ & $0.31 \pm 0.05$ & $0.84 \pm 0.03$ & $0.30 \pm 0.04$ & $37.0 \pm 0.4$ \\
$\beta$-TCVAE & $0.63 \pm 0.06$ & $0.62 \pm 0.06$ & $\mathbf{0.6 \pm 0.1}$ & $0.35 \pm 0.04$ & $\mathbf{0.87 \pm 0.01}$ & $0.34 \pm 0.04$ & $37.2 \pm 0.4$ \\
VQ-VAE & $0.70 \pm 0.05$ & $\mathbf{0.86 \pm 0.08}$ & $0.5 \pm 0.1$ & $\mathbf{0.64 \pm 0.09}$ & $\mathbf{0.91 \pm 0.07}$ & $\mathbf{0.36 \pm 0.06}$ & $\mathbf{45 \pm 2}$ \\
QLAE & $0.66 \pm 0.06$ & $0.81 \pm 0.05$ & $0.49 \pm 0.05$ & $0.59 \pm 0.04$ & $\mathbf{0.92 \pm 0.03}$ & $0.31 \pm 0.02$ & $38 \pm 4$ \\
HRR (ours) & $\mathbf{0.73 \pm 0.06}$ & $\mathbf{0.89 \pm 0.02}$ & $\mathbf{0.55 \pm 0.05}$ & $\mathbf{0.680 \pm 0.006}$ & $\mathbf{0.92 \pm 0.02}$ & $\mathbf{0.364 \pm 0.005}$ & $\mathbf{43.4 \pm 0.7}$ \\
        \bottomrule
    \end{tabular}
    \vspace{0.8em}
    \caption{Full results for Isaac3D}
    \label{tab:full_results_isaac3d}
    \vspace{1.0em}
    
    \begin{tabular}{l c c c c c c c}
        \toprule
        model & InfoM $\uparrow$ & InfoE $\uparrow$ & InfoC $\uparrow$ & D $\uparrow$ & I $\uparrow$ & C $\uparrow$ & PSNR (dB) $\uparrow$ \\
        \midrule
$\beta$-VAE & $0.37 \pm 0.05$ & $0.23 \pm 0.03$ & $0.29 \pm 0.09$ & $0.08 \pm 0.03$ & $0.57 \pm 0.02$ & $0.10 \pm 0.03$ & $29.5 \pm 0.3$ \\
$\beta$-TCVAE & $0.41 \pm 0.01$ & $0.42 \pm 0.06$ & $0.39 \pm 0.06$ & $0.27 \pm 0.02$ & $0.75 \pm 0.01$ & $0.26 \pm 0.02$ & $\mathbf{35.9 \pm 0.1}$ \\
VQ-VAE & $0.5 \pm 0.1$ & $\mathbf{0.82 \pm 0.03}$ & $\mathbf{0.53 \pm 0.05}$ & $0.49 \pm 0.08$ & $0.73 \pm 0.07$ & $\mathbf{0.34 \pm 0.02}$ & $34.9 \pm 0.5$ \\
QLAE & $\mathbf{0.56 \pm 0.05}$ & $0.78 \pm 0.04$ & $0.48 \pm 0.02$ & $0.51 \pm 0.04$ & $\mathbf{0.78 \pm 0.03}$ & $0.29 \pm 0.02$ & $34.2 \pm 0.3$ \\
HRR (ours) & $\mathbf{0.5 \pm 0.1}$ & $\mathbf{0.83 \pm 0.02}$ & $\mathbf{0.53 \pm 0.06}$ & $\mathbf{0.55 \pm 0.04}$ & $\mathbf{0.769 \pm 0.007}$ & $\mathbf{0.34 \pm 0.02}$ & $\mathbf{37 \pm 1}$ \\
        \bottomrule
    \end{tabular}
    \vspace{0.8em}
    \vspace{1.0em}
    \caption{Full results for MPI3D-C}
    \label{tab:full_results_mpi3d_complex}

\caption{Full InfoMEC results. Bold marks the best per column using the CI-overlap criterion. $\uparrow$: higher is better.}
\label{tab:full_infomec_results}
\end{table}

\section{Additional noise results} \label{app:additional_noise_results}

We include further results for the noise robustness experiment for the remaining datasets. Figs.~\ref{fig:gaussian_noise_grid_isaac3d},~\ref{fig:gaussian_noise_grid_falcor3d},~\ref{fig:gaussian_noise_grid_mpi3d_complex} show the results on Isaac3D, Falcor3D, and MPI3D-C, respectively.

\begin{figure}[h]
    \centering
    \includegraphics[width=1.0\linewidth]{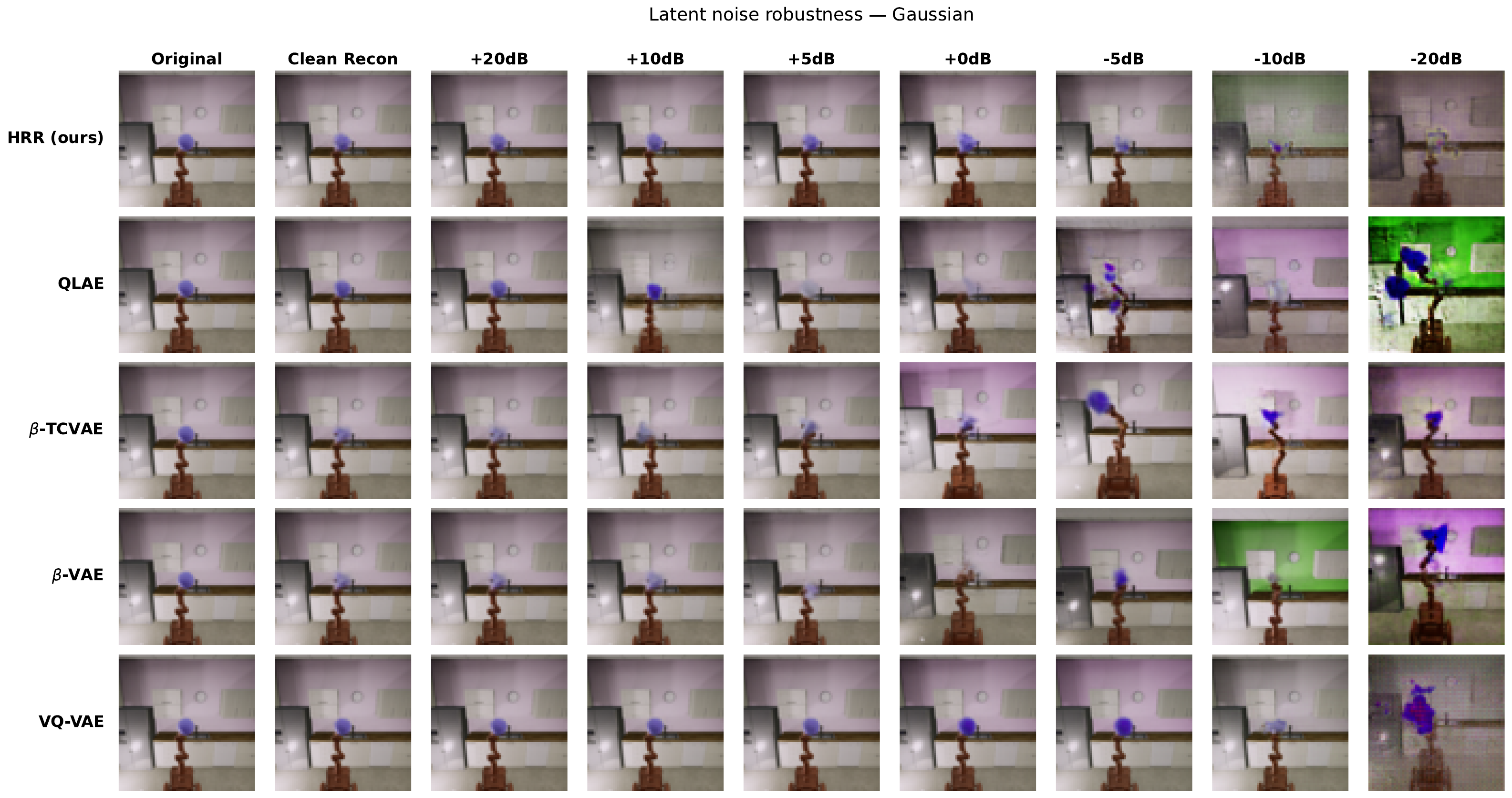}
    \caption{Each row is a visualization of how reconstruction quality degrades as noise intensity increases for each model. For each model, the seed with the highest InfoM score was used for evaluation.}
    \label{fig:gaussian_noise_grid_isaac3d}
\end{figure}

\begin{figure}[h]
    \centering
    \includegraphics[width=1.0\linewidth]{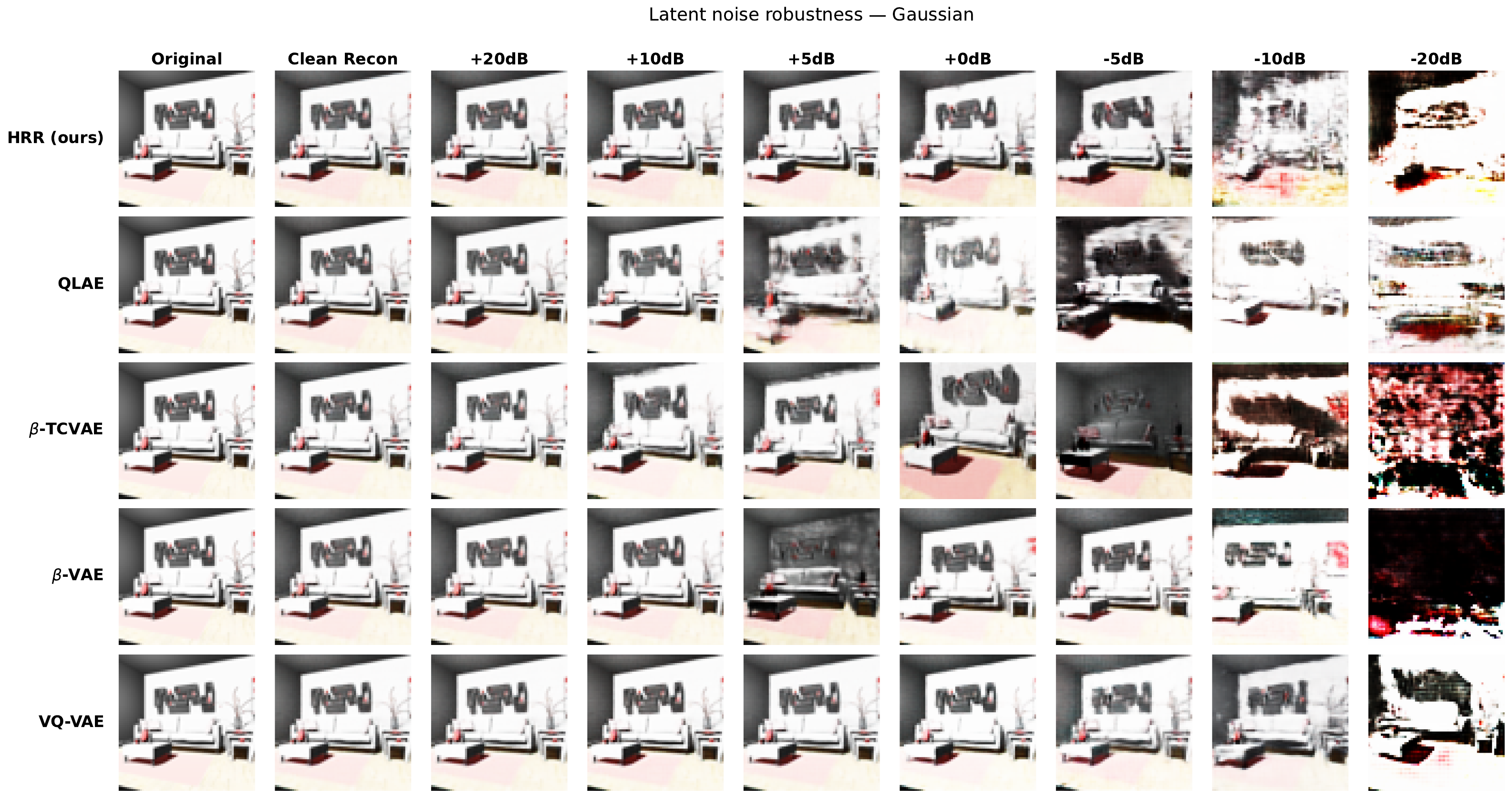}
    \caption{Each row is a visualization of how reconstruction quality degrades as noise intensity increases for each model. For each model, the seed with the highest InfoM score was used for evaluation. }
    \label{fig:gaussian_noise_grid_falcor3d}
\end{figure} 

\begin{figure}[h]
    \centering
    \includegraphics[width=1.0\linewidth]{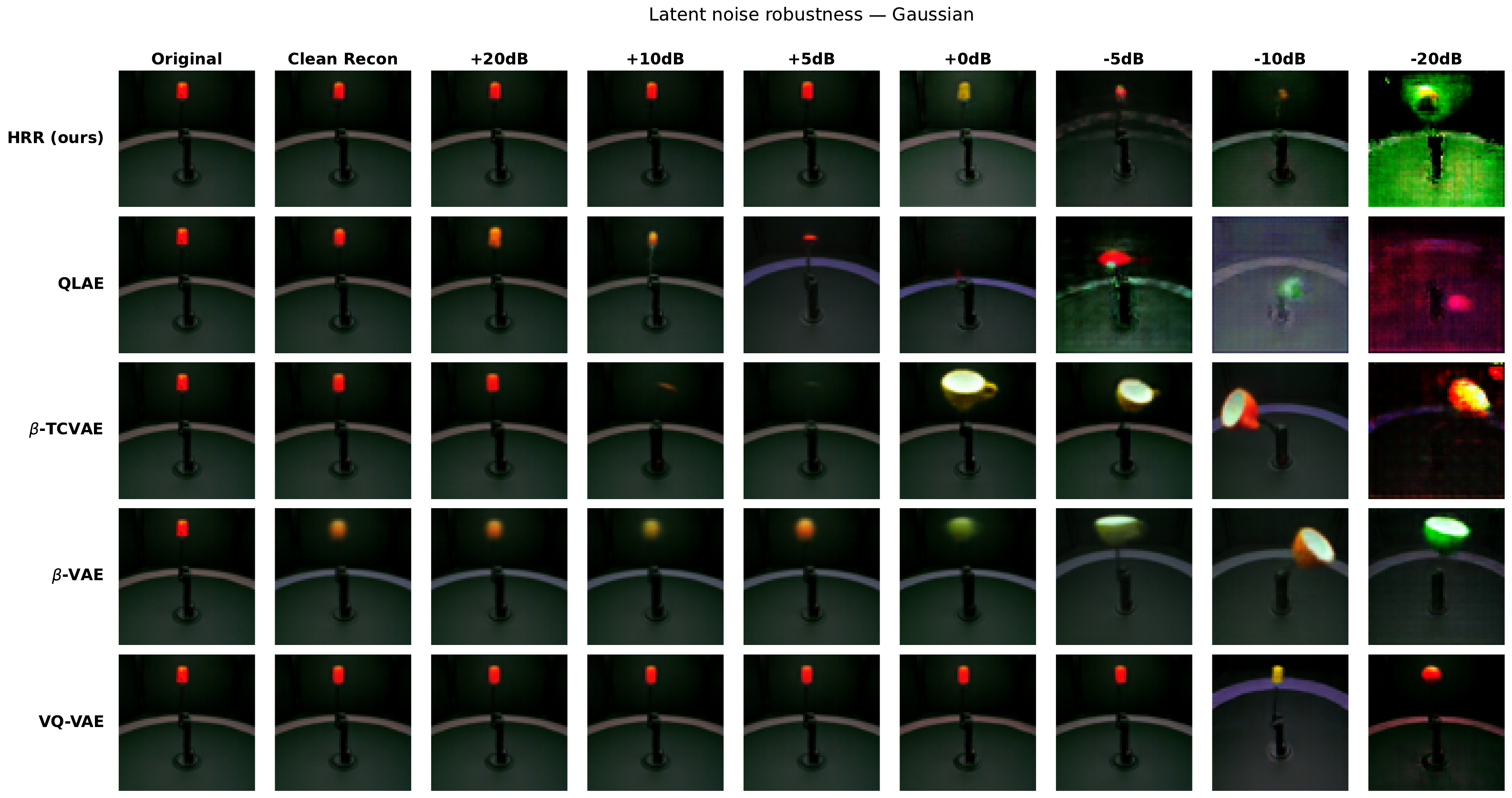}
    \caption{Each row is a visualization of how reconstruction quality degrades as noise intensity increases for each model. For each model, the seed with the highest InfoM score was used for evaluation.}
    \label{fig:gaussian_noise_grid_mpi3d_complex}
\end{figure} 

\clearpage

\section{Additional latent component swaps}

We include further results for the latent component swap experiment for the remaining datasets. Figs.~\ref{fig:component_swap_isaac3d},~\ref{fig:component_swap_falcor3d},~\ref{fig:component_swap_mpi3d_complex} show the results on Isaac3D, Falcor3D, and MPI3D-C, respectively.

\begin{figure}[h]
    \centering
    \includegraphics[width=1.0\linewidth]{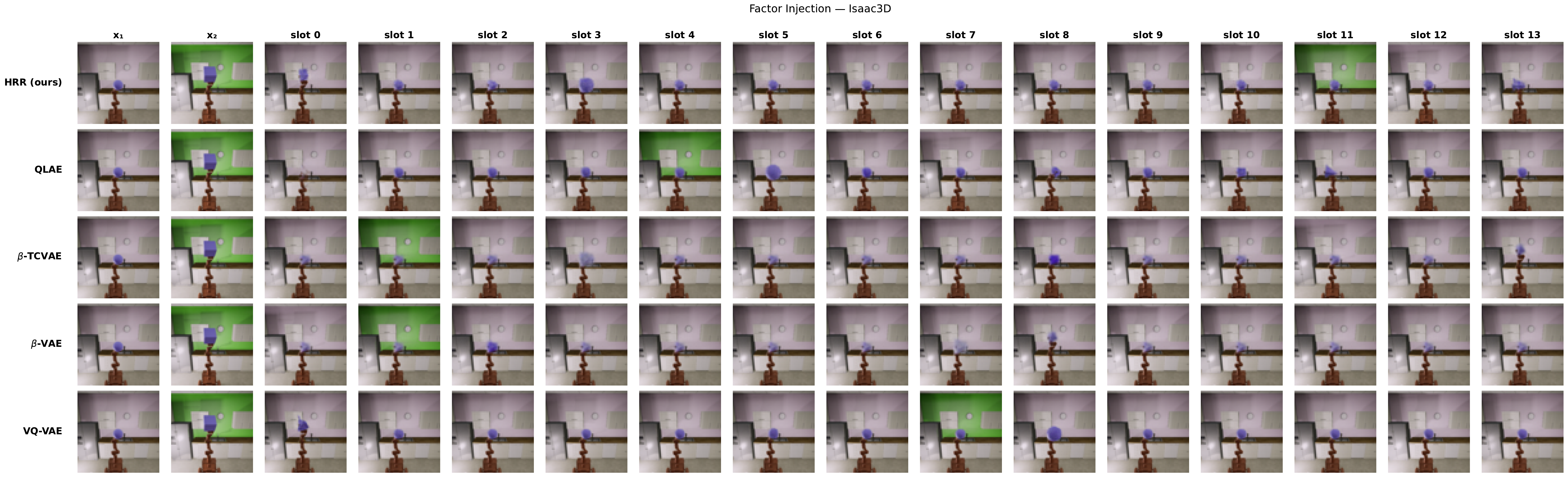}
    \caption{Latent component swaps performance for highest performing model relative to InfoM score. The slot 0 column represents taking the first latent unit from $x_2$, using it for the first latent unit for $x_1$, and decoding the result. Slot 1 does the same, but for the second latent unit, and so on. Each slot is replaced with its original value after visualization, so we expect each decoded result to manipulate either one, or no, aspects of the image.}
    \label{fig:component_swap_isaac3d}
\end{figure}

\begin{figure}[h]
    \centering
    \includegraphics[width=1.0\linewidth]{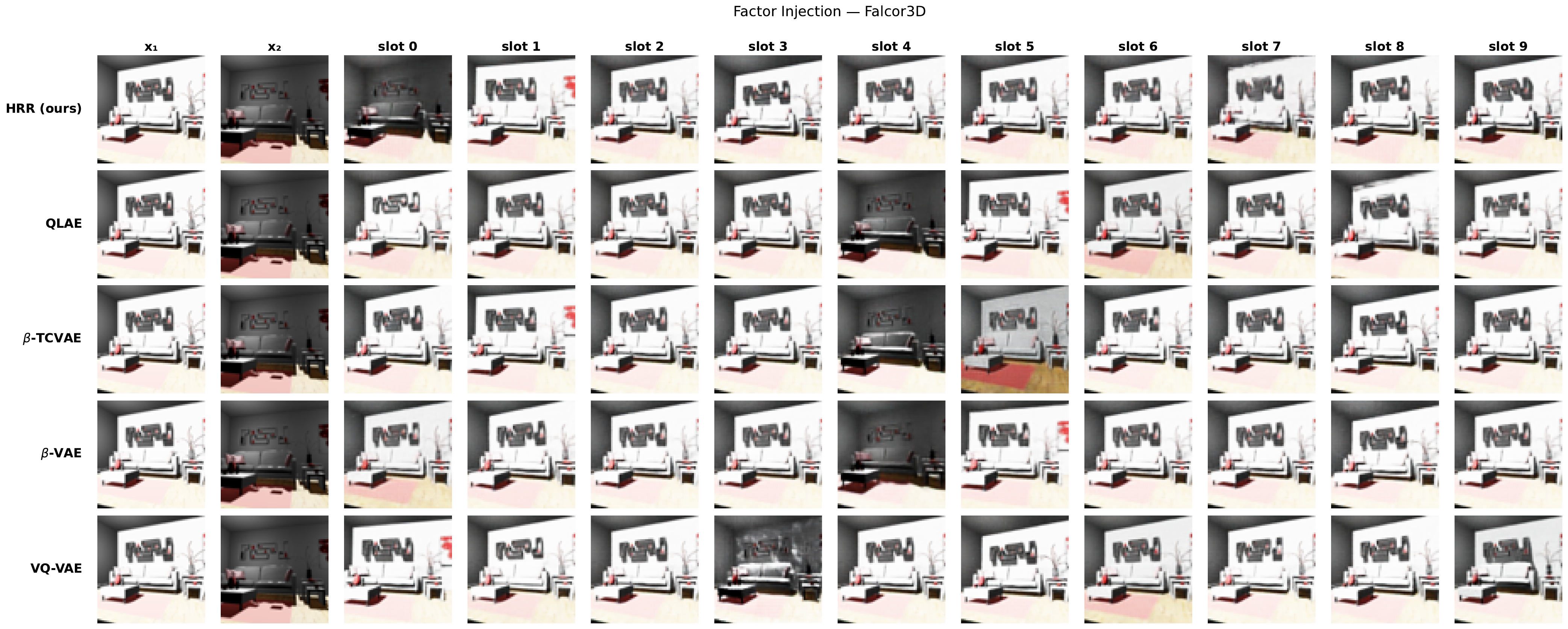}
    \caption{Latent component swaps performance for highest performing model relative to InfoM score. The slot 0 column represents taking the first latent unit from $x_2$, using it for the first latent unit for $x_1$, and decoding the result. Slot 1 does the same, but for the second latent unit, and so on. Each slot is replaced with its original value after visualization, so we expect each decoded result to manipulate either one, or no, aspects of the image.}
    \label{fig:component_swap_falcor3d}
\end{figure}

\begin{figure}[h]
    \centering
    \includegraphics[width=1.0\linewidth]{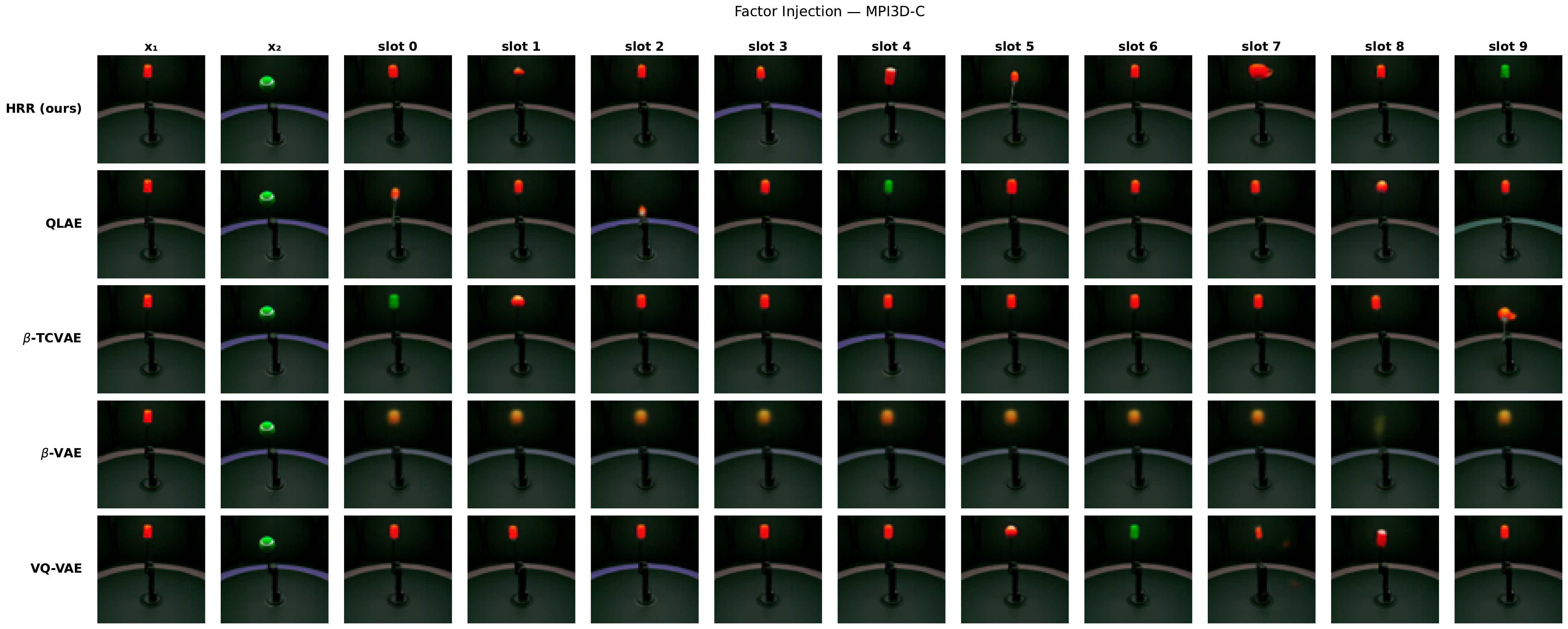}
    \caption{Latent component swaps performance for highest performing model relative to InfoM score. The slot 0 column represents taking the first latent unit from $x_2$, using it for the first latent unit for $x_1$, and decoding the result. Slot 1 does the same, but for the second latent unit, and so on. Each slot is replaced with its original value after visualization, so we expect each decoded result to manipulate either one, or no, aspects of the image.}
    \label{fig:component_swap_mpi3d_complex}
\end{figure}

\clearpage

\section{Additional latent interpolations}

We include further results for the latent interpolation experiment for the remaining datasets. Figs.~\ref{fig:interp_isaac3d},~\ref{fig:interp_falcor3d},~\ref{fig:interp_mpi3d_complex} show the results on Isaac3D, Falcor3D, and MPI3D-C, respectively.

\begin{figure}[h]
    \centering
    \includegraphics[width=1.0\linewidth]{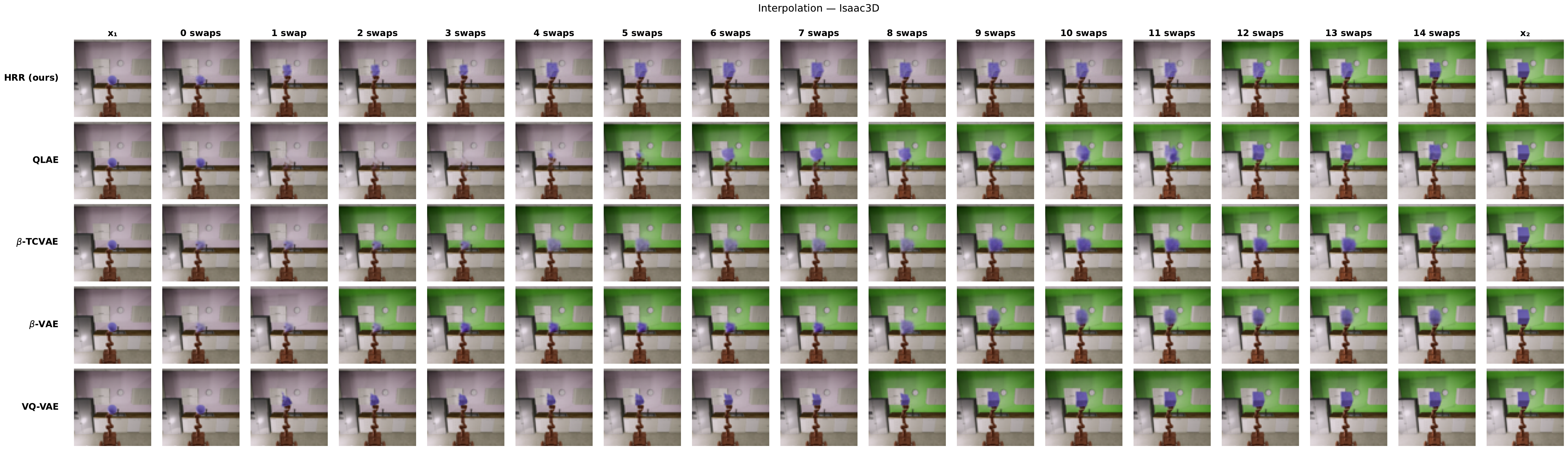}
    \caption{Latent interpolation swaps components progressively until the representation is fully transformed into that of the second image.}
    \label{fig:interp_isaac3d}
\end{figure}

\begin{figure}[h]
    \centering
    \includegraphics[width=1.0\linewidth]{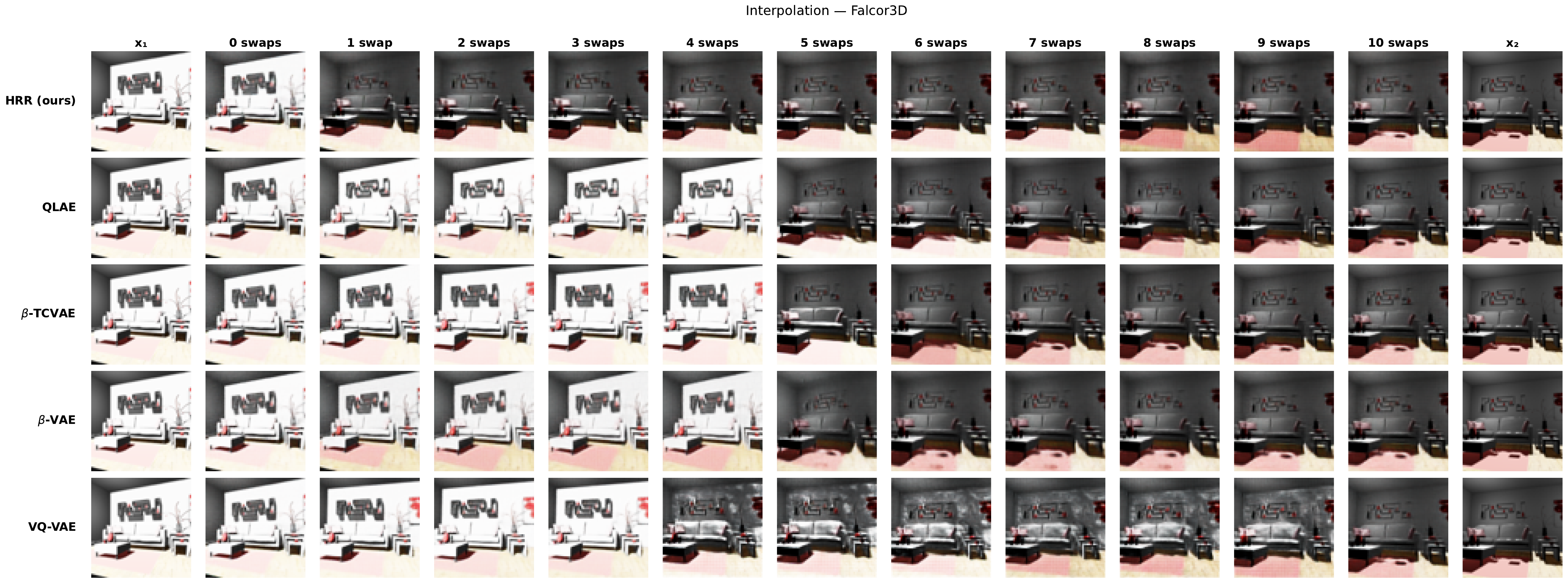}
    \caption{Latent interpolation swaps components progressively until the representation is fully transformed into that of the second image.}
    \label{fig:interp_falcor3d}
\end{figure}

\begin{figure}[h]
    \centering
    \includegraphics[width=1.0\linewidth]{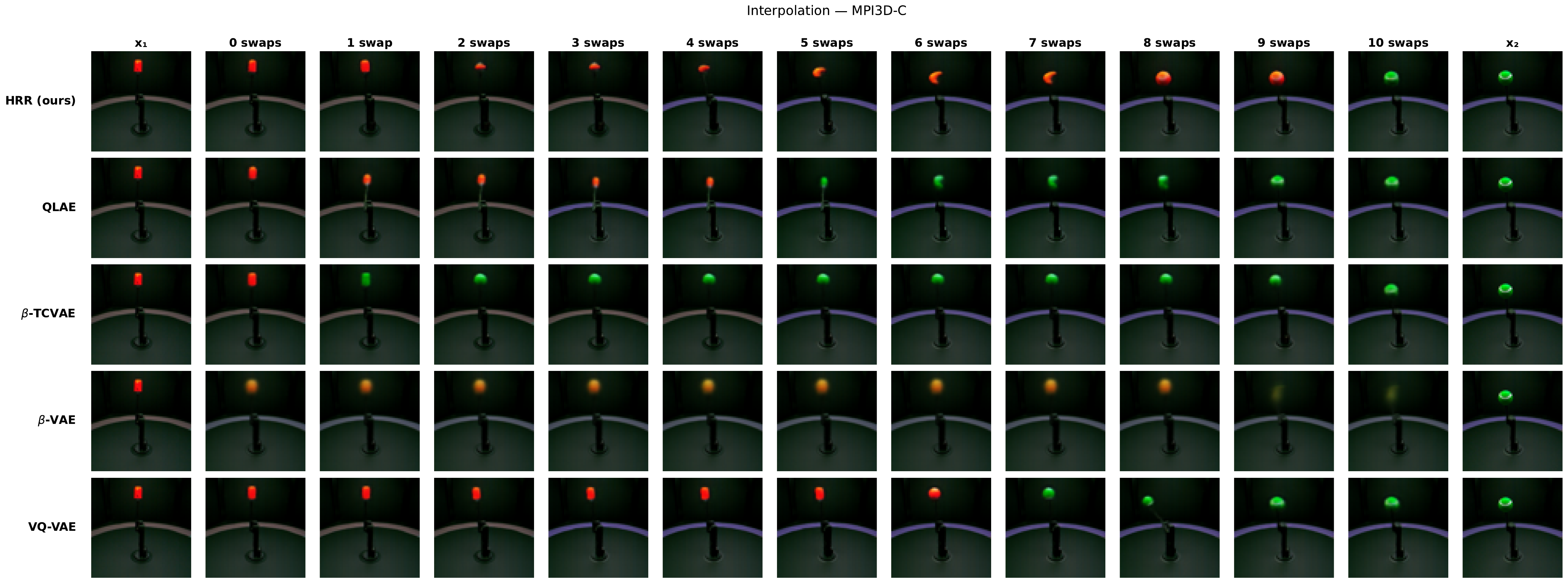}
    \caption{Latent interpolation swaps components progressively until the representation is fully transformed into that of the second image.}
    \label{fig:interp_mpi3d_complex}
\end{figure}

\clearpage

\end{document}